\documentclass[accepted]{article}

\usepackage{subfig}
\usepackage{booktabs} 
\usepackage{hyperref}
\usepackage[utf8]{inputenc}
\usepackage[T1]{fontenc}
\usepackage{nicefrac}      
\usepackage{microtype}
\usepackage{mdframed}
\usepackage{amssymb}
\usepackage{amsthm}
\usepackage{mathrsfs} 
\usepackage{mathtools}
\usepackage{dsfont}

\newtheorem{ass}{Assumption}
\newtheorem{theorem}{Theorem} 
\newtheorem{defn}{Definition}

\newtheorem{lem}{Lemma}
\newtheorem{proposition}{Proposition}

\newcommand{\dist}{\sf dist}

\usepackage{icml2020}

\icmltitlerunning{Learning disconnected manifolds: no GAN's land}
\begin{document}
\twocolumn[
\icmltitle{Learning Disconnected Manifolds: a no GAN's land}

\begin{icmlauthorlist}
\icmlauthor{Ugo Tanielian}{goo,to}
\icmlauthor{Thibaut Issenhuth}{to}
\icmlauthor{Elvis Dohmatob}{to}
\icmlauthor{Jérémie Mary}{to}
\end{icmlauthorlist}
\icmlaffiliation{to}{Criteo AI Lab, France}
\icmlaffiliation{goo}{Université Paris-Sorbonne, Paris, France}
\icmlcorrespondingauthor{Ugo Tanielian}{u.tanielian@criteo.com}

\icmlkeywords{Generative adversarial networks, manifold learning, no-free-lunch theorem}
\vskip 0.3in
]
\printAffiliationsAndNotice{}

\begin{abstract}
Typical architectures of Generative Adversarial Networks make use of a unimodal latent/input distribution transformed by a continuous generator. Consequently, the modeled distribution always has connected support which is cumbersome when learning a disconnected set of manifolds. We formalize this problem by establishing a "no free lunch" theorem for the disconnected manifold learning stating an upper-bound on the precision of the targeted distribution. This is done by building on the necessary existence of a low-quality region where the generator continuously samples data between two disconnected modes. Finally, we derive a rejection sampling method based on the norm of generator's Jacobian and show its efficiency on several generators including BigGAN.
\end{abstract}

\section{Introduction}\label{section:introduction}
GANs \cite{GANs} provide a very effective tool for the unsupervised learning of complex probability distributions. For example, \citet{karras2019style} generate very realistic human faces while \citet{yu2017seqgan} match state-of-the-art text corpora generation. Despite some early theoretical results on the stability of GANs \cite{arjovsky2017towards} and on their approximation and asymptotic properties \cite{biau2018some}, their training remains challenging. More specifically, GANs raise a mystery formalized by \citet{khayatkhoei2018disconnected}: \textit{how can they fit disconnected manifolds when they are trained to continuously transform a unimodal latent distribution?} While this question remains widely open, we will show that studying it can lead to some improvements in the sampling quality of GANs.  

\begin{figure}
    \centering
    \subfloat[Heatmap of the generator's Jacobian norm. White circles: quantiles of the latent distribution $\mathcal{N}(0,I)$.
    \label{fig:intro_gradients}]
    {   
        \includegraphics[width=0.53\linewidth]{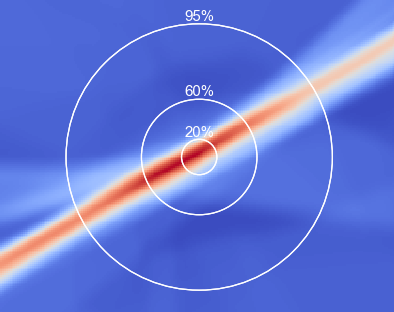}
    }\vfill
    \captionsetup[subfigure]{oneside,margin={0cm,0cm}}
    \subfloat[Green: target distribution. Coloured dots: generated samples colored w.r.t. the Jacobian Norm using same heatmap than (a).
    \label{fig:intro_data_points}]
    {   
        \includegraphics[width=0.65\linewidth]{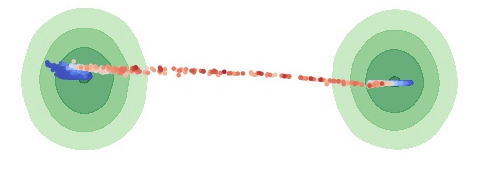}
    }
    \caption{\label{fig:intro} Learning disconnected manifolds leads to the apparition of an area with high gradients and data sampled in between modes.}
\end{figure}
Indeed, training a GAN with the objective of continuously transforming samples from an unimodal distribution into a disconnected requires balancing between two caveats. On one hand, the generator could just ignore all modes but one, producing a very limited variety of high quality samples: this is an extreme case of the well known mode collapse \cite{arjovsky2017towards}. On the other hand, the generator could cover the different modes of the target distribution and necessarily generates samples out of the real data manifold as previously explained by \citet{khayatkhoei2018disconnected}.

As brought to the fore by \citet{roth2017stabilizing}, there is a density mis-specification between the true distribution and the model distribution. Indeed, one cannot find parameters such that the model density function is arbitrarily close to the true distribution. To solve this issue, many empirical works have  proposed to over-parameterize the generative distributions, as for instance, using a mixture of generators to better fit the different target modes. \citet{tolstikhin2017adagan} rely on boosting while \citet{khayatkhoei2018disconnected} force each generator to target different sub-manifolds thanks to a criterion based on mutual information. Another direction is to add complexity in the latent space using a mixture of Gaussian distributions \cite{gurumurthy2017deligan}. 

To better visualize this phenomenon, we consider a simple 2D motivational example where the real data lies on two disconnected manifolds. Empirically, when learning the distribution, GANs split the Gaussian latent space into two modes, as highlighted by the separation line in red in Figure \ref{fig:intro_gradients}. More importantly, each sample drawn inside this red area in Figure \ref{fig:intro_gradients} is then mapped in the output space in between the two modes (see Figure \ref{fig:intro_data_points}). For the quantitative evaluation of the presence of  out-of-manifold samples, a natural metric is the Precision-Recall (PR) proposed by \citet{sajjadi2018assessing} and its improved version (Improved PR) \citep{kynkaanniemi2019improved}. A first contribution of this paper is to formally link them. Then, taking advantage of these metrics, we lower bound the measure of this out-of-manifold region and formalize the impossibility of learning disconnected manifolds with standard GANs. We also extend this observation to the multi-class generation case and show that the volume of off-manifold areas increases with the number of covered manifolds. In the limit, this increase drives the precision to zero. 

To solve this issue and increase the precision of GANs, we argue that it is possible to remove out-of-manifold samples using a truncation method. Building on the work of \citet{arvanitidis2017latentspace} who define a Riemaniann metric that significantly improves clustering in the latent space, our truncation method is based on information conveyed by the Jacobian's norm of the generator. We empirically show that this rejection sampling scheme enables us to better fit disconnected manifolds without over-parametrizing neither the generative class of functions nor the latent distribution. Finally, in a very large high dimensional setting, we discuss the advantages of our rejection method and compare it to the truncation trick introduced by \cite{brock2018large}.

In a nutshell, our contributions are the following: 
\begin{itemize}
    \item We discuss evaluation of GANs and formally link the PR measure \cite{sajjadi2018assessing} and its Improved PR version \cite{kynkaanniemi2019improved}.
    \item We upper bound the precision of GANs with Gaussian latent distribution and formalize an impossibility result for disconnected manifolds learning.
    \item Using toy datasets, we illustrate the behavior of GANs when learning disconnected manifolds and derive a new truncation method based on the Jacobian's Frobenius norm of the generator. We confirm its empirical performance on state-of-the-art models and datasets. 
\end{itemize}

\section{Related work}
\paragraph{Fighting mode collapse.}
\citet{GANs} were the first to raise the problem of mode collapse in the learning of disconnected manifolds with GANs. They observed that when the generator is trained too long without updating the discriminator, the output distribution collapses to a few modes reducing the diversity of the samples. To tackle this issue, \citet{salimans2016improved, lin2018pacgan} suggested feeding several samples to the discriminator. \citet{srivastava2017veegan} proposed the use of a reconstructor network, mapping the data to the latent space to increase diversity. 

In a different direction, \citet{arjovsky2017towards} showed that training GANs using the original formulation \cite{GANs} leads to instability or vanishing gradients. To solve this issue, they proposed a Wasserstein GAN architecture \cite{arjovsky2017wasserstein} where they restrict the class of discriminative functions to 1-Lipschitz functions using weight clipping. Pointing to issues with this clipping, 
\citet{gulrajani2017improved, spectral_normGANs} proposed relaxed ways to enforce the Lipschitzness of the discriminator, either by using a gradient penalty or a spectral normalization. Albeit not exactly approximating the Wasserstein's distance \citep{petzka2017regularization},  both implementations lead to good empirical results, significantly reducing  mode collapse. Building on all of these works, we will further assume that generators are now able to cover most of the modes of the target distribution, leaving us the problem of out-of-manifold samples (\textit{a.k.a.} low-quality pictures). 

\paragraph{Generation of disconnected manifolds.}
When learning complex manifolds in high dimensional spaces using deep generative models, \citet{fefferman2016testing} highlighted the importance of understanding the underlying geometry. More precisely, the learning of  disconnected manifold requires the introduction of disconnectedness in the model. \citet{gurumurthy2017deligan} used a multi-modal entry distribution, making the latent space disconnected, and showed better coverage when data is limited and diverse. Alternatively, \citet{khayatkhoei2018disconnected} studied the learning of a mixture of generators. Using a mutual information term, they encourage each generator to focus on a different submanifold so that the mixture covers the whole support.
This idea of using an ensemble of generators is also present in the work of \citet{tolstikhin2017adagan} and \citet{zhong2019rethinking}, though they were primarily interested in the reduction of mode collapse. 

In this paper, we propose a truncation method to separate the latent space into several disjoint areas. It is a way to learn disconnected manifolds without relying on the previously introduced over-parameterization techniques. As our proposal can be applied without retraining the whole architecture, we can use it successfully on very larges nets. Close to this idea, \citet{azadi2018discriminator} introduced a rejection strategy based on the output of the discriminator. However, this rejection sampling scheme requires the discriminator to be trained with a classification loss while our proposition can be applied to any generative models. 

\paragraph{Evaluating GANs.} The evaluation of generative models is an active area of research. 
Some of the proposed metrics only measure the quality of the generated samples such as the Inception score \cite{salimans2016improved} while others define distances between probability distributions. This is the case of the Frechet Inception distance \cite{heusel2017gans}, the Wasserstein distance \cite{arjovsky2017wasserstein} or kernel-based metrics \cite{gretton2012kernel}. The other main caveat for evaluating GANs lies in the fact that one does not have access to the true density nor the model density, prohibiting the use of any density based metrics. To solve this issue, the use of a third network that acts as an objective referee is common. For instance, the Inception score uses outputs from InceptionNet while the Fréchet Inception Distance compares statistics of InceptionNet activations. Since our work focuses on out-of-manifold samples, a natural measure is the PR measure \citep{sajjadi2018assessing} and its Improved PR version \citep{kynkaanniemi2019improved}, extensively discussed in the next section. 

In the following, alongside precise definitions, we exhibit an upper bound on the precision of GANs with high recall (\textit{i.e.} no mode collapse) and present a new truncation method. 

\section{Our approach}\label{section:approach}
We start  with a formal description of the framework of GANs and the relevant metrics. We later show a "no free lunch" theorem proving the necessary existence of an area in the latent space that generates out-of-manifold samples. We name this region the \textit{no GAN's land} since any data point sampled from this area will be in the frontier in between two different modes. We claim that dealing with it requires special care. Finally, we propose a rejection sampling  procedure to  avoid points out of the true manifold.

\subsection{Notations}
In the original setting of Generative Adversarial Networks (GANs), one tries to generate data that are ``similar'' to samples collected from some unknown probability measure $\mu_\star$. To do so, we use a parametric family of generative distribution where each distribution is the push-forward measure of a latent distribution $Z$ and a continuous function modeled by a neural network. 

\begin{ass}[$Z$ Gaussian]\label{ass:z_is_connected}
    The latent distribution $Z$ is a standard multivariate Gaussian.
\end{ass}
Note that for any distribution $\mu$, $S_\mu$ refers to its support. Assumption \ref{ass:z_is_connected} is common for GANs as in many practical applications, the random variable $Z$ defined on a low dimensional space $\mathds{R}^d$ is either a multivariate Gaussian. Practicioners also studied distribution or uniform distribution defined on a compact. 

The measure $\mu_\star$ is defined on a subset $E$ of $\mathds{R}^D$ (potentially a highly dimensional space), equipped with the norm $\|\cdot\|$. The generator has the form of a parameterized class of functions from $\mathds{R}^d$ (a space with a much lower dimension) to $E$, say $\mathscr{G}= \{G_\theta: \theta \in \Theta \}$, where $\Theta \subseteq \mathds{R}^p$ is the set of parameters describing the model. Each function $G_\theta$ thus takes input from a $d$-dimensional random variable $Z$ ($Z$ is associated with probability distribution $\gamma$) and outputs ``fake'' observations with distribution $\mu_\theta$. Thus, the class of probability measures $\mathscr{P}= \{ \mu_\theta : \theta \in \Theta \}$ is the natural class of distributions associated with the generator, and the objective of GANs is to find inside this class of candidates the one that generates the most realistic samples, closest to the ones collected from the unknown distribution $\mu_\star$.  

\begin{ass} \label{ass:lipschitz}
    Let $L>0$. The generator $G_\theta$ takes the form of a neural network whose Lipchitz constant is smaller than $L$, \textit{i.e.} for all $(z,z')$, we have $\|G_\theta(z') - G_{\theta}(z)\| \leqslant L \|z-z'\|$.
\end{ass}
This is a reasonable assumption, since \citet{virmaux2018lipschitz} present an algorithm that upper-bounds the Lipschitz constant of deep neural networks. Initially, 1-Lipschitzness was enforced only for the discriminator by clipping the weigths \cite{arjovsky2017wasserstein, ZhLiZhXuHe18}, adding a gradient penalty \cite{gulrajani2017improved, roth2017stabilizing, petzka2017regularization}, or penalizing the spectral norms \cite{spectral_normGANs}. Nowadays, state-of-the-art architectures for large scale generators such as SAGAN \cite{zhang2019self} and BigGAN \cite{brock2018large} also make use of spectral normalization for the generator. 

\subsection{Evaluating GANs with Precision and Recall}
When learning disconnected manifolds, \citet{srivastava2017veegan} proved the need of measuring simultaneously the quality of the samples generated and the mode collapse. \citet{sajjadi2018assessing} proposed the use of a PR metric to measure the quality of GANs. The key intuition is that precision should quantify how much of the fake distribution can be generated by the true distribution while recall measures how much of the true distribution can be re-constructed by the model distribution. More formally, it is defined as follows:
\begin{defn}\label{def:prec_rec_metric}\cite{sajjadi2018assessing}
Let $X, Y$ be two random variables. For $\alpha, \beta \in (0,1]$, $X$ is said to have an attainable precision $\alpha$ at recall $\beta$ w.r.t. $Y$ if there exists probability distributions $\mu, \nu_X, \nu_Y$ such that 
    \begin{equation*}
        Y= \beta \mu + (1- \beta) \nu_Y \quad \text{and} \quad X = \alpha \mu + (1 - \alpha) \nu_X
    \end{equation*}
\end{defn}
The component $\nu_Y$ denotes the part of $Y$ that is “missed”  by $X$, whereas, $\nu_X$ denotes the "noise" part of $X$. We denote $\bar{\alpha}$ (respectively $\bar{\beta}$) the maximum attainable precision (respectively recall).   Th. 1 of \cite{sajjadi2018assessing} states:
\begin{equation*}
    X\big(S_Y \big) = \bar{\alpha} \quad \text{and} \quad Y\big(S_X\big) = \bar{\beta}
\end{equation*}

\paragraph{Improved PR metric.}
\citet{kynkaanniemi2019improved} highlighted an important drawback of the PR metric proposed by \citet{sajjadi2018assessing}: it cannot correctly interpret situations when a large numbers of samples are packed together. To better understand this situation, consider a case where the generator slightly collapses on a specific data point, i.e. there exists $x \in E, \mu_\theta(x)>0$. We show in Appendix \ref{appendix:sajjadi_metric} that if $\mu_\star$ is a non-atomic probability measure and $\mu_\theta$ is highly precise (\textit{i.e.} $\alpha=1$), then the recall $\beta$ must be $0$.

To solve these issues, \citet{kynkaanniemi2019improved} proposed an \textit{Improved Precision-Recall} (Improved PR) metric built on a nonparametric estimation of support of densities. 

\begin{defn} \label{def:improved_prec_rec} \cite{kynkaanniemi2019improved}
    Let $X, Y$ be two random variables and $D_X, D_Y$ two finite sample datasets such that $D_X \sim X^n$ and $D_Y \sim Y^n$. For any $x \in D_X$ (respectively for any $y \in D_Y$), we consider $(x_{(1)}, \dots, x_{(n-1)})$, the re-ordening of elements in $D_X \setminus x$ given their euclidean distance with $x$. For any $k \in \mathds{N}$ and $x \in D_X$, the precision $\alpha_k^n(x)$ of point $x$ is defined as 
    \begin{equation*}
    \alpha_k^n(x) = 1 \iff \exists y \in D_Y, \|x-y\| \leqslant \|y_{(k)}-y\|
    \end{equation*}
    Similarly, the recall $\beta_k^n(y)$ of any given $y \in D_Y$ is
    \begin{equation*}
    \beta_k^n(y) = 1 \iff \exists x \in D_X,  \|y-x\| \leqslant \|x_{(k)}-x\|
    \end{equation*}
    Improved precision (respectively recall) are defined as the average over $D_X$ (respectively $D_Y$) as follows
    \begin{equation*}
        \alpha_k^n = \frac{1}{n} \sum_{x_i \in D_X} \alpha_k^n(x_i) \quad  \quad  \beta_k^n = \frac{1}{n} \sum_{y_i \in D_Y} \beta_k^n(y_i)
    \end{equation*}
\end{defn}

A first contribution is to formalize the link  between PR and Improved PR with the following theorem:

\begin{theorem}\label{th:improved_prec_rec}
    Let $X, Y$ two random variables with probability distributions $\mu$ and $\nu$. Assume that both $\mu$ and $\nu$ are associated with uniformly continuous probability density functions $f_\mu$ and $f_\nu$. Besides, there exists constants $a_1>0, a_2>0$ such that for all $x \in E$ we have $a_1 < f_{\mu_\star}(x) \leqslant a_2$ and $a_1 < f_{\mu_\theta}(x) \leqslant a_2$ for some $c>0$. Also, $(k,n)$ are such that $\frac{k}{\log(n)} \to +\infty$ and  $\frac{k}{n} \to 0$. Then,
    \begin{equation*}
        \alpha_k^n \to \bar{\alpha} \ \ \text{in probability} \quad \text{and} \quad \beta_k^n \to \bar{\beta} \ \ \text{in proba.}
    \end{equation*}
\end{theorem}
This theorem, whose proof is delayed to Appendix \ref{appendix:th:comparison_metrics}, underlines the nature of the Improved PR metric: the metric compares the supports of the modeled probability distribution $\mu_\theta$ and of the true distribution $\mu_\star$. This means that Improved PR is a tuple made of both maximum attainable precision $\bar{\alpha}$ and recall $\bar{\beta}$ (e.g. Theorem 1 of \cite{sajjadi2018assessing}). 
As Improved PR is shown to have a better  performance  evaluating GANs sample quality, we use this metric for both the following theoretical results and experiments. 

\subsection{Learning disconnected manifolds}
In this section, we aim to stress the difficulties of learning disconnected manifolds with standard GANs architectures. 
To begin with, we recall the following lemma.
\begin{lem}\label{lem:connected_image_space}
    Assume that Assumptions \ref{ass:z_is_connected} and \ref{ass:lipschitz} are satisfied. Then, for any $\theta \in \Theta$, the support $S_{\mu_\theta}$ is  connected.
\end{lem}
There is consequently a discrepancy between the connectedness of $S_{\mu_\theta}$ and the disconnectedness of $S_{\mu_\star}$. In the case where the manifold lays on two disconnected components, our next theorem exhibit a no free lunch theorem: 

\begin{theorem}("No free lunch" theorem)\label{th:no_free_lunch}
    Assume that Assumptions \ref{ass:z_is_connected} and \ref{ass:lipschitz} are satisfied. Assume also that true distribution $\mu_\star$ lays on two equally measured disconnected manifolds distant from a distance $D>0$. Then, any estimator $\mu_\theta$ that samples equally in both modes must have a precision $\bar{\alpha}$ such that $\bar{\alpha} + \frac{D}{\sqrt{2\pi}L}  e^{\frac{-\Phi^{-1}(\frac{\bar{\alpha}}{2})^2}{2}} \leqslant 1$, where $\Phi$ is the c.d.f. of a standard normal distribution. 
    
    Besides, if $\bar{\alpha} \geqslant 3/4$, $\bar{\alpha} \lesssim 1 - \sqrt{\frac{2}{\pi}} W(\frac{D^2}{4L^2})$ where $W$ is the Lambert $W$ function. 
\end{theorem}
 The proof of this theorem is delayed to Appendix \ref{appendix:proof_th1}. It is mainly based on the Gaussian isoperimetric inequality \cite{borell1975brunn, sudakov1978extremal} that states that among all sets of given Gaussian measure in any finite dimensional Euclidean space, half-spaces have the minimal Gaussian boundary measure. If in Fig.~\ref{fig:intro}, the generator has thus learned the optimal separation, it is yet not known, to the limit of our knowledge, how to enforce such geometrical properties in the latent space.

In real world applications, when the number of distinct sub-manifolds increases, we expect the volume of these boundaries to increase with respect to the number of different classes covered by the modeled distribution $\mu_\theta$. Going in this direction, we better formalize this situation, and show an extended "no free lunch theorem" by expliciting an upper-bound of the precision $\bar{\alpha}$ in this broader framework.

\begin{ass}\label{ass:disconnected_extension}
    The true distribution $\mu_\star$ lays on $M$ equally-measured disconnected components at least distant from some constant $D>0$. 
\end{ass}
This is likely to be true for datasets made of symbol designed to be highly distinguishable (\textit{e.g.} digits in the MNIST dataset). In very high dimension, this assumption also holds for complex classes of objects appearing in many different contexts  (\textit{e.g.} the bubble class in ImageNet, see Appendix). 

To better apprehend the next theorem, note $A_m$ the pre-image in the latent space of mode $m$ and $A_m^r$ its $r$-enlargement: $A_m^r := \{z \in \mathds R^d \mid \dist(z,A_m) \le r\}, r>0$.
\begin{theorem}(Generalized "no free lunch" theorem)\label{th:extended_no_free_lunch}
    Assume that Assumptions \ref{ass:z_is_connected}, \ref{ass:lipschitz}, and \ref{ass:disconnected_extension} are satisfied, and that the pre-image enlargements $A_m^\varepsilon$, with $\varepsilon=\frac{D}{2L}$, form a partition of the latent space with equally measured elements. 
    
    Then, any estimator $\mu_\theta$ with recall $\bar{\beta}>\frac 1M$ must have a precision $\bar{\alpha}$ at most $ \frac{1+x^2}{x^2} e^{-\frac{1}{2}\varepsilon^2}e^{-\varepsilon x}$ where $x = \Phi^{-1} (1-\frac{1}{\bar{\beta}M})$ and $\Phi$ is the c.d.f. of a standard normal distribution.
\end{theorem}

\begin{figure}
    \captionsetup[subfigure]{margin={0.5cm,0cm}}
    \subfloat[WGAN 4 classes: \newline visualisation of $\|J_G(z)\|_{F}$.  \label{fig:gradients_synthetic_4class}]
    {   
        \includegraphics[width=0.45\columnwidth]{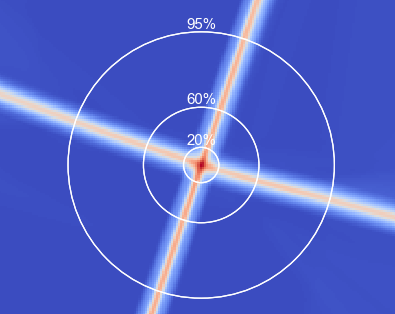}
    }
    \captionsetup[subfigure]{margin={0.5cm,0cm}}
    \subfloat[WGAN 9 classes: \newline
    visualisation of $\|J_G(z)\|_{F}$.  \label{fig:gradients_synthetic_9class}]
    {   
        \includegraphics[width=0.45\columnwidth]{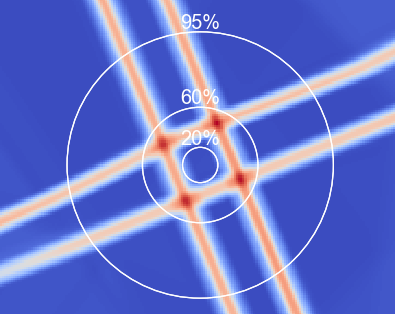}
    }
    \hfill
    \captionsetup[subfigure]{margin={0.5cm,0cm}}
    \subfloat[WGAN 25 classes: \newline
    visualisation of $\|J_G(z)\|_{F}$.  \label{fig:gradients_synthetic_25class}]
    {   
        \includegraphics[width=0.45\columnwidth]{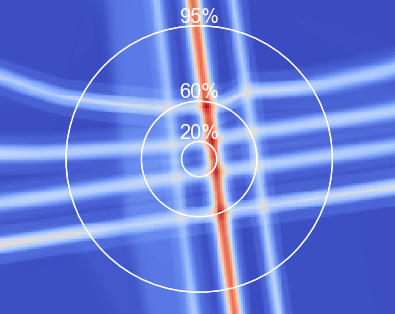}
    }
    \captionsetup[subfigure]{margin={0.6cm,0cm}}
    \subfloat[Precision w.r.t. $D$ (mode distance) and $M$ (classes).\label{fig:WGAN_no_truncation}]
    {   
        \includegraphics[width=0.47\linewidth]{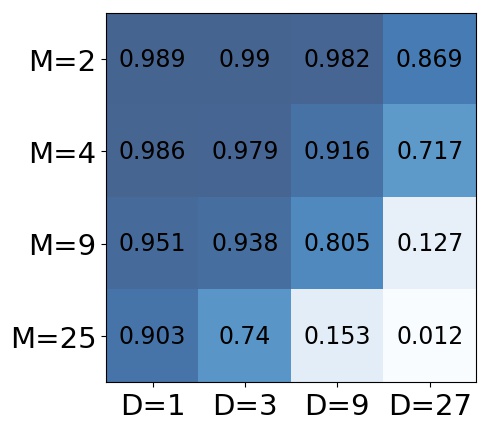}
    }
    \caption{\label{fig:explaining_theorems} Illustration of Theorem \ref{th:extended_no_free_lunch}. If the number of classes $M \to \infty$ or the distance $D \to \infty$, then the precision $\bar{\alpha} \to 0$. We provide in appendix heatmaps for more values of $M$.}
\end{figure}
Theorem \ref{th:extended_no_free_lunch}, whose proof is delayed to Appendix \ref{appendix:proof_theorem3}, states a lower-bound the measure of samples mapped out of the true manifold. We expect our bound to be loose since no theoretical results are known, to the best of our knowledge, on the geometry of the separation that minimizes the boundary between different classes (when $M \geqslant 3$). Finding this optimal cut would be an extension of the honeycomb theorem \cite{hales2001honeycomb}. In Appendix \ref{appendix:theorem3_more_general_setting} we give a more technical statement of Theorem \ref{th:extended_no_free_lunch} without assuming equality of measure of the sets $A_m^\varepsilon$.

The idea of the proof is to consider the border of an individual cell with the rest of the partition. It is clear that at least half of the frontier will be inside this specific cell. Then, to get to the final result, we sum the measures of the frontiers contained inside all of the different cells. Remark that our analysis is fine enough to keep a dependency in $M$ which translates into a maximum precision that goes to zero when $M$ goes to the infinity and all the modes are covered. More precisely, in this scenario where all pre-images have equal measures in the latent space, one can derive the following bound, when the recall $\bar{\beta}$ is kept fixed and $M$ increases:
\begin{equation}\label{asymptotic_illustration_of_th3}
    \bar{\alpha} \overset{M \rightarrow \infty}{\leqslant} e^{-\frac{1}{2}\varepsilon^2}e^{-\varepsilon\sqrt{2\log(\bar{\beta}M})} \quad \text{where} \ \varepsilon = \frac{D}{2L}
\end{equation}
For a fixed generator, this equation illustrates that the precision $\bar{\alpha}$ decreases when either the distance $D$ (equivalently $\varepsilon$) or the number of classes $M$ increases. For a given $\varepsilon$, $\bar{\alpha}$ converges to $0$ with a speed  $O(\frac{1}{(\bar{\beta}M)^{\sqrt{2}\varepsilon}})$. To better illustrate this asymptotic result, we provide results from a 2D synthetic setting. In this toy dataset, we control both the number $M$ of disconnected manifolds and the distance $D$. Figure \ref{fig:explaining_theorems} clearly corroborates \eqref{asymptotic_illustration_of_th3} as we can easily get the maximum precision close to $0$ ($M=25$, $D=27$). 

\subsection{Jacobian-based truncation (JBT) method}
The analysis of the deformation of the latent space offers a grasp on the behavior of GANs. For instance, \citet{arvanitidis2017latentspace} propose a distance accounting for the distortions made by the generator. For any pair of points $(z_1, z_2) \sim Z^2$, the distance is defined as the length of the geodesic $d(z_1, z_2) = \int_{[0,1]} \| J_{G_\theta}(\gamma_t) \frac{d \gamma_t}{dt}\| dt$ where $\gamma$ is the geodesic parameterized by $t \in [0,1]$ and $J_{G_\theta}(z)$ denotes the Jacobian matrix of the generator at point $z$. Authors have shown that the use of this distance in the latent space improves clustering and interpretability. We make a similar observation that the generator's Jacobian Frobenius norm provides meaningful information. 

Indeed, the frontiers highlighted in Figures \ref{fig:gradients_synthetic_4class},  \ref{fig:gradients_synthetic_9class}, and \ref{fig:gradients_synthetic_25class} correspond to areas of low precision mapped out of the true manifold: this is the \textit{no GAN's land}. We argue that when learning disconnected manifolds, the generator tries to minimize the number of samples that do not belong to the support of the true distribution and that this can only be done by making paths steeper in the \textit{no GAN's land}. Consequently,  data points $G_\theta(z)$ with high Jacobian Frobenius norm (JFN) are more likely to be outside the true manifold. 
To improve the precision of generative models, we thus define a new truncation method by removing points with highest JFN. 

However, note that computing the generators's JFN is expensive to compute for neural networks, since being defined as follows,
\begin{equation*}\label{eq:JFN}
    \| J_{G_\theta}(z) \|_{F}^2 = \sum\limits_{i = 1}^{m} \sum\limits_{j = 1}^{n} \left(\frac{\partial G_\theta(z)_i}{\partial z_j}\right)^2,    
\end{equation*}
it requires a number of backward passes equal to the output dimension. To make our truncation method tractable, we use a stochastic approximation of the Jacobian Frobenius norm based on the following result from \citet{rifai2011higher}:
\begin{equation*}
\| J_{G_\theta}(z) \|^2 = \lim\limits_{\substack{N \to \infty \\ \sigma \to 0}} \ \frac{1}{N} \sum_{\varepsilon_i}^N \frac{1}{\sigma^2} \| G_\theta(z+\varepsilon_i) - G_\theta(z) \|^2 
\end{equation*}
where $\varepsilon_i \sim  \sim \mathcal{N}(0,\sigma^2 I$ and $I$ is the identity matrix of dimension $d$. The variance $\sigma$ of the noise and the number of samples are used as hyper-parameters. In practice,  $\sigma$ in $[1\mathrm{e}{-4}; 1\mathrm{e}{-2}]$ and $N = 10$ give consistent results. 

Based on the preceding analysis, we propose a new \textbf{Jacobian-based truncation} (JBT) method that rejects a certain ratio of the generated points with highest JFN. This truncation ratio is considered as an hyper-parameter for the model. We show in our experiments that our JBT can be used to to detect samples outside the real data manifold and that it consequently improves the precision of the generated distribution as measured by the Improved PR metric. 

\section{Experiments}\label{section:experiments}
In the following, we show that our truncation method, JBT, can significantly improve the performances of generative models on several models, metrics and datasets. Furthermore, we compare JBT with over-parametrization techniques specifically designed for disconnected manifold learning. We show that our truncation method reaches or surpasses their performance, while it has the benefit of not modifying the training process of GANs nor using a mixture of generators, which is computationally expensive. Finally, we confirm the efficiency of our method by applying it on top of BigGAN \cite{brock2018large}. 

Except for BigGAN, for all our experiments, we use Wasserstein GAN with gradient penalty \cite{gulrajani2017improved}, called WGAN for conciseness. We give in Appendix \ref{appendix:experimental_details} the full details of our experimental setting. The use of WGAN is motivated by the fact that it was shown to stabilize the training and significantly reduce mode collapse \cite{arjovsky2017towards}. However, we want to emphasise that our method can be plugged on top of any generative model fitting disconnected components. 

\subsection{Evaluation metrics}
To measure performances of GANs when dealing with low dimensional applications - as with synthetic datasets - we equip our space with the standard Euclidean distance. However, for high dimensional applications such as image generation, \citet{brock2018large, kynkaanniemi2019improved} have shown that embedding images into a feature space with a pre-trained convolutional classifier provides more semantic information. In this setting, we consequently use the euclidean distance between the images' embeddings from a classifier. 
For a pair of images $(a,b)$, we define the distance $d(a,b)$ as $d(a,b) = \| \phi(a) - \phi(b) \|_2$ where $\phi$ is a pre-softmax layer of a supervised classifier, trained specifically on each dataset. Doing so, they will more easily separate images sampled from the true distribution $\mu_\star$ from the ones sampled by the distribution $\mu_\theta$. 
 
We compare performances  using Improved PR  \cite{kynkaanniemi2019improved}. We also report the \textit{Marginal Precision} which is the precision of newly added samples when increasing the ratio of kept samples. Besides, for completeness, we report FID \cite{heusel2017gans} and recall precise definitions in Appendix~\ref{appendix:metrics}. Note that FID was not computed with InceptionNet, but a classifier pre-trained on each dataset. 

\subsection{Synthetic dataset}\label{subsection:exp_synthetic}
We first consider the true distribution to be a 2D Gaussian mixture of 9 components. Both the generator and the discriminator are modeled with feed-forward neural networks.

\begin{figure}
    \centering
    \hspace{-0.15cm}
    \subfloat[WGAN - 2500 samples\label{fig:WGAN_no_truncationBIS}]
    {   
        \includegraphics[width=0.41\linewidth]{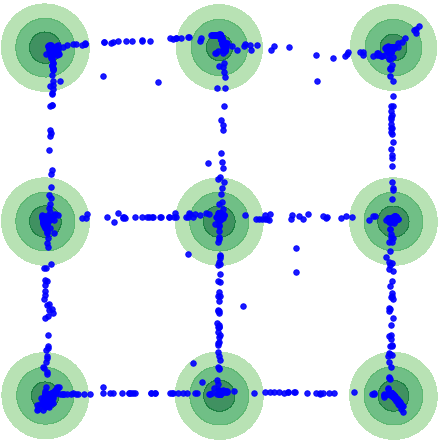}
    }
    \hspace{0.4cm}
    \subfloat[WGAN  90\% JBT. \label{fig:WGAN_truncation_highest_gradient_synth_90}]
    {   
        \includegraphics[width=0.41\linewidth]{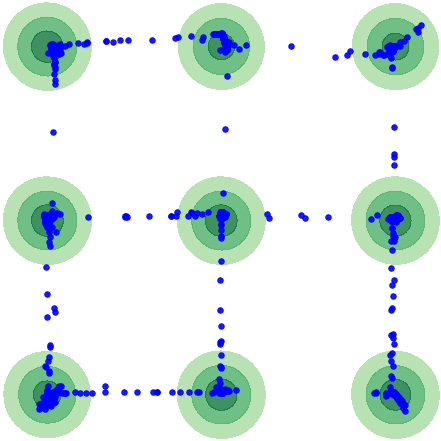}
    }\hfill
    \subfloat[WGAN 70\% JBT. \label{fig:WGAN_truncation_highest_gradient_synth_80}]
    {   
        \includegraphics[width=0.41\linewidth]{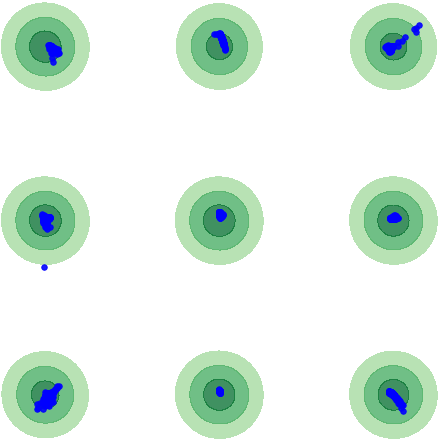}
    }
    \subfloat[97\% confidence intervals .\label{fig:global_plots}]
    {   
        \includegraphics[width=0.5\linewidth]{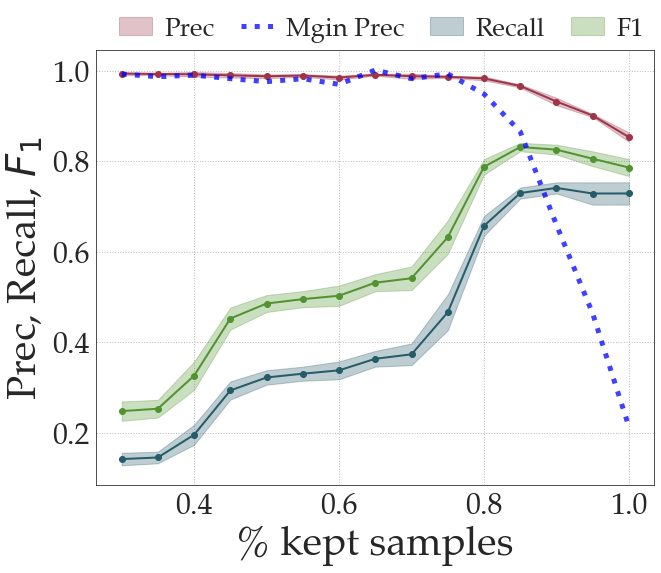}
    }
    \caption{\label{fig:synthetic_9_modes} Mixture of 9 Gaussians in green, generated points in blue. Our truncation method (JBT)  removes least precise data points as marginal precision plummets.}
\end{figure}

\begin{figure*}
    {   
        \includegraphics[width=0.28\linewidth]{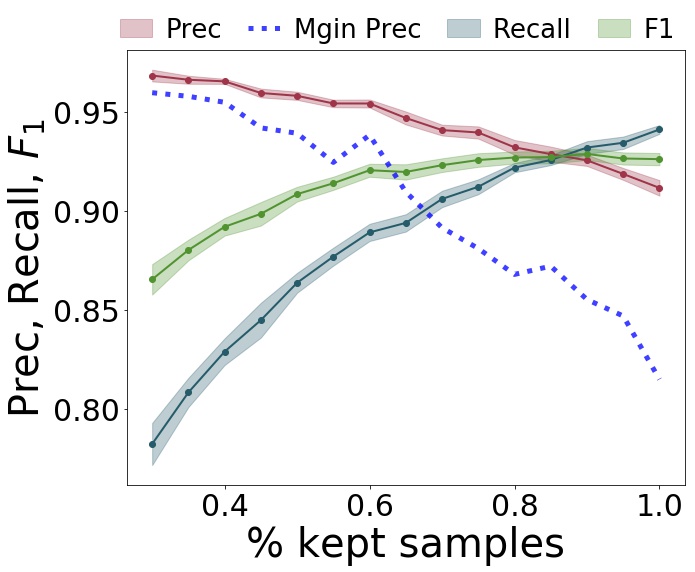}
    } \hfill
    {   
        \includegraphics[width=0.28\linewidth]{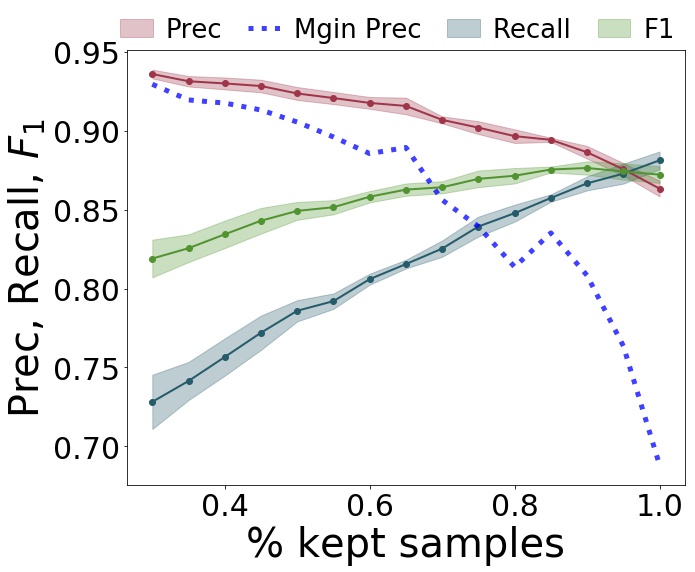}
    } \hfill
    {   
        \includegraphics[width=0.28\linewidth]{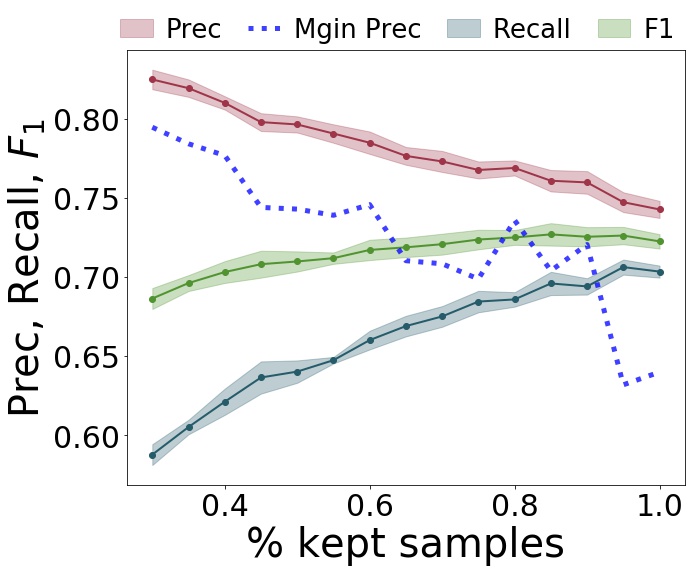}
    }\\
    ~\hspace{1cm}
     \subfloat[MNIST dataset. \label{fig:ranked_MNIST}]
    {   
        \includegraphics[width=0.28\linewidth]{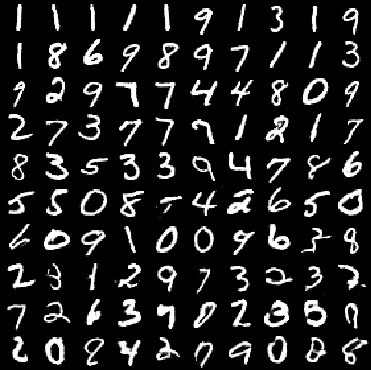}
    } \hfill
     \subfloat[F-MNIST dataset. \label{fig:ranked_FMNIST}]
    {   
        \includegraphics[width=0.28\linewidth]{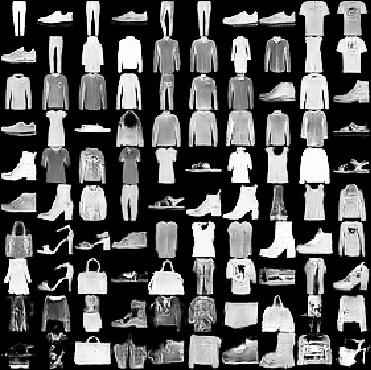}
    }\hfill
    \subfloat[CIFAR10 datatset. \label{fig:ranked_CIFAR10}]
    {   
        \includegraphics[width=0.28\linewidth]{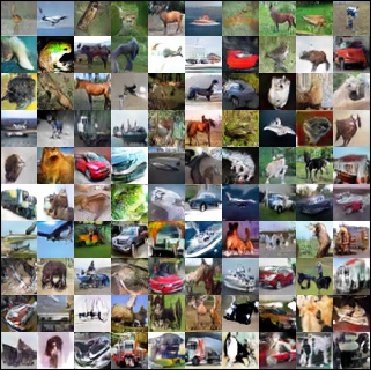}
    }
    \\
    \vspace{-0.4cm}
    \caption{For high levels of kept samples, the marginal precision plummets of newly added samples, underlining the efficiency of our truncation method (JBT). Reported confidence intervals are $97\%$ confidence intervals. On the second row, generated samples ordered by their JFN (left to right, top to bottom). In the last row, the data points generated are blurrier and outside the true manifold.  \label{fig:trunc_MNIST_FMNIST}} 
\end{figure*}

Interestingly, the generator tries to minimize the sampling of off-manifolds data during training until its JFN gets saturated (see Appendix \ref{appendix:precision_saturated_synthetic}). One way to reduce the number of off-manifold samples is to use JBT. Indeed, off-manifold data points progressively disappear when being more and more selective, as illustrated in Figure \ref{fig:WGAN_truncation_highest_gradient_synth_80}. We quantitatively confirm that our truncation method (JBT) improves the precision. On Fig. \ref{fig:global_plots}, we observe that keeping the 70\% of lowest JFN samples leads to an almost perfect precision of the support of the generated distribution. Thus, off-manifold samples are in the 30\% samples with highest JFN. 

\subsection{Image datasets}\label{subsection:image_datasets}
We further study JBT on three different datasets: MNIST \cite{lecun98gradientbasedlearning}, FashionMNIST \cite{xiao2017/online} and CIFAR10 \cite{krizhevsky2009learning}. Following \cite{khayatkhoei2018disconnected} implementation, we use a standard CNN architecture for MNIST and FashionMNIST while training a ResNet-based model for CIFAR10 \citep{gulrajani2017improved}. 

\begin{table}[ht]
\begin{center}
\begin{tabular}{|l|c|c|c|}
\cline{2-4}
\multicolumn{1}{l|}{\textbf{MNIST}}  & Prec. & Rec. &  FID \\
\hline
WGAN& $91.2 {\scriptstyle \pm 0.3}$ & $\mathbf{93.7 {\scriptstyle \pm 0.5}} $ & $24.3 {\scriptstyle  \pm 0.3}$ \\
WGAN JBT 90\%& $92.5 {\scriptstyle \pm 0.5}$ & $92.9 {\scriptstyle \pm 0.3}$ & $26.9 {\scriptstyle \pm 0.5}$ \\
WGAN JBT 80\%& $\mathbf{93.3 {\scriptstyle \pm 0.3}}$ & $91.8 {\scriptstyle \pm 0.4}$ & $33.1 {\scriptstyle \pm 0.3}$ \\
W-Deligan & $89.0 {\scriptstyle \pm 0.6}$ & $\mathbf{93.6 {\scriptstyle \pm 0.3}}$ & $31.7 {\scriptstyle \pm 0.5}$\\
DMLGAN & $\mathbf{93.4 {\scriptstyle \pm 0.2}}$ & $92.3 {\scriptstyle \pm 0.2}$ & $\mathbf{16.8 {\scriptstyle \pm 0.4}}$ \\
\hline
\multicolumn{4}{l} {\textbf{F-MNIST}}  \\
\hline
WGAN & $86.3 {\scriptstyle \pm 0.4}$ &  $\mathbf{88.2 {\scriptstyle \pm 0.2}}$ & $259.7 {\scriptstyle \pm 3.5}$\\
WGAN JBT 90\%& $88.6 {\scriptstyle \pm 0.6}$ & $86.6 {\scriptstyle \pm 0.5}$ & $\mathbf{257.4 {\scriptstyle \pm 3.0}}$ \\
WGAN JBT 80\% & $\mathbf{89.8 {\scriptstyle \pm 0.4}}$ & $84.9 {\scriptstyle \pm 0.5}$ & $396.2 {\scriptstyle \pm 6.4}$\\
W-Deligan & $88.5 {\scriptstyle \pm 0.3}$ & $85.3 {\scriptstyle \pm 0.6}$ & $310.9 {\scriptstyle \pm 3.1}$ \\
DMLGAN & $87.4 {\scriptstyle \pm 0.3}$ & $\mathbf{88.1 {\scriptstyle \pm 0.4}}$ & $\mathbf{253.0 {\scriptstyle \pm 2.8}}$ \\
\hline
\end{tabular}
\end{center}
\caption{JBT $x\%$ means we keep the $x\%$ samples with lowest Jacobian norm. Our truncation method (JBT) matches over-parameterization techniques. \label{table:prec_rec_mnist_paper} $\pm$ is $97\%$ confidence interval.}
\end{table}

Figure \ref{fig:trunc_MNIST_FMNIST} highlights that JBT also works on high dimensional datasets as the marginal precision plummets for high truncation ratios. Furthermore, when looking at samples ranked by increasing order of their JFN, we notice that samples with highest JFN are standing in-between manifolds. For example, those are ambiguous digits resembling both a "0" and a "6" or shoes with unrealistic shapes. 

To further assess the efficiency of our truncation method, we also compare its performances with two state-of-the-art over-parameterization techniques that were designed for disconnected manifold learning. First, \cite{gurumurthy2017deligan} propose DeliGAN, a reparametrization trick to transform the unimodal Gaussian latent distribution into a mixture. The different mixture components are later learnt by gradient descent. For fairness, the re-parametrization trick is used on top of WGAN. Second, \cite{khayatkhoei2018disconnected} define DMLGAN, a mixture of generators to better learn disconnected manifolds. In this architecture, each generator is encouraged to target a different submanifold by enforcing high mutual information between generated samples and generator's ids. Keep in mind that for DeliGAN (respectively DMLGAN), the optimal number of components (respectively generators) is not known and is a hyper-parameter of the model that has to be cross-validated. 

\begin{figure*}
    {   
        \includegraphics[width=0.25\linewidth]{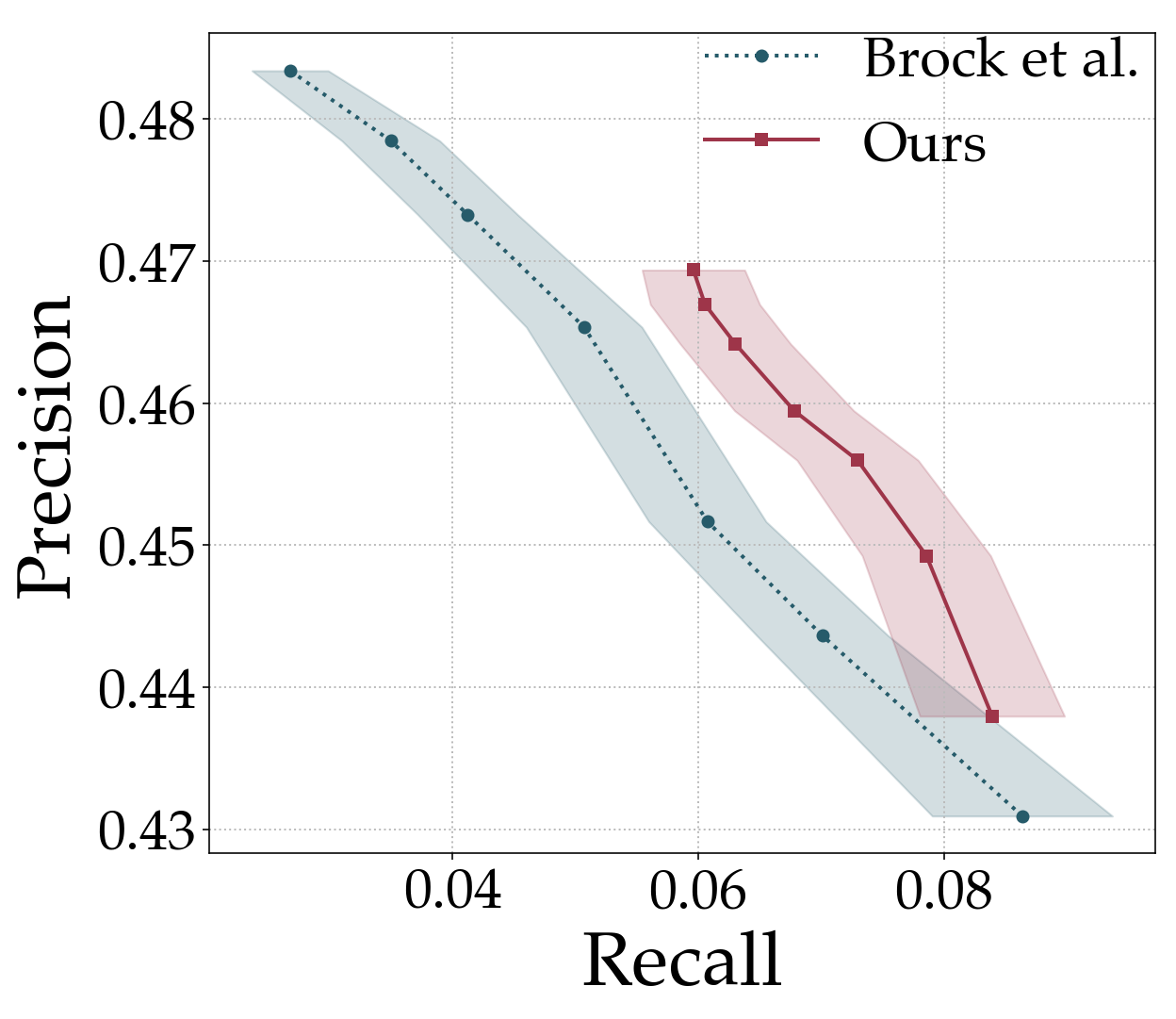}
    } \hfill
    {   
        \includegraphics[width=0.25\linewidth]{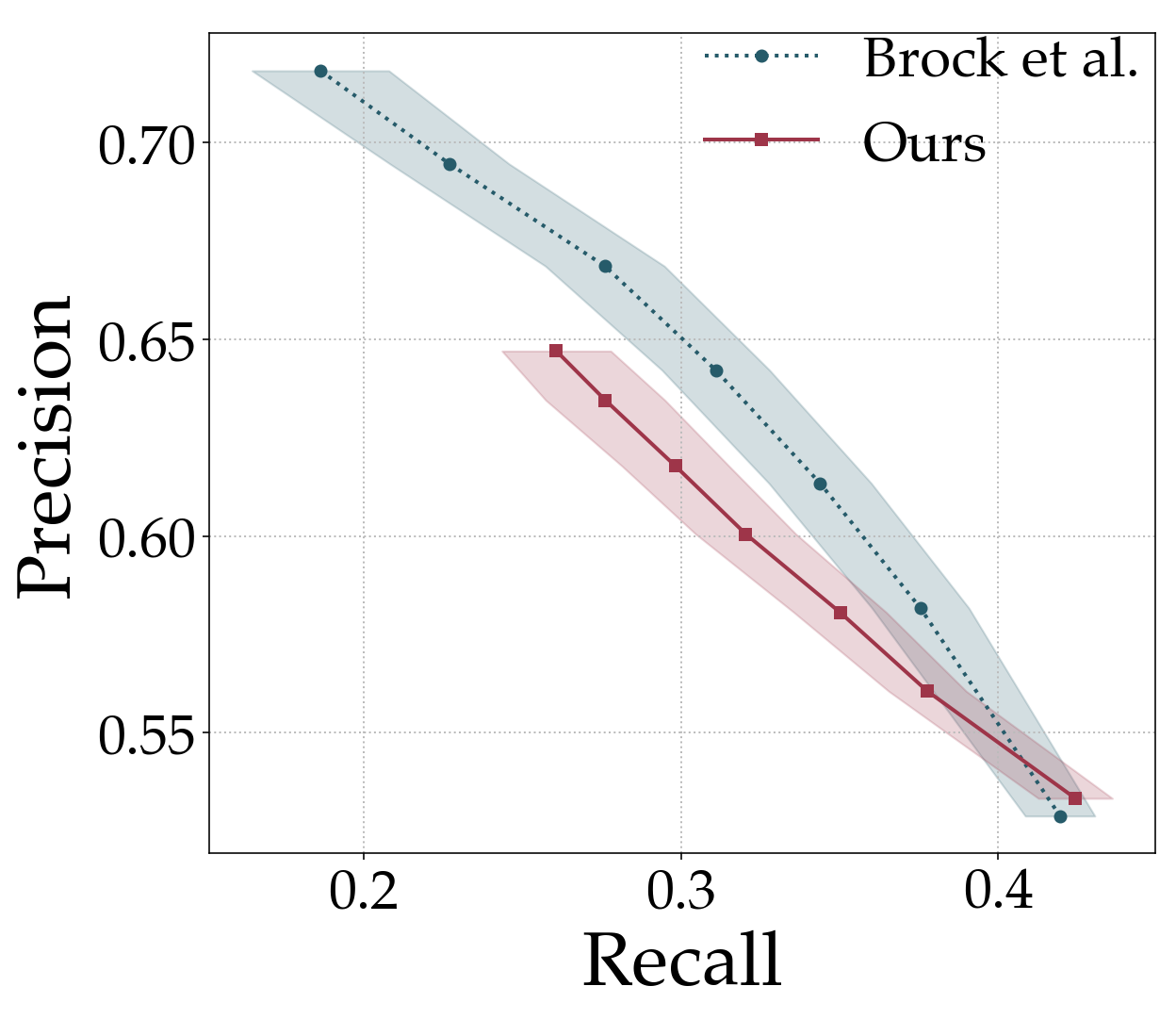}
    } \hfill
    {   
        \includegraphics[width=0.25\linewidth]{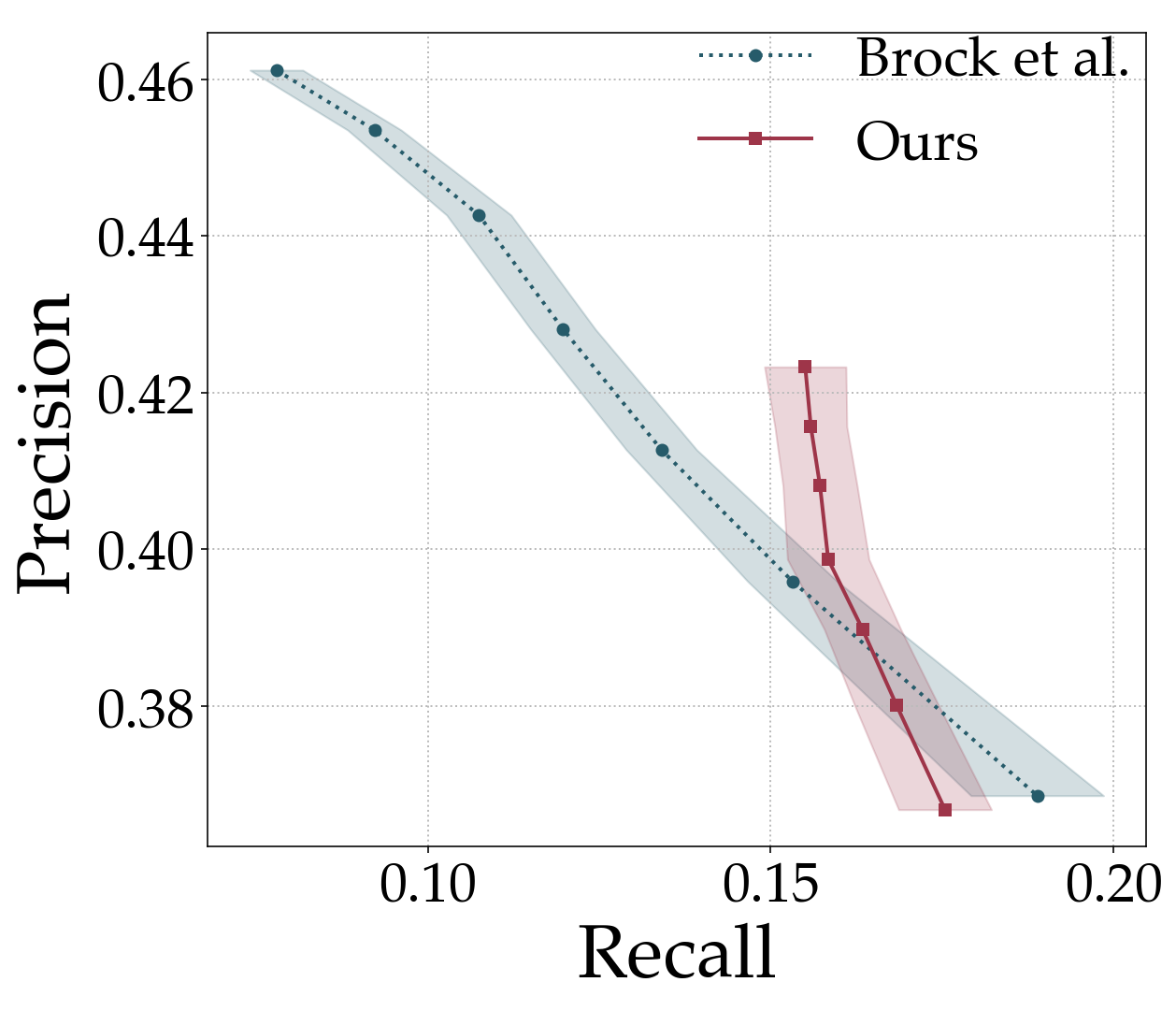}
    }\\
    ~\hspace{1cm}
     \subfloat[House finch. \label{fig:ranked_BG1}]
    {   
        \includegraphics[width=0.25\linewidth]{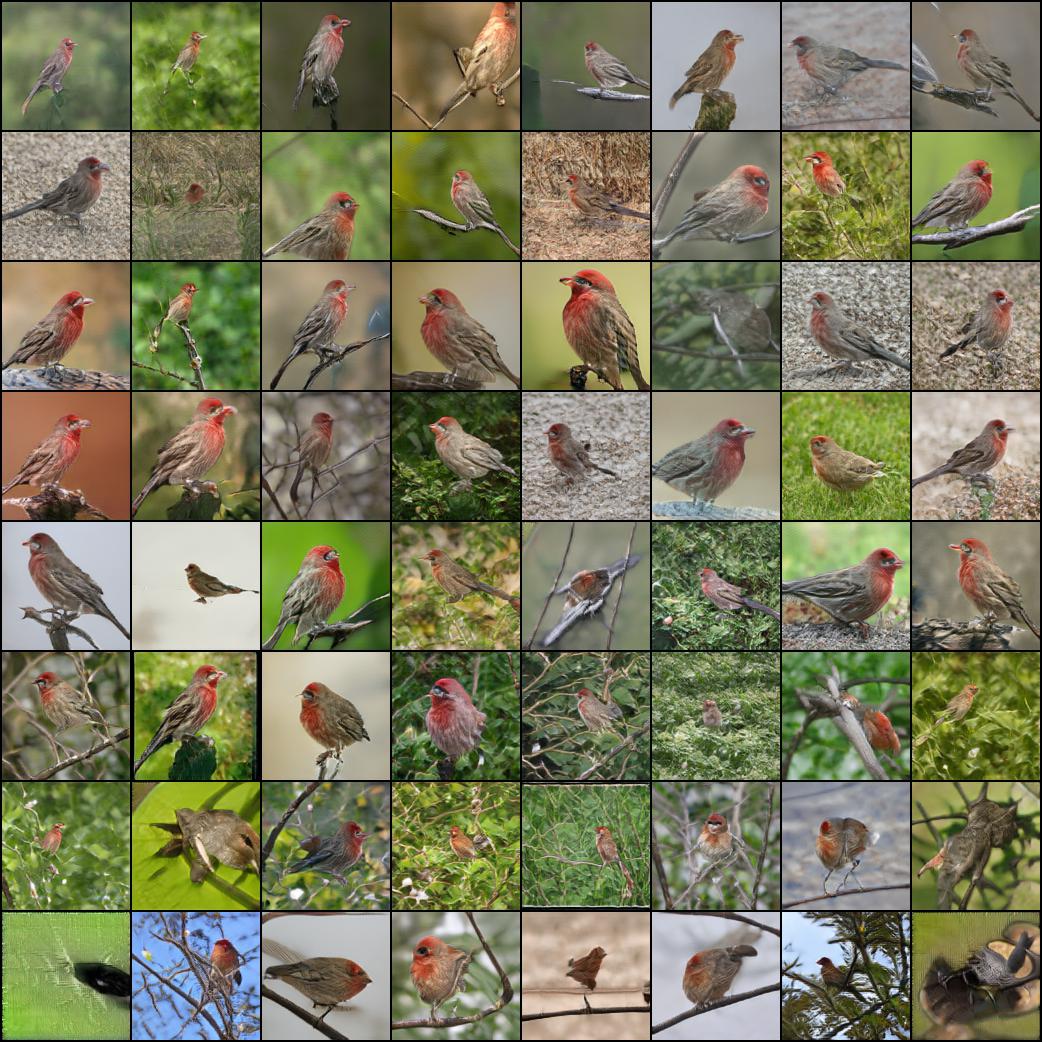}
    } \hfill
     \subfloat[Parachute.  \label{fig:ranked_BG2}]
    {   
        \includegraphics[width=0.25\linewidth]{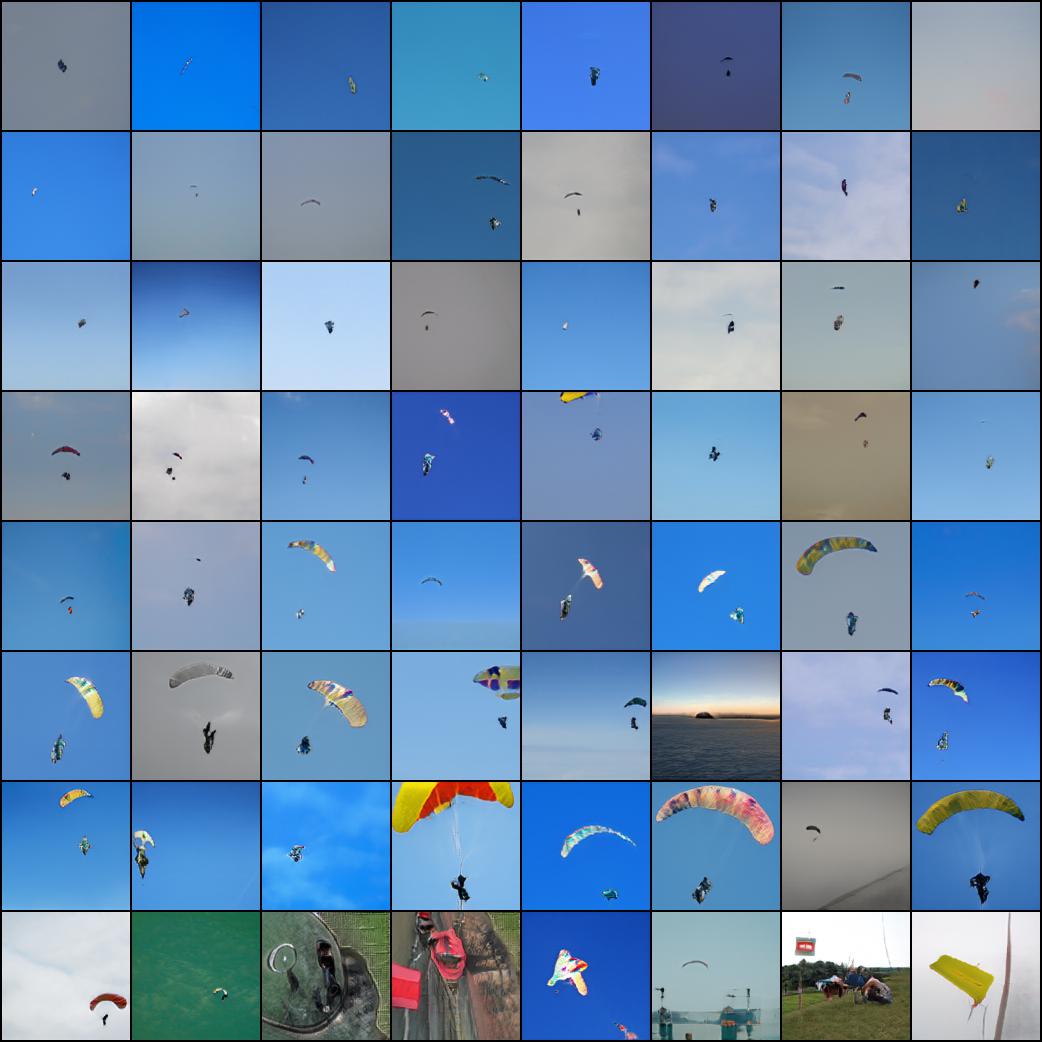}
    }\hfill
    \subfloat[Bubble.  \label{fig:ranked_BG3}]
    {   
        \includegraphics[width=0.25\linewidth]{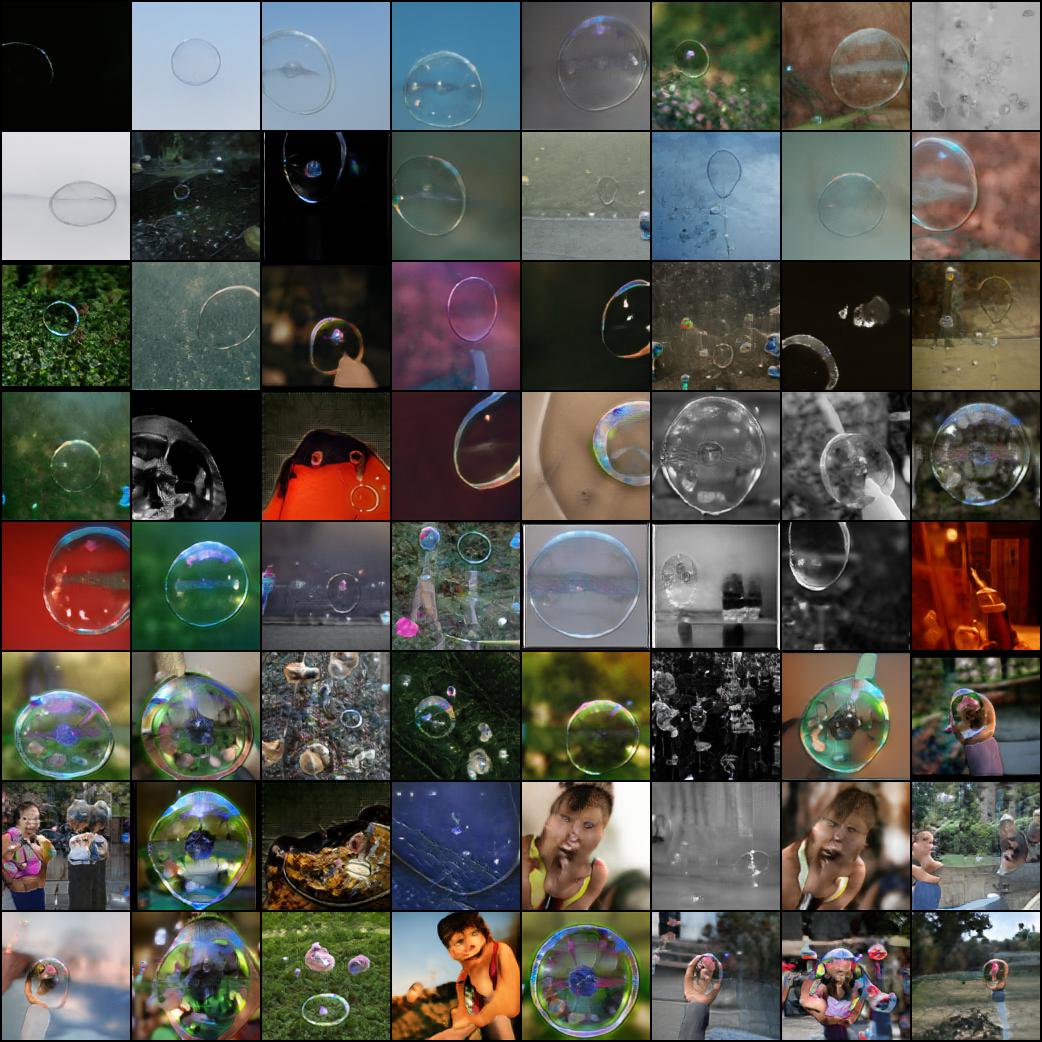}
    }
    \\ 
    \vspace{-0.4cm}
    \caption{On the first row, per-class precision-recall curves comparing \citet{brock2018large}'s truncation trick and our truncation method (JBT), on three ImageNet classes generated by BigGAN. We show better results on complex and disconnected classes (\textit{e.g.} bubble). Reported confidence intervals are 97\% confidence intervals. On the second row, generated samples ordered by their JFN (left to right, top to bottom). We observe a concentration of off-manifold samples for images on the bottom row, confirming the soundness of JBT. 
    \label{fig:trunc_BG}} 
\end{figure*}

The results of the comparison are presented in Table \ref{table:prec_rec_mnist_paper}. In both datasets,  JBT 80 \% outperforms  DeliGAN and DMLGAN in terms of precision while keeping a reasonnable recall. This confirms our claim that over-parameterization techniques are unnecessary. 
As noticed by \citet{kynkaanniemi2019improved}, we also observe that FID does not correlate properly with the Improved PR metric. Based on the Frechet distance, only a distance between multivariate Gaussians, we argue that FID is not suited for disconnected manifold learning as it approximates distributions with unimodal ones and looses many information. 

\subsection{Spurious samples rejections on BigGAN \label{subsection:biggan}}
Thanks to the simplicity of  JBT, we can also apply it on top of any trained generative model. In this subsection, we use JBT to improve the precision of a pre-trained BigGAN model \cite{brock2018large}, which generates class-conditionned ImageNet \cite{deng2009imagenet} samples. The class-conditioning lowers the problem of off-manifold samples, since it reduces the disconnectedness in the output distribution. However, we argue that the issue can still exist on high-dimensional natural images, in particular complex classes can still be multi-modal (\textit{e.g.} the bubble class). The bottom row in Figure \ref{fig:trunc_BG} shows a random set of 128 images for three different classes ranked by their JFN in ascending order (left to right, top to bottom). We observe a clear concentration of spurious samples on the bottom row images. 

To better assess the Jacobian based truncation method, we compare it with the truncation trick from \citet{brock2018large}. This truncation trick aims to reduce the variance of the latent space distribution using truncated Gaussians. While easy and effective, this truncation has some issues: it requires to complexify the loss to enforce orthogonality in weight matrices of the network. Moreover, as explained by \citet{brock2018large} \textit{"only 16\% of models are amenable to truncation, compared to 60\% when trained with Orthogonal Regularization"}. For fairness of comparison, the pre-trained network we use is  optimized for their truncation method. On the opposite, JBT is simpler to apply since 100\% of the tested models were amenable to the proposed truncation. 

Results of this comparison are shown in the upper row of Figure \ref{fig:trunc_BG}. Our method can outperform their truncation trick on difficult classes with high intra-class variation, \textit{e.g.} bubble and house finch. This confirms our claim that JBT can detect outliers within a class. However, one can note that their trick is particularly well suited for simpler unimodal classes, \textit{e.g.} parachute and reaches high precision levels.

\section{Conclusion}
In this paper, we provide  insights on the learning of disconnected manifolds with GANs. Our  analysis shows the existence of an off-manifold area with low precision. We empirically show on several datasets and models that we can detect these areas and remove samples located in between two modes thanks to a newly proposed truncation method. %In future work, we will study how the use and understanding of latent space geometry can further help us to better fit the interior of  disconnected manifolds. We suspect the \textit{no GAN's land} to be involved in the training unstability of GANs since the presence of high gradients in a low probability region might induce a lot of variance.

\clearpage
\bibliographystyle{icml2020}
\bibliography{adversarial_training}

\begin{thebibliography}{44}
\providecommand{\natexlab}[1]{#1}
\providecommand{\url}[1]{\texttt{#1}}
\expandafter\ifx\csname urlstyle\endcsname\relax
  \providecommand{\doi}[1]{doi: #1}\else
  \providecommand{\doi}{doi: \begingroup \urlstyle{rm}\Url}\fi

\bibitem[Arjovsky \& Bottou(2017)Arjovsky and Bottou]{arjovsky2017towards}
Arjovsky, M. and Bottou, L.
\newblock Towards principled methods for training generative adversarial
  networks.
\newblock In \emph{International Conference on Learning Representations}, 2017.

\bibitem[Arjovsky et~al.(2017)Arjovsky, Chintala, and
  Bottou]{arjovsky2017wasserstein}
Arjovsky, M., Chintala, S., and Bottou, L.
\newblock Wasserstein generative adversarial networks.
\newblock In \emph{International Conference on Machine Learning}, pp.\
  214--223, 2017.

\bibitem[Arvanitidis et~al.(2017)Arvanitidis, Hansen, and
  Hauberg]{arvanitidis2017latentspace}
Arvanitidis, G., Hansen, L.~K., and Hauberg, S.
\newblock Latent space oddity: on the curvature of deep generative models.
\newblock In \emph{ICLR}, 2017.

\bibitem[Azadi et~al.(2019)Azadi, Olsson, Darrell, Goodfellow, and
  Odena]{azadi2018discriminator}
Azadi, S., Olsson, C., Darrell, T., Goodfellow, I., and Odena, A.
\newblock Discriminator rejection sampling.
\newblock In \emph{International Conference on Learning Representations}, 2019.

\bibitem[Biau \& Devroye(2015)Biau and Devroye]{biau2015lectures}
Biau, G. and Devroye, L.
\newblock \emph{Lectures on the nearest neighbor method}.
\newblock Springer, 2015.

\bibitem[Biau et~al.(2018)Biau, Cadre, Sangnier, and Tanielian]{biau2018some}
Biau, G., Cadre, B., Sangnier, M., and Tanielian, U.
\newblock Some theoretical properties of gans.
\newblock \emph{arXiv:1803.07819}, 2018.

\bibitem[Borell(1975)]{borell1975brunn}
Borell, C.
\newblock The brunn-minkowski inequality in gauss space.
\newblock \emph{Inventiones mathematicae}, 30\penalty0 (2):\penalty0 207--216,
  1975.

\bibitem[Boucheron et~al.(2013)Boucheron, Lugosi, and Massart]{boucheron2013}
Boucheron, S., Lugosi, G., and Massart, P.
\newblock \emph{Concentration Inequalities: A Nonasymptotic Theory of
  Independence}.
\newblock OUP Oxford, 2013.
\newblock ISBN 9780199535255.

\bibitem[Brock et~al.(2019)Brock, Donahue, and Simonyan]{brock2018large}
Brock, A., Donahue, J., and Simonyan, K.
\newblock Large scale {GAN} training for high fidelity natural image synthesis.
\newblock In \emph{International Conference on Learning Representations}, 2019.

\bibitem[Deng et~al.(2009)Deng, Dong, Socher, Li, Li, and
  Fei-Fei]{deng2009imagenet}
Deng, J., Dong, W., Socher, R., Li, L.-J., Li, K., and Fei-Fei, L.
\newblock Imagenet: A large-scale hierarchical image database.
\newblock In \emph{2009 IEEE conference on computer vision and pattern
  recognition}, pp.\  248--255. Ieee, 2009.

\bibitem[Devroye \& Wise(1980)Devroye and Wise]{devroye1980detection}
Devroye, L. and Wise, G.
\newblock Detection of abnormal behavior via nonparametric estimation of the
  support.
\newblock \emph{SIAM Journal on Applied Mathematics}, 38:\penalty0 480--488,
  1980.

\bibitem[Dowson \& Landau(1982)Dowson and Landau]{dowson1982frechet}
Dowson, D. and Landau, B.
\newblock The fr{\'e}chet distance between multivariate normal distributions.
\newblock \emph{Journal of multivariate analysis}, pp.\  450--455, 1982.

\bibitem[Dudley(2002)]{dudley_2002}
Dudley, R.~M.
\newblock \emph{Real Analysis and Probability}.
\newblock Cambridge Studies in Advanced Mathematics. Cambridge University
  Press, 2 edition, 2002.
\newblock \doi{10.1017/CBO9780511755347}.

\bibitem[Fefferman et~al.(2016)Fefferman, Mitter, and
  Narayanan]{fefferman2016testing}
Fefferman, C., Mitter, S., and Narayanan, H.
\newblock Testing the manifold hypothesis.
\newblock \emph{Journal of the American Mathematical Society}, 2016.

\bibitem[Goodfellow et~al.(2014)Goodfellow, Pouget-Abadie, Mirza, Xu,
  Warde-Farley, Ozair, Courville, and Bengio]{GANs}
Goodfellow, I., Pouget-Abadie, J., Mirza, M., Xu, B., Warde-Farley, D., Ozair,
  S., Courville, A., and Bengio, J.
\newblock Generative adversarial nets.
\newblock In \emph{Advances in Neural Information Processing Systems 27}, pp.\
  2672--2680. 2014.

\bibitem[Gretton et~al.(2012)Gretton, Borgwardt, Rasch, Sch{\"o}lkopf, and
  Smola]{gretton2012kernel}
Gretton, A., Borgwardt, K.~M., Rasch, M.~J., Sch{\"o}lkopf, B., and Smola, A.
\newblock A kernel two-sample test.
\newblock \emph{Journal of Machine Learning Research}, pp.\  723--773, 2012.

\bibitem[Gulrajani et~al.(2017)Gulrajani, Ahmed, Arjovsky, Dumoulin, and
  Courville]{gulrajani2017improved}
Gulrajani, I., Ahmed, F., Arjovsky, M., Dumoulin, V., and Courville, A.~C.
\newblock Improved training of wasserstein gans.
\newblock In \emph{Advances in Neural Information Processing Systems}, pp.\
  5767--5777, 2017.

\bibitem[Gurumurthy et~al.(2017)Gurumurthy, Kiran~Sarvadevabhatla, and
  Venkatesh~Babu]{gurumurthy2017deligan}
Gurumurthy, S., Kiran~Sarvadevabhatla, R., and Venkatesh~Babu, R.
\newblock Deligan: Generative adversarial networks for diverse and limited
  data.
\newblock In \emph{Proceedings of the IEEE Conference on Computer Vision and
  Pattern Recognition}, 2017.

\bibitem[Hales(2001)]{hales2001honeycomb}
Hales, T.~C.
\newblock The honeycomb conjecture.
\newblock \emph{Discrete \& Computational Geometry}, pp.\  1--22, 2001.

\bibitem[Heusel et~al.(2017)Heusel, Ramsauer, Unterthiner, Nessler, and
  Hochreiter]{heusel2017gans}
Heusel, M., Ramsauer, H., Unterthiner, T., Nessler, B., and Hochreiter, S.
\newblock Gans trained by a two time-scale update rule converge to a local nash
  equilibrium.
\newblock In \emph{Advances in Neural Information Processing Systems}, pp.\
  6626--6637, 2017.

\bibitem[Kallenberg(2006)]{kallenberg2006foundations}
Kallenberg, O.
\newblock \emph{Foundations of modern probability}.
\newblock Springer Science \& Business Media, 2006.

\bibitem[Karras et~al.(2019)Karras, Laine, and Aila]{karras2019style}
Karras, T., Laine, S., and Aila, T.
\newblock A style-based generator architecture for generative adversarial
  networks.
\newblock In \emph{Proceedings of the IEEE Conference on Computer Vision and
  Pattern Recognition}, pp.\  4401--4410, 2019.

\bibitem[Khayatkhoei et~al.(2018)Khayatkhoei, Singh, and
  Elgammal]{khayatkhoei2018disconnected}
Khayatkhoei, M., Singh, M.~K., and Elgammal, A.
\newblock Disconnected manifold learning for generative adversarial networks.
\newblock In \emph{Advances in Neural Information Processing Systems}, pp.\
  7343--7353, 2018.

\bibitem[Krizhevsky et~al.(2009)Krizhevsky, Hinton,
  et~al.]{krizhevsky2009learning}
Krizhevsky, A., Hinton, G., et~al.
\newblock Learning multiple layers of features from tiny images.
\newblock 2009.

\bibitem[Kynk{\"a}{\"a}nniemi et~al.(2019)Kynk{\"a}{\"a}nniemi, Karras, Laine,
  Lehtinen, and Aila]{kynkaanniemi2019improved}
Kynk{\"a}{\"a}nniemi, T., Karras, T., Laine, S., Lehtinen, J., and Aila, T.
\newblock Improved precision and recall metric for assessing generative models.
\newblock In \emph{Advances in Neural Information Processing Systems}, pp.\
  3929--3938, 2019.

\bibitem[LeCun et~al.(1998)LeCun, Bottou, Bengio, and
  Haffner]{lecun98gradientbasedlearning}
LeCun, Y., Bottou, L., Bengio, Y., and Haffner, P.
\newblock Gradient-based learning applied to document recognition.
\newblock In \emph{Proceedings of the IEEE}, pp.\  2278--2324, 1998.

\bibitem[Ledoux(1996)]{ledoux1996isoperimetry}
Ledoux, M.
\newblock Isoperimetry and gaussian analysis.
\newblock In \emph{Lectures on probability theory and statistics}, pp.\
  165--294. Springer, 1996.

\bibitem[Lin et~al.(2018)Lin, Khetan, Fanti, and Oh]{lin2018pacgan}
Lin, Z., Khetan, A., Fanti, G., and Oh, S.
\newblock Pacgan: The power of two samples in generative adversarial networks.
\newblock In \emph{Advances in Neural Information Processing Systems}, pp.\
  1498--1507, 2018.

\bibitem[Miyato et~al.(2018)Miyato, Kataoka, Koyama, and
  Yoshida]{spectral_normGANs}
Miyato, T., Kataoka, T., Koyama, M., and Yoshida, Y.
\newblock Spectral normalization for generative adversarial networks.
\newblock In \emph{International Conference on Learning Representations}, 2018.

\bibitem[Pandeva \& Schubert(2019)Pandeva and Schubert]{pandeva2019mmgan}
Pandeva, T. and Schubert, M.
\newblock Mmgan: Generative adversarial networks for multi-modal distributions.
\newblock \emph{arXiv:1911.06663}, 2019.

\bibitem[Petzka et~al.(2018)Petzka, Fischer, and
  Lukovnikov]{petzka2017regularization}
Petzka, H., Fischer, A., and Lukovnikov, D.
\newblock On the regularization of wasserstein {GAN}s.
\newblock In \emph{International Conference on Learning Representations}, 2018.

\bibitem[Rifai et~al.(2011)Rifai, Mesnil, Vincent, Muller, Bengio, Dauphin, and
  Glorot]{rifai2011higher}
Rifai, S., Mesnil, G., Vincent, P., Muller, X., Bengio, Y., Dauphin, Y., and
  Glorot, X.
\newblock Higher order contractive auto-encoder.
\newblock In \emph{Joint European Conference on Machine Learning and Knowledge
  Discovery in Databases}, pp.\  645--660. Springer, 2011.

\bibitem[Roth et~al.(2017)Roth, Lucchi, Nowozin, and
  Hofmann]{roth2017stabilizing}
Roth, K., Lucchi, A., Nowozin, S., and Hofmann, T.
\newblock Stabilizing training of generative adversarial networks through
  regularization.
\newblock In \emph{Advances in Neural Information Processing Systems}, pp.\
  2018--2028, 2017.

\bibitem[Sajjadi et~al.(2018)Sajjadi, Bachem, Lucic, Bousquet, and
  Gelly]{sajjadi2018assessing}
Sajjadi, M.~S., Bachem, O., Lucic, M., Bousquet, O., and Gelly, S.
\newblock Assessing generative models via precision and recall.
\newblock In \emph{Advances in Neural Information Processing Systems}, pp.\
  5228--5237, 2018.

\bibitem[Salimans et~al.(2016)Salimans, Goodfellow, Zaremba, Cheung, Radford,
  and Chen]{salimans2016improved}
Salimans, T., Goodfellow, I., Zaremba, W., Cheung, V., Radford, A., and Chen,
  X.
\newblock Improved techniques for training gans.
\newblock In \emph{Advances in Neural Information Processing Systems}, pp.\
  2234--2242, 2016.

\bibitem[Srivastava et~al.(2017)Srivastava, Valkov, Russell, Gutmann, and
  Sutton]{srivastava2017veegan}
Srivastava, A., Valkov, L., Russell, C., Gutmann, M.~U., and Sutton, C.
\newblock Veegan: Reducing mode collapse in gans using implicit variational
  learning.
\newblock In \emph{Advances in Neural Information Processing Systems}, pp.\
  3308--3318, 2017.

\bibitem[Sudakov \& Tsirelson(1978)Sudakov and Tsirelson]{sudakov1978extremal}
Sudakov, V.~N. and Tsirelson, B.~S.
\newblock Extremal properties of half-spaces for spherically invariant
  measures.
\newblock \emph{Journal of Mathematical Sciences}, pp.\  9--18, 1978.

\bibitem[Tolstikhin et~al.(2017)Tolstikhin, Gelly, Bousquet, Simon-Gabriel, and
  Sch{\"o}lkopf]{tolstikhin2017adagan}
Tolstikhin, I.~O., Gelly, S., Bousquet, O., Simon-Gabriel, C.-J., and
  Sch{\"o}lkopf, B.
\newblock Adagan: Boosting generative models.
\newblock In \emph{Advances in Neural Information Processing Systems}, pp.\
  5424--5433, 2017.

\bibitem[Virmaux \& Scaman(2018)Virmaux and Scaman]{virmaux2018lipschitz}
Virmaux, A. and Scaman, K.
\newblock Lipschitz regularity of deep neural networks: analysis and efficient
  estimation.
\newblock In \emph{Advances in Neural Information Processing Systems}, pp.\
  3835--3844, 2018.

\bibitem[Xiao et~al.(2017)Xiao, Rasul, and Vollgraf]{xiao2017/online}
Xiao, H., Rasul, K., and Vollgraf, R.
\newblock Fashion-mnist: a novel image dataset for benchmarking machine
  learning algorithms.
\newblock 2017.

\bibitem[Yu et~al.(2017)Yu, Zhang, Wang, and Yu]{yu2017seqgan}
Yu, L., Zhang, W., Wang, J., and Yu, Y.
\newblock Seqgan: Sequence generative adversarial nets with policy gradient.
\newblock In \emph{Thirty-First AAAI Conference on Artificial Intelligence},
  2017.

\bibitem[Zhang et~al.(2019)Zhang, Goodfellow, Metaxas, and
  Odena]{zhang2019self}
Zhang, H., Goodfellow, I., Metaxas, D., and Odena, A.
\newblock Self-attention generative adversarial networks.
\newblock In Chaudhuri, K. and Salakhutdinov, R. (eds.), \emph{Proceedings of
  the 36th International Conference on Machine Learning}, volume~97 of
  \emph{Proceedings of Machine Learning Research}, pp.\  7354--7363, 2019.

\bibitem[Zhang et~al.(2018)Zhang, Liu, Zhou, Xu, and He]{ZhLiZhXuHe18}
Zhang, P., Liu, Q., Zhou, D., Xu, T., and He, X.
\newblock On the discriminative-generalization tradeoff in {GAN}s.
\newblock In \emph{International Conference on Learning Representations}, 2018.

\bibitem[Zhong et~al.(2019)Zhong, Mo, Xiao, Chen, and
  Zheng]{zhong2019rethinking}
Zhong, P., Mo, Y., Xiao, C., Chen, P., and Zheng, C.
\newblock Rethinking generative mode coverage: A pointwise guaranteed approach.
\newblock In \emph{Advances in Neural Information Processing Systems}, pp.\
  2086--2097, 2019.

\end{thebibliography}
\appendix
\clearpage
\section{Highlighting drawbacks of the PR metric by \citet{sajjadi2018assessing}} 
\label{appendix:sajjadi_metric}

\begin{lem}
    Assume that the modeled distribution $\mu_\theta$ slightly collapses on a specific data point, i.e. there exists $x \in E, \mu_\theta(x)>0$. Assume also that $\mu_\star$ is a continuous probability measure and that $\mu_\theta$ has a recall $\beta=1$. Then the precision must be such that $\alpha=0$. 
    \begin{proof}
        Using Definition \ref{def:prec_rec_metric}, we have that there exists $\mu$ such that
        \begin{equation*}
            \mu_\star= \alpha \mu + (1 - \alpha) \nu_{\mu_\star} \quad \text{and} \quad \mu_\theta = \mu
    \end{equation*}
    Thus, $0=\mu_\star(x) \geqslant \alpha \mu(x) = \alpha \mu_\theta(x)$. Which implies that $\alpha=0$.
    \end{proof}
\end{lem}

\section{Proof of Theorem \ref{th:improved_prec_rec}} \label{appendix:th:comparison_metrics}
The proof of Theorem \ref{th:improved_prec_rec} relies on theoretical results from non-parametric estimation of the supports of probability distribution studied by \citet{devroye1980detection}. 

For the following proofs, we will require the following notation: let $\varphi$ be a strictly monotonous function be such that $\underset{n \to \infty}{\lim} \ \frac{\varphi(n)}{n} = 0$ and $\underset{n \to \infty}{\lim} \ \frac{\varphi(n)}{\log(n)} = \infty$. We note $B(x, r) \subseteq E$, the open ball centered in $x$ and of radius $r$. For a given probability distribution $\mu$, $S_\mu$ refers to its support. We recall that for any $x$ in a dataset $D$, $x_{(k)}$ denotes its $k$ nearest neighbor in $D$. Finally, for a given probability distribution $\mu$ and a dataset $D_\mu$ sampled from $\mu^n$, we note $R_{\text{min}}$ and $R_{\text{max}}$ the following:
\begin{equation}\label{eq:rmin_rmax}
    R_{\text{min}} = \underset{x \in E}{\min} \|x-x_{(\varphi(n))}\|, \ \   R_{\text{max}} = \underset{x \in E}{\max} \|x-x_{(\varphi(n))}\|
\end{equation}

In the following lemma, we show asymptotic behaviours for both $R_{\text{min}}$ and $R_{\text{max}}$.
\begin{lem}\label{lem:auxiliary_lemma}
    Let $\mu$ be a probability distribution associated with a uniformly continuous probability density function $f_\mu$. Assume that there exists constants $a_1>0, a_2>0$ such that for all $x \in E$, we have $a_1 < f_\mu(x) \leqslant a_2$. Then,  
    \begin{align*}
        R_{\text{min}} \underset{n \to \infty}{\longrightarrow} 0 \ \text{a.s.} \quad \text{and} \quad R_{\text{min}}^d \underset{n \to \infty}{\longrightarrow} \infty \ \text{a.s.} \\
        R_{\text{max}} \underset{n \to \infty}{\longrightarrow} 0 \ \text{a.s.} \quad \text{and} \quad R_{\text{max}}^d \underset{n \to \infty}{\longrightarrow} \infty \ \text{a.s.}         
    \end{align*}
    \begin{proof}
        We will only prove that $R_{\text{max}} \underset{n \to \infty}{\longrightarrow} 0 \ \text{a.s.}$ and $\text{and} \quad R_{\text{min}}^d \underset{n \to \infty}{\longrightarrow} \infty \ \text{a.s.}$ as the rest follows. 
        
        The result is based on a nearest neighbor result from \citet[][]{biau2015lectures}. Considering the $\varphi(n)$ nearest neighbor density estimate $f_n^{\varphi(n)}$ based on a finite sample dataset $D_\mu$, Theorem 4.2 states that if $f_\mu$ is uniformly continuous then:
        \begin{equation*}
            \underset{x \in E}{\sup} \ \|f_n^{\varphi(n)}(x) - f_\mu(x) \| \to 0.
        \end{equation*}
        where $f_n^{\varphi(n)}(x) = \frac{\varphi(n)}{n V_d \|x-x_{\varphi(n)}\|^d}$ with $V_d$ being the volume of the unit ball in $\mathds{R}^d$.
        
        Let $\varepsilon>0$ such that $\varepsilon<a_1/2$. There exists $N \in \mathds{N}$ such that for all $n \geqslant N$, we have, almost surely, for all $x \in E$:
        \begin{align*}
            a_1 - \varepsilon \leqslant f_n^{\varphi(n)}(x) \leqslant a_2 + \varepsilon \\
            a_1 - \varepsilon \leqslant \frac{\varphi(n)}{n V_d \|x-x_{\varphi(n)}\|^d} \leqslant a_2 + \varepsilon
        \end{align*}
        
        Consequently, for all $n \geqslant N$, for all $x \in E$ almost surely:
        \begin{align*}
            \|x-x_{\varphi(n)}\| \leqslant \Big( \frac{\varphi(n)}{n V_d (a_1-\varepsilon)}\Big)^{1/d} \\
            \text{Thus  ,} \underset{x \in E}{\sup} \ \|x-x_{\varphi(n)}\| \to 0 \quad \text{a.s.}
        \end{align*}
        
        Also, almost surely
        \begin{align*}
            n \|x-x_{\varphi(n)}\|^d \geqslant \frac{\varphi(n)}{V_d (a_2+\varepsilon)} \\
            \text{Thus,} \quad \underset{x \in E}{\inf} \ \|x-x_{\varphi(n)}\| \to \infty \quad \text{a.s.}
        \end{align*}
    \end{proof}
\end{lem}

\begin{lem}\label{lem:lemma_appendix}
    Let $\mu,\nu$ be two probability distributions associated with uniformly continuous probability density functions $f_\mu$ and $f_\nu$. Assume that there exists constants $a_1>0, a_2>0$ such that for all $x \in E$, we have $a_1 < f_\mu(x) \leqslant a_2 $ and $a_1 < f_\nu \leqslant a_2$. Also, let $D_\mu, D_\nu$ be datasets sampled from $\nu^n, \mu^n$. If $\mu$ is an estimator for $\nu$, then 
    \begin{align*}
        (i) \text{  for all } x \in D_\mu, \ \alpha_{\varphi(n)}^n(x) &\underset{n \to \infty}{\rightarrow} \mathds{1}_{\emph{supp}(\nu)}(x) \quad \text{in proba.}\\ 
        (ii) \text{  for all } y \in D_\nu, \ \beta_{\varphi(n)}^n(y) &\underset{n \to \infty}{\rightarrow} \mathds{1}_{\emph{supp}(\mu)}(x) \quad \text{in proba.}
    \end{align*}
    
    \begin{proof}
    We will only show the result for $(i)$, since a similar proof holds for $(ii)$. 
    
    Thus, we want to show that 
    \begin{equation*}
        \text{  for all } x \in D_\mu, \ \alpha_{\varphi(n)}^n(x) \underset{n \to \infty}{\rightarrow} \mathds{1}_{\emph{supp}(\nu)}(x) \quad \text{a. s.}
    \end{equation*}
    
    First, let's assume that $x \notin S_\nu$. \citet[][Lemma  2.2]{biau2015lectures} have shown that 
    \begin{equation*}
        \underset{n \to \infty}{\lim} \|x_{(\varphi(n))}-x\| = \inf \{\|x-y\| \mid y \in S_\nu \} \quad \text{a.s.}
    \end{equation*}
    As $S_\nu$ is a closed set - e.g. \cite{kallenberg2006foundations} - we have 
    \begin{equation*}
        \underset{n \to \infty}{\lim} \ \|x-x_{(\varphi(n))}\| > 0 \quad \text{a.s.}
    \end{equation*}
    and 
    \begin{equation*}
        \text{for all } y \in D_\nu, \underset{{n} \to \infty}{\lim} \ \|y-y_{({\varphi(n)})}\| = 0 \quad \text{a.s.}
    \end{equation*}
    Thus, $\underset{n \to \infty}{\lim} \alpha_{\varphi(n)}^n(x) = 0 \quad \text{a.s.}$.    
    
    Now, let's assume that $x \in S_\nu$. Using Definition \ref{def:improved_prec_rec}, the precision of a given data point $x$ can be rewritten as follows:
    \begin{equation*}
        \alpha_{\varphi(n)}^n(x) = 1 \iff \exists y \in D_\nu , x \in B(y, \|y-y_{(\varphi(n))}\|) 
    \end{equation*}
    Using notation from \eqref{eq:rmin_rmax}, we note
    \begin{equation*}
        R_{\text{min}} = \underset{y \in }{\min} \|y-y_{(\varphi(n))}\|, \ \   R_{\text{max}} = \underset{y \in E}{\max} \|y-y_{(\varphi(n))}\|.
    \end{equation*}
    It is clear that :
    \begin{equation}\label{eq:bounding_s_mu_n}
         \bigcup_{y \in D_\nu} B(y, R_{\text{min}}) \subseteq S_\nu^n \subseteq \bigcup_{y \in D_\nu} B(y, R_{\text{max}}),
    \end{equation}
    where $S_\nu^n = \bigcup_{y \in D_\nu} B(y, \|y-y_{(\varphi(n))}\|))$. 
    
    Besides, combining Lemma \ref{lem:auxiliary_lemma} with \citet[][Theorem 1]{devroye1980detection}, we have that:
    \begin{align*}
        \nu(S_\nu \Delta \bigcup_{y \in D_\nu} B(y, R_{\text{min}})) \underset{n \to 0}{\longrightarrow} 0 \quad \text{in proba.} \\
        \nu(S_\nu \Delta \bigcup_{y \in D_\nu} B(y, R_{\text{max}})) \underset{n \to 0}{\longrightarrow} 0 \quad \text{in proba.} \\
    \end{align*}
    where $\Delta$ here refers to the symmetric difference.
    
    Thus, using \eqref{eq:bounding_s_mu_n}, it is now clear that, $\mu(S_\nu \Delta S_\nu^n) \to 0$ in probability. Finally, given $x \in S_\mu$, we have $\mu(x \in S_\nu^n) = \nu(\alpha_{\varphi(n)}^n(x) = 1) \to 1 $ in probability. 
    \end{proof}
\end{lem}

We can now finish the proof for Theorem \ref{th:improved_prec_rec}. Recall that $\bar{\alpha} = \mu \big(S_\nu\big)$ and similarly, $\bar{\beta} = \nu\big(S_\mu\big)$.
\begin{proof}
    We have that
    \begin{equation*}
        |\alpha_{\varphi(n)}^n - \bar{\alpha}| = |\frac{1}{n} \sum_{x_i \in D_\mu} \alpha_{\varphi(n)}^n(x_i) - \int_E \mathds{1}_{x \in S_\nu} \mu({\rm d} x)| \nonumber \\
    \end{equation*}
    Then,
    \begin{align}
        |\alpha_{\varphi(n)}^n - \bar{\alpha}| &= |\frac{1}{n} \sum_{x_i \in D_\mu} (\alpha_{\varphi(n)}^n(x_i) - \mathds{1}_{x_i \in S_\nu}) \nonumber \\
        \quad + \big(\frac{1}{n} & \sum_{x_i \in D_\mu}  \mathds{1}_{x_i \in S_\nu} - \int_E \mathds{1}_{x \in S_\nu} \mu({\rm d}x) \big)| \nonumber \\
        &=  |\mathds{E}_{x_i \sim \mu_n} (\alpha_{\varphi(n)}^n(x_i) - \mathds{1}_{x_i \in S_\nu}) \label{eq:kolmogorov} \\
        & \quad + \big(\mathds{E}_{\mu_n} \mathds{1}_{S_\nu} - \mathds{E}_{\mu} \mathds{1}_{S_\nu} \big) \label{eq:vadarajan} |
    \end{align}
    where $\mu_n$ is the empirical distribution of $\mu$. As $\mu_n$ converges weakly to $\mu$ almost surely (e.g. \citet[][Theorem 11.4.1]{dudley_2002}) and since $\mathds{1}_{x \in  S_\nu}$ is bounded, we can bound \eqref{eq:vadarajan} as follows:
    \begin{equation*}
        \underset{n \to \infty}{\lim} \quad \mathds{E}_{x \sim \mu_n} \mathds{1}_{x \in \text{supp}(\mu)} - \mathds{E}_{x \sim \mu} \mathds{1}_{x \in \text{supp}(\mu)}  = 0 \quad \text{a. s.}
    \end{equation*}
    Now, to bound \eqref{eq:kolmogorov}, we use the fact that for any $x \in D_\mu$, the random variable $\alpha_{\varphi(n)}^n(x)$ converges to $\mathds{1}_{x \in S_\nu}$ in probability (Lemma \ref{lem:lemma_appendix}) and that for all $x \in D_\mu$, both $\alpha_{\varphi(n)}^n(x) \leqslant 1$ and $\mathds{1}_{x \in S_\nu} \leqslant 1$. Consequently, using results from the weak law for triangular arrays, we have that 
    \begin{equation*}
        \underset{n \to \infty}{\lim} \quad \frac{1}{n} \sum_{x_i \in D_\mu} (\alpha_{\varphi(n)}^n(x_i) - \mathds{1}_{x_i \in S_\nu}) = 0 \quad \text{in proba.}
    \end{equation*}
    Finally, 
    \begin{equation*}
        |\alpha_{\varphi(n)}^n - \bar{\alpha}| \underset{n \to \infty}{\to} 0 \quad \text{in proba.},
    \end{equation*}
    which proves the result. The same proof works for $\underset{k \to \infty}{\lim} \ \beta_k^n = \bar{\beta}$.
\end{proof}

\section{Proof of Theorem \ref{th:no_free_lunch}} \label{appendix:proof_th1}

This proof is based on the Gaussian isoperimetric inequality historically shown by \cite{borell1975brunn, sudakov1978extremal}.

\begin{proof}
    Let $\mu_\star$ be a distribution defined on $E$ laying on two disconnected manifolds $M_1$ and $M_2$ such that $\mu_\star(M_1) = \mu_\star(M_2) = \frac{1}{2}$ and $d(M_1, M_2) = D$. Note that for any subsets $A \subseteq E$ and $B \subseteq E$, $d(A,B) := \underset{(x,y) \in A \times B}{\inf} \|x-y\|$.
    
    Let $G_\theta^{-1}(M_1)$ (respectively $G_\theta^{-1}(M_2)$ be the subset in $\mathds{R}^d$ be the pre-images of $M_1$ (respectively $M_2$). 
    
    Consequently, we have for all $k \in [1,n]$
    \begin{equation*}
        \gamma(G_\theta^{-1}(M_1)) = \mu_\theta(M_1) = \gamma(G_\theta^{-1}(M_2)) \geqslant \frac{\bar{\alpha}}{2}
    \end{equation*}
    
    We consider $(G_\theta^{-1}(M_1))^\varepsilon$ (respectively $(G_\theta^{-1}(M_2))^\varepsilon$) the $\varepsilon$ enlargement of $G_\theta^{-1}(M_1)$ (respectively $G_\theta^{-1}(M_2)$ where $\varepsilon = \frac{D}{2L}$. We know that $(G_\theta^{-1}(M_1))^\varepsilon \bigcap (G_\theta^{-1}(M_2))^\varepsilon = \emptyset$. 
    
    Thus, we have that: 
    \begin{equation*}
        \gamma \big((G_\theta^{-1}(M_1))^\varepsilon \big) + \gamma \big((G_\theta^{-1}(M_2))^\varepsilon \big) \leqslant 1
    \end{equation*}
    Besides, by denoting $\Phi$ the function defined for any $t \in \mathds{R}$ by $\Phi(t) = \int_{-\infty}^{t} \frac{\exp(-t^2/2)}{\sqrt{2\pi}} ds$, we have
    
\begin{align*}
    &\gamma \big((G_\theta^{-1}(M_1))^\varepsilon \big) + \gamma \big((G_\theta^{-1}(M_2))^\varepsilon \big) \geqslant 2 \Phi \big(\Phi^{-1}(\frac{\alpha}{2}) + \varepsilon \big) \\
    & \quad \text{(using Theorem 1.3 from \cite{ledoux1996isoperimetry})}\\
    & \quad \geqslant \alpha + \frac{2\varepsilon}{\sqrt{2\pi}} e^{-\Phi^{-1}(\frac{\alpha}{2})^2/2} \\
    &\text{(since $\Phi^{-1}(\frac{\alpha}{2})+\varepsilon<0$ and $\Phi$ convex on $]-\infty, 0]$)}
\end{align*}

Thus, we have that 
\begin{equation*}
    \alpha + \frac{2\varepsilon}{\sqrt{2\pi}} e^{-\Phi^{-1}(\frac{\alpha}{2})^2/2} \leqslant 1
\end{equation*}

Thus, by noting 
\begin{equation*}
    \alpha^\star = \sup \{\alpha \in [0,1] \mid \alpha + \frac{2\varepsilon}{\sqrt{2\pi}} e^{\frac{-\Phi^{-1}(\frac{\alpha}{2})^2}{2}} \leqslant 1 \},
\end{equation*}
we have our result.

For $\alpha \geqslant 3/4$. By noting $\alpha = 1-x$, we have
\begin{align*}
    \Phi^{-1}(\frac{\alpha}{2}) &= \frac{\sqrt{2 \pi} x}{2} + O(x^3)\\
    \text{And,} \ e^{\frac{-\Phi^{-1}(\frac{\alpha}{2})^2}{2}} &= e^{\frac{-\pi x^2}{4}} + O(e^{-x^4}) \\
    \text{Thus,} \ 1-x+\frac{2\varepsilon}{\sqrt{2\pi}}e^{\frac{-\pi x^2}{4}} &+ O(e^{-x^4}) \leqslant 1 \\
    \iff x &\geqslant \frac{2\varepsilon}{\sqrt{2\pi}} e^{\frac{-\pi x^2}{4}} + O(e^{-x^4}) \\
    \implies x &\geqslant \sqrt{\frac{2}{\pi}} W(\epsilon^2)
\end{align*}
where $W$ is the product log function. Thus, $\alpha \leqslant 1 - \sqrt{\frac{2}{\pi}} W(\epsilon^2)$.
\end{proof}
As an example, in the case where $\varepsilon=1$, we have that $W(1) \approx 0.5671 $, $x>0.4525$ and $\alpha < 0.5475 $.

\section{Proof of Theorem \ref{th:extended_no_free_lunch}}\label{appendix:proof_theorem3}
\subsection{Equitable setting}
This result is a consequence of Prop.  \ref{appendix:thm_main_result} that we will assume true in this section. 

We consider that the unknown true distribution $\mu_\star$ lays on $M$ disjoint manifolds of equal measure. As specified in Section \ref{section:approach}, the latent distribution $\gamma$ is a multivariate Gaussian defined on $\mathds{R}^d$. For each $k \in [1,M]$, we consider in the latent space, the pre-images $A_k$. 

It is clear that $A_1,\ldots,A_M$ are pairwise disjoint Borel subsets of $\mathds{R}^d$. We denote $\bar{M}$, the number of classes covered by the estimator $\mu_\theta$, such that for all $i \in [1, \bar{M}]$, we have $\gamma(A_i) >0$. We know that $\bar{M} \geqslant M \bar{\beta} > 1$.

For each $i \in [1,\bar{M}]$, we denote $A_i^\varepsilon$, the $\varepsilon$-enlargement of $A_i$. For any pair $(i,j)$ it is clear that $A_i^\varepsilon \bigcap A_j^\varepsilon = 0$ where $\varepsilon = \frac{D}{2L}$ ($D$ being the minimum distance between two sub-manifolds and $L$ being the Lipschitz constant of the generator). 

As assumed, we know that $A_i^\varepsilon, i \in [1, \bar{M}]$ partition the latent space in equal measure, consequently, we assume that 
\begin{equation}\label{appendix:assumption_equal_partition_Gaussian_space}
    \sum_{i=1}^n \gamma(A_i^\varepsilon) = 1 \quad \text{and} \quad \gamma(A_1) = \hdots = \gamma(A_{\bar{M}}) = 1/\bar{M}
\end{equation}

Thus, we have that
\begin{equation*}
    \bar{\alpha} = \sum_{i=1}^{\bar{M}} \gamma(A_i^\varepsilon) = 1 - \gamma(\Delta^{-\varepsilon} (A_1^\varepsilon,\ldots,A_{\bar{M}}^\varepsilon))
\end{equation*}

Using Proposition \ref{appendix:thm_main_result}, we have 
\begin{align*}
    \gamma(\Delta^{-\varepsilon}(A_1^\varepsilon, \ldots,A_{n}^\varepsilon)) &\geqslant 1-\frac{1+x^2}{x^2}e^{-\frac{1}{2}\varepsilon^2}e^{-\varepsilon x} \\
    \text{Thus, } \bar{\alpha} &\leqslant \frac{1+y^2}{y^2}e^{-\frac{1}{2}\varepsilon^2}e^{-\varepsilon y}
\end{align*}
where $y= \Phi^{-1}\left(1-\max_{k \in [\!\bar{M}\!]} \gamma(A_k^\varepsilon) \right) = \Phi^{-1}(\frac{\bar{M}-1}{\bar{M}})$ and $\Phi(t) = \int_{-\infty}^{t} \frac{\exp(-t^2/2)}{\sqrt{2\pi}} ds$.

Knowing that $\bar{M} \geqslant \bar{\beta}M$ we have that
\begin{equation*}
    \Phi^{-1}(1-\frac{1}{\bar{M}}) \geqslant \Phi^{-1} (1-\frac{1}{\bar{\beta}M})    
\end{equation*}

We conclude by saying that the function $x \mapsto \frac{1 + x^2}{x^2}e^{-\varepsilon x}$ is decreasing for $x > 0$. Thus,
\begin{equation}\label{appendix:final_result_theorem3}
    \bar{\alpha} \leqslant \frac{1+y^2}{y^2}e^{-\frac{1}{2}\varepsilon^2}e^{-\varepsilon y}
\end{equation}
where $y = \Phi^{-1} (1-\frac{1}{\bar{\beta}M})$ and $\Phi(t) = \int_{-\infty}^{t} \frac{\exp(-t^2/2)}{\sqrt{2\pi}} ds$.

For further analysis, when $\bar{M} \to \infty$, refer to subsection \ref{appendix:main_result_equitable_scenario} and note using the result in \eqref{appendix:eq:explicit} that one obtains the desired upper-bound on $\bar{\alpha}$
\begin{equation*}
    \bar{\alpha} \overset{\bar{M} \rightarrow \infty}{\leqslant} e^{-\frac{1}{2}\varepsilon^2}e^{-\varepsilon\sqrt{2\log(\bar{M}})}
\end{equation*}
 
\subsection{More general setting}\label{appendix:theorem3_more_general_setting}
As done previously, we denote $\bar{M}$, the number of classes covered by the estimator $\mu_\theta$, such that for all $i \in [1, \bar{M}]$, we have $\gamma(A_i) >0$. We still assume that $\bar{M}>1$. However, we now relax the previous assumption made in \eqref{appendix:assumption_equal_partition_Gaussian_space} and  assume the milder assumption that there exists $w_1, \hdots, w_M \in [0,1]^M$ such that for all $m \in [1,M], \gamma(A_m^\varepsilon) = w_m$, $\sum_m w_m \leqslant 1$ and $\underset{i \in [1, M]}{\max} \ w_m = w^\text{max}<1$. 

Consider, $A^\complement = \Big(\bigcup_{i=1}^{\bar{M}} A_i^\varepsilon \Big)^\complement$ and denote $w^c = \gamma(A^\complement)\leqslant 1-\bar{\alpha}$. Consequently, we have
\begin{align*}
    &\sum_{i=1}^n \gamma(A_i^\varepsilon) + \gamma(A^\complement) = 1 \\
    &\gamma(\Delta^{-\varepsilon}(A_1^\varepsilon,\ldots,A_{M}^\varepsilon, A^\complement)) + \sum_{i=1}^{M} \gamma(A_i^\varepsilon)
    = 1 - \gamma(A^\complement) \\ 
    &\bar{\alpha} = 1 - w^\complement - \gamma(\Delta^{-\varepsilon} (A_1^\varepsilon,\ldots,A_{M}^\varepsilon, A^\complement))
\end{align*}

In this setting, it is clear that $A_1, \hdots, A_{\bar{M}}, A^\complement$ is a a partition of $\mathds{R}^d$ under the measure $\gamma$. Using, result from Proposition \ref{appendix:thm_main_result}, we have 
\begin{equation*}
    \gamma(\Delta^{-\varepsilon} (A_1^\varepsilon,\ldots,A_{M}^\varepsilon, A^\complement)) \geqslant 1-\frac{1+x^2}{x^2}e^{-\frac{1}{2}\varepsilon^2}e^{-\varepsilon x}
\end{equation*}
where $x= \Phi^{-1}\left(1-\max(w^\complement, w^\text{max}) \right)$ and $\Phi(t) = \int_{-\infty}^{t} \frac{\exp(-t^2/2)}{\sqrt{2\pi}} ds$. 

Finally, we have that 
\begin{equation}\label{appendix:second_case_general}
    \bar{\alpha} \leqslant \frac{1+x^2}{x^2}e^{-\frac{1}{2}\varepsilon^2}e^{-\varepsilon x} -w^\complement
\end{equation}
In the case where $\gamma(A^\complement) = 0$, we find a result similar to \eqref{appendix:final_result_theorem3}.

\section{Lower-bounding boundaries of partitions in a Gaussian space} \label{appendix:proof_main_result}
\paragraph{Notations and preliminaries}
Given $\varepsilon \ge 0$ and a subset $A$ of euclidean space $\mathbb R^d=(\mathbb R^d,\|\cdot-\cdot\|)$, let $A^\varepsilon := \{z \in \mathds R^d \mid \dist(z,A) \le \varepsilon\}$ be its $\varepsilon$-enlargement,
where $\dist(z,A) := \inf_{z' \in A}\|z'-z\|_2$ is the distance of the point $z \in \mathbb R^d$ from $A$. 
Let $\gamma$ be the standard Gaussian distribution in $\mathbb R^d$ and let $A_1,\ldots,A_K$ be $K \ge 2$ pairwise disjoint Borel subsets of $\mathbb R^d$ whose union has unit (i.e full) Gaussian measure $\sum_{k=1}^K w_k = 1$, where $w_k := \gamma(A_k)$. Such a collection $\{A_1,\ldots,A_K\}$ will be called an $(w_1,\ldots,w_K)$-partition of standard $d$-dimensional Gaussian space $(\mathbb R^d,\gamma)$.

For each $k \in [\![K]\!]$, define the compliment $A_{-k} := \cup_{k' \ne k}A_{k'}$, and let $\partial^{-\varepsilon}A_k := \{z \in A_k \mid \dist(z,A_{-k}) \le \varepsilon\}$ be the \emph{inner $\varepsilon$-boundary} of $A_k$, i.e the points of $A_k$ which are within distance $\varepsilon$ of some other $A_{k'}$.
For every $(k,k') \in [\![K]\!]^2$ with $k' \ne k$, it is an easy exercise to show that 
\begin{align}\label{appendix:eq:unions}
    \partial^{-\varepsilon}A_k \cap \partial^{-\varepsilon} A_{k'} &= \emptyset \\
    \partial^{-\varepsilon}A_k \cap A_{-k} &= \emptyset \nonumber \\
    A_{-k}^\varepsilon = \partial^{-\varepsilon} A_k &\cup A_{-k} \nonumber
\end{align}
Now, let $\Delta^{-\varepsilon}(A_1,\ldots,A_K) := \cup_{k=1}^K \partial^{-\varepsilon}A_k$ be the union of all the inner $\varepsilon$-boundaries. This is $\Delta^{-\varepsilon}(A_1,\ldots,A_K)$ the set of points of $\cup_{k=1}^KA_k$ which are on the boundary between some two distinct $A_k$ and $A_{k'}$. We want to find a lower bound in the measure $\gamma(\Delta^{-\varepsilon}(A_1,\ldots,A_K))$.

\begin{proposition}\label{appendix:thm_main_result}
Given $K \ge 4$ and  $w_1,\ldots,w_K \in (0, 1/4]$ such that $\sum_{k=1}^K w_k=1$, we have the bound:
\begin{align*}
\inf_{A_1,\ldots,A_K}\gamma(\Delta^{-\varepsilon}(A_1,\ldots,A_K)) &\ge 1-\frac{1+x^2}{x^2}e^{-\frac{1}{2}\varepsilon^2}e^{-\varepsilon x} 
\end{align*}
where the infinimum is taken over all $(w_1,\ldots,w_k)$-partitions of standard Gaussian space $(\mathbb R^d,\gamma)$, and $ x :=\Phi^{-1}\left(1-\max_{k \in [\![M]\!]} w_k\right)$. 
\end{proposition}

\begin{proof}[Proof]
By \eqref{appendix:eq:unions}, we have the formula
\begin{align}\label{appendix:eq:formula}
\gamma(\Delta^{-\varepsilon}(A_1,\ldots,A_K))&=\sum_{k=1}^K\gamma(\partial^{-\varepsilon}A_k) \\
&= \sum_{k=1}^K \gamma(A_{-k}^\varepsilon)-\gamma(A_{-k}).
\end{align}
Let $w_{-k} := \gamma(A_{-k})=1-w_k$, and assume $w_{-k} \ge 3/4$, i.e $w_k \le 1/4$, for all $k \in [\![K]\!]$.

For example, this condition holds in the equitable
scenario where $w_k=1/K$ for all $k$. 

Now, by standard \emph{Gaussian Isoperimetric Inequality} (see ~\cite{boucheron2013} for example), one has
\begin{align}
\gamma(A^{\varepsilon}_{-k}) 
&\ge \Phi(\Phi^{-1}(\gamma(A_{-k})+\varepsilon) \nonumber \\
&=\Phi(\Phi^{-1}(1-w_k)+\varepsilon).
\label{appendix:eq:bound}
\end{align}
Using the bound $\frac{x}{1+x^2}\varphi(x) < 1-\Phi(x) < \frac{1}{x}\varphi(x)\;\forall x > 0$ where $\varphi$ is the density of the standard Gaussian law. We can further find that 
\begin{align}\label{appendix:eq:refined_bound}
    &\Phi(\Phi^{-1}(1-w_{k})+\varepsilon) \ge 1-w_k\frac{1+\Phi^{-1}(1-w_k)^2}{\Phi^{-1}(1-w_k)^2} \times \nonumber\\
    & \quad \quad \quad e^{-\frac{1}{2}\varepsilon^2}e^{-\varepsilon\Phi^{-1}(1-w_k)} \nonumber\\
    & \quad \quad \quad \quad \quad \quad \quad \quad \ge 1-w_k\frac{1+ x^2}{ x^2}e^{-\frac{1}{2}\varepsilon^2}e^{-\varepsilon x} > 0 \\
    &\text{(since the function $x \mapsto \frac{1 + x^2}{x^2}e^{-\varepsilon x}$ is decreasing for $x > 0$)} \nonumber
\end{align}
where $ x := \min_{k \in [\![K]\!]}\Phi^{-1}(1-w_k) = \Phi^{-1} \left(1-\max_{k \in [\![K]\!]} w_k\right) \ge \Phi^{-1}(3/4) > 0.67$. Combining \eqref{appendix:eq:formula}, \eqref{appendix:eq:bound}, and \eqref{appendix:eq:refined_bound} yields the following

\begin{align*}
\gamma(\Delta^{-\varepsilon}(A_1,\ldots,A_K)) 
&\ge \sum_{k=1}^K \big(1-w_k\frac{1+ x^2}{ x^2}e^{-\frac{1}{2}\varepsilon^2}e^{-\varepsilon x} \\
&-(1-w_k))\big) \\
&=\sum_{k=1}^K \left(1-\frac{1+ x^2}{ x^2}e^{-\frac{1}{2}\varepsilon^2}e^{-\varepsilon x}\right) w_k \\
&= 1-\frac{1+ x^2}{ x^2}e^{-\frac{1}{2}\varepsilon^2}e^{-\varepsilon x},
\end{align*}

\paragraph{Asymptotic analysis}
In the limit, it is easy to check that in the case where $\max_{k \in [\![K]\!]}w_{k} \longrightarrow 0$, we have that $x \longrightarrow \infty$. In this setting, we thus have $\frac{1+ x^2}{ x^2} \longrightarrow 1$ and can now derive the following bound:
\begin{equation*}
     \inf_{A_1,\ldots,A_K}\gamma(\Delta^{-\varepsilon}(A_1,\ldots,A_K)) \overset{\max_{k \in [\![K]\!]}w_{k} \rightarrow 0}{\longrightarrow} 1-e^{-\frac{1}{2}\varepsilon^2}e^{-\varepsilon x}.
\end{equation*}

\paragraph{Equitable scenario}\label{appendix:main_result_equitable_scenario}
In the equitable scenario where $w_k = 1/K$ for all $k$, we have 

\begin{equation*}
    \inf_{A_1,\ldots,A_K}\gamma(\Delta^{-\varepsilon}(A_1,\ldots,A_K)) \geqslant 1-\frac{1+x^2}{x^2}e^{-\frac{1}{2}\varepsilon^2}e^{-\varepsilon x} 
\end{equation*}
where $x = \Phi^{-1}(1-1/K)$. When $K\geq 8$ we have:
  \begin{eqnarray}  
     \Phi^{-1}(1-1/K)  \geqslant \sqrt{2 \log \left ( \frac{K \left(q(K)^2-1\right)}{\sqrt{2 \pi}  q(K)^3} \right )}
\label{eq:crudeman}
\end{eqnarray}
where $q(K)=\sqrt{2\log(\sqrt{2\pi}K)}$. 

Consequently, we have when $K \to \infty$, the following behavior:
\begin{equation}\label{appendix:eq:explicit}
    \gamma(\Delta^{-\varepsilon}(A_1, \ldots,A_K)) \overset{K \rightarrow \infty}{\leqslant} 1-e^{-\frac{1}{2}\varepsilon^2} e^{-\varepsilon\sqrt{2\log(K})}
\end{equation}
\end{proof}

\begin{proof}[Proof of the inequality \eqref{eq:crudeman}]
Set $p := 1/K$. First, for any $x>0$, we have the following upper:
$$\int_x^\infty e^{-y^2/2} dy = \int_x^\infty \frac{y}{y} e^{-y^2/2} dy \leq \frac{1}{x} \int_x^\infty y e^{-y^2/2} dy = \frac{e^{-x^2/2}}{x}.$$
For a lower bound:
\begin{eqnarray*}
\begin{split}
\int_x^\infty e^{-y^2/2} dy &= \int_x^\infty \frac{y}{y} e^{-y^2/2} dy = \frac{e^{-x^2/2}}{x} - \int_x^\infty \frac{1}{y^2} e^{-y^2/2} dy
\end{split}
\end{eqnarray*}
and
$$
\int_x^\infty \frac{1}{y^2} e^{-y^2/2} dy = \int_x^\infty \frac{y}{y^3} e^{-y^2/2} dy \leq \frac{e^{-x^2/2}}{x^3}
$$
and combining these gives
$$\int_x^\infty e^{-y^2/2} dy \geq \left ( \frac{1}{x} - \frac{1}{x^3} \right ) e^{-x^2/2}.$$
Thus
$$\frac{1}{\sqrt{2\pi}}\left ( \frac{1}{x} - \frac{1}{x^3} \right ) e^{-x^2/2} \leq 1-\Phi(x) \leq \frac{1}{\sqrt{2\pi}} \frac{1}{x} e^{-x^2/2},$$
from where
\begin{align}
    &\frac{1}{\sqrt{2\pi}}\left ( \frac{1}{\Phi^{-1}(1-p)} - \frac{1}{\Phi^{-1}(1-p)^3} \right ) e^{-\Phi^{-1}(1-p)^2/2} \label{eq:lower_boundd}\\
    &\quad\leq p\leq \frac{1}{\sqrt{2\pi}} \frac{1}{\Phi^{-1}(1-p)} e^{-\Phi^{-1}(1-p)^2/2} \label{eq:upper_boundd}
\end{align}
Using \eqref{eq:upper_boundd}, when $\Phi^{-1}(1-p) \geq 1$ (that is $p \leqslant 0.15$ or equivalently $K\geq8$), we have the following upper bound $\Phi^{-1}(1-p) \leqslant q(p)$ where $q(p):=\sqrt{2\log(\sqrt{2\pi}/p)}$. Then, injecting $q(p)$ in \eqref{eq:lower_boundd}:
$$\frac{1}{\sqrt{2\pi}} \left ( \frac{1}{q(p)} - \frac{1}{q(p)^3} \right ) e^{-\Phi^{-1}(1-p)^2/2} \leq p.$$
Now when $q(p) \geq 1$ you have: 
$$e^{-\Phi^{-1}(1-p)^2/2} \leq \frac{\sqrt{2\pi} p q(p)^3}{q(p)^2-1}$$
and
\begin{equation*}
    \Phi^{-1}(1-p) \geq \sqrt{2 \log \left ( \frac{q(p)^2-1}{\sqrt{2 \pi} p q(p)^3} \right )}.
\end{equation*}
There is one additional requirement on $p$ which is simply that the argument of the log should be $\geq 1$ i.e. $q(p)^2-1 \geq \sqrt{2 \pi} p q(p)^3$, which is true as soon as  $K\geq 8$.
\end{proof}

\newpage
\section{Visualization of Theorem \ref{th:extended_no_free_lunch}}\label{appendix:visualization_of_theorem3}

\begin{figure}[H]
    \captionsetup[subfigure]{margin={0.5cm,0cm}}
    \subfloat[WGAN 4 classes: \newline
    visualisation of $\|J_G(z)\|_{F}$.]
    {   
        \includegraphics[width=0.47\columnwidth]{images/simple_synthetic_dataset/norm_jacobian_4modes.png}
    }
    \captionsetup[subfigure]{margin={0.5cm,0cm}}
    \subfloat[Green blobs: true densities. Dots: generated points.]
    {   
        \includegraphics[width=0.47\columnwidth]{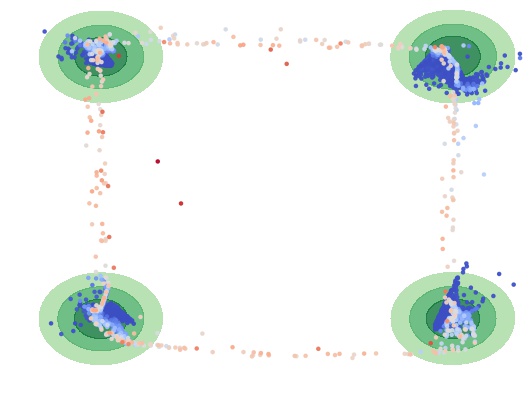}
    }
    \hfill
    \captionsetup[subfigure]{margin={0.5cm,0cm}}
    \subfloat[WGAN 9 classes: \newline
    visualisation of $\|J_G(z)\|_{F}$.]
    {   
        \includegraphics[width=0.46\columnwidth]{images/simple_synthetic_dataset/norm_jacobian_9modes.png}
    }
    \captionsetup[subfigure]{margin={0.6cm,0cm}}
    \subfloat[Green blobs: true densities. Dots: generated points.]
    {   
        \includegraphics[width=0.48\linewidth]{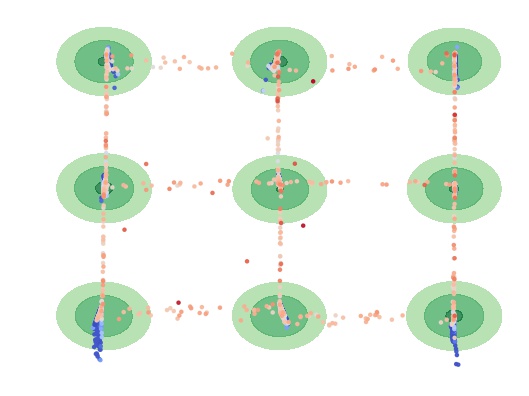}
    }
    \hfill
    \captionsetup[subfigure]{margin={0.5cm,0cm}}
    \subfloat[WGAN 3 classes: \newline
    visualisation of $\|J_G(z)\|_{F}$.]
    {   
        \includegraphics[width=0.43\columnwidth]{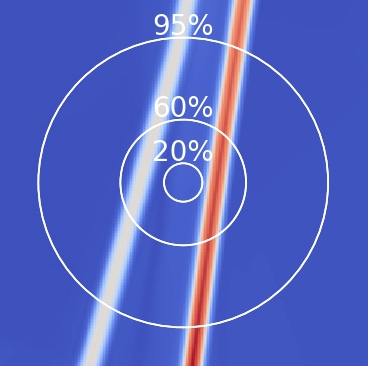}
    }
    \captionsetup[subfigure]{margin={0.6cm,0cm}}
    \subfloat[Green blobs: true densities. Dots: generated points.]
    {   
        \includegraphics[width=0.48\linewidth]{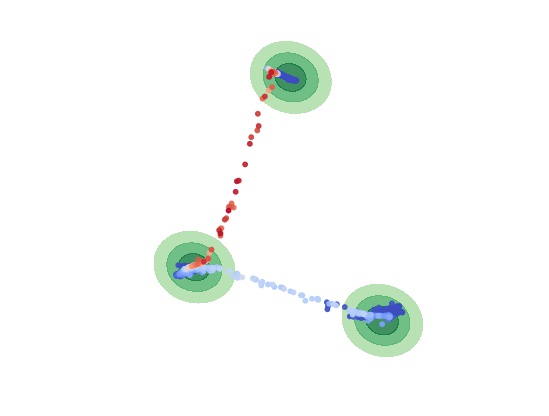}
    }
    \hfill
    \subfloat[WGAN 5 classes: \newline
    visualisation of $\|J_G(z)\|_{F}$.]
    {   
        \includegraphics[width=0.45\linewidth]{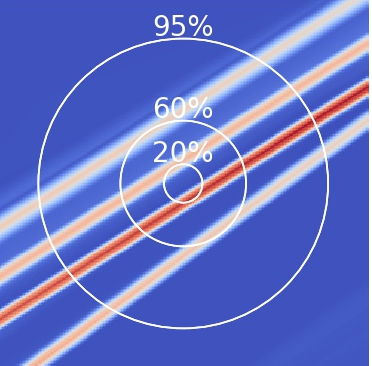}
    }
    \captionsetup[subfigure]{margin={0.6cm,0cm}}
    \subfloat[Green blobs: true densities. Dots: generated points.]
    {   
        \includegraphics[width=0.48\linewidth]{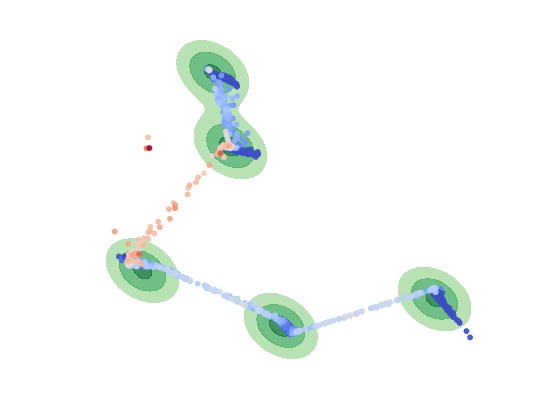}
    }
    \caption{Learning disconnected manifolds: visualization of the gradient of the generator (JFN) in the latent space and densities in the output space.}
\end{figure}

\newpage
\section{Definition of the different metrics used}\label{appendix:metrics}
In the sequel, we present the different metrics used in Section \ref{section:experiments} of the paper to assess performances of GANs. We have:
\begin{itemize}
    \item Improved Precision/Recall (PR) metric \cite{kynkaanniemi2019improved}: it has been presented in Definition \ref{def:improved_prec_rec}. Intuitively, 
    Based on a k-NN estimation of the manifold of real (resp. generated) data, it assesses whether generated (resp. real) points belong in the real (resp. generated) data manifold or not. The proportion of generated (resp. real) points that are in the real (resp. generated) data manifold is the precision (resp. recall). 
    
    \item the Hausdorff distance: it is defined by
    \begin{equation*}
    \begin{aligned}
        \text{Haus}(A,B)  = \max\left\{ \max_{a\in A} \min_{b\in B} \|a-b\|, \max_{b\in B} \min_{a\in A} \|a-b\|  \right\}
    \end{aligned}
    \end{equation*}
     Such a distance is useful to evaluate the closeness of two different supports from a metric space, but is sensitive to outliers because of the max operation. It has been recently used for theoretical purposes by \citet{pandeva2019mmgan}.
     
     \item the Frechet Inception distance: first proposed by \citet{dowson1982frechet}, the Frechet distance was applied in the setting of GANs by \citet{heusel2017gans}. This distance between mutlivariate Gaussians compares statistic of generated samples to real samples as follows
     \begin{equation*}
         \text{FID} = \|\nu_\star - \nu_\theta \|^2 + Tr\big(\Sigma_\star + \Sigma_\theta + 2(\Sigma_\star \Sigma_\theta)^{\frac{1}{2}}\big)
     \end{equation*}
     where $X_\star = \mathcal{N}(\nu_\star, \Sigma_\star)$ and $X_\theta = \mathcal{N}(\nu_\theta, \Sigma_\theta)$ are the activations of a pre-softmax layer. However, when dealing with disconnected manifolds, we argue that this distance is not well suited as it approximates the distributions with unimodal one, thus loosing many information.
\end{itemize}
The choice of such metrics is motivated by the fact that metrics measuring the performances of GANs should not rely on relative densities but should rather be point sets based metrics. 

\newpage
\section{Saturation of a MLP neural network}
\label{appendix:precision_saturated_synthetic}

In Section \ref{subsection:exp_synthetic}, we claim that the generator reduces the sampling of off-manifold data points up to a saturation point. Figure \ref{fig:exp_synthetic_saturation} below provides a visualization of this phenomenon. In this synthetic case, we learn a 9-component mixture of Gaussians using simple GANs architecture (both the generator and the discriminator are MLP with two hidden layers). The minimal distance between two modes is set to 9. We clearly see in Figure \ref{subfig:precision_saturation} that the precision saturates around 80\%. 

\begin{figure}[h]
    \captionsetup[subfigure]{margin={0.5cm,0cm}}
    \subfloat[Data points sampled after 5,000 steps of training.]
    {   
        \includegraphics[width=0.47\columnwidth]{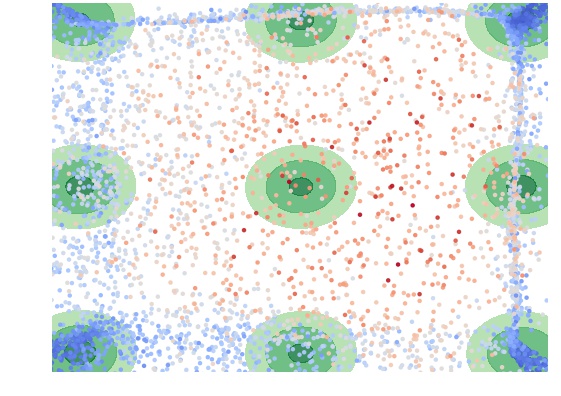}
    }
    \captionsetup[subfigure]{margin={0.5cm,0cm}}
    \subfloat[Data points sampled after 50,000 steps of training.]
    {   
        \includegraphics[width=0.47\columnwidth]{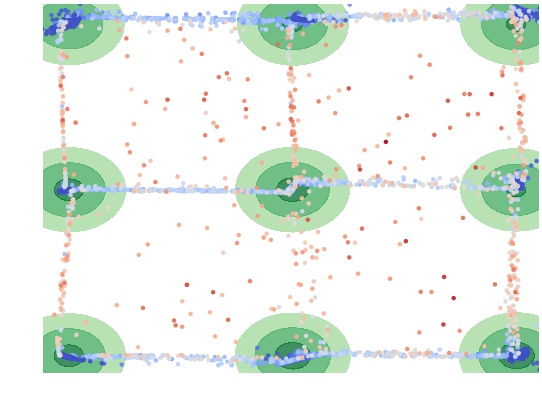}
    }
    \hfill
    \captionsetup[subfigure]{margin={0.5cm,0cm}}
    \subfloat[Data points sampled after 100,000 steps of training.]
    {   
        \includegraphics[width=0.47\columnwidth]{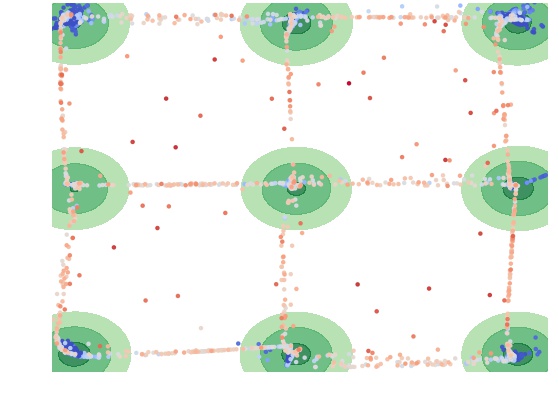}
    }
    \captionsetup[subfigure]{margin={0.6cm,0cm}}
    \subfloat[Evolution of the precision $\bar{\alpha}$ during training.\label{subfig:precision_saturation}]
    {   
        \includegraphics[width=0.47\linewidth]{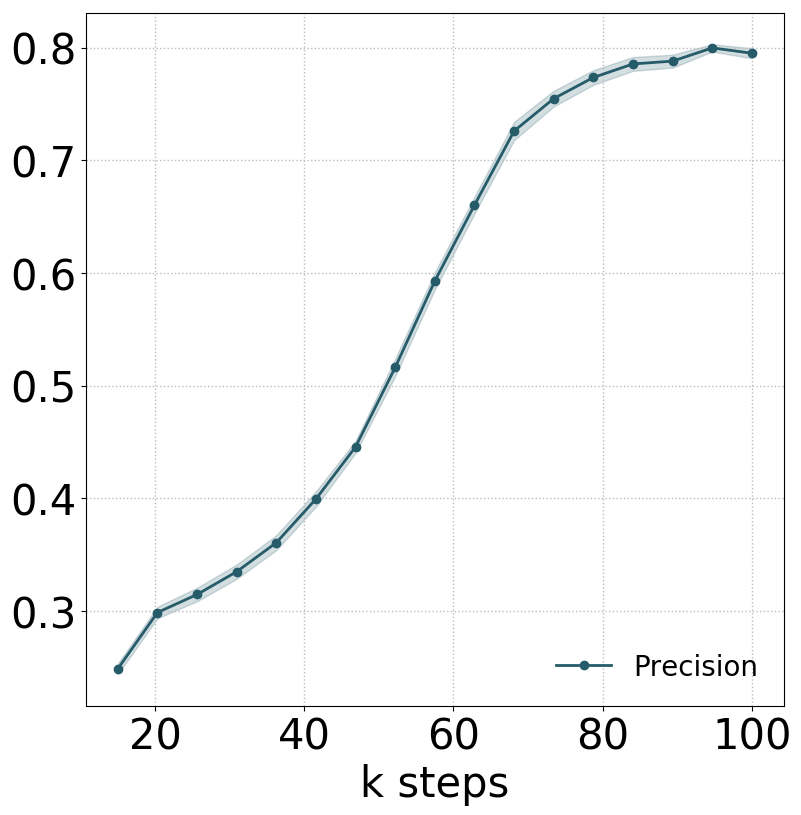}
    }
    \caption{\label{fig:exp_synthetic_saturation}Learning 9 disconnected manifolds with a standard GANs architecture.}
\end{figure}

\newpage
\section{More results and visualizations on MNIST/F-MNIST/CIFAR10}
\label{appendix:mnist_fmnist_and_cifar10}

Additionally to those in Section \ref{subsection:image_datasets}, we provide in Figure \ref{appendix:fig_appendix_trunc_supplementary_results} and Table \ref{appendix:table_prec_rec_mnist_fmnist_cifar10} supplementary results for MNIST, F-MNIST and CIFAR-10 datasets. 

\begin{figure}[H]
    \subfloat[MNIST: examples of data points selected by our JBT with a truncation ratio of 90\% (we thus removed the 10\% highest gradients). \label{appendix:fig:removing_high_gradient_MNIST}]
    {   
        \includegraphics[width=0.43\linewidth]{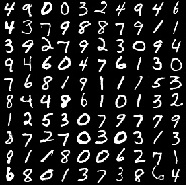}
    }\hfill
    \subfloat[MNIST: examples of data points removed by our JBT with a truncation ratio of 90\% (these are the 10\% highest gradients data points). \label{appendix:fig:removing_low_gradient_MNIST}]
    {   
        \includegraphics[width=0.43\linewidth]{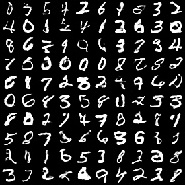}
    } \hfill
    \subfloat[F-MNIST: examples of data points selected by our JBT with a truncation ratio of 90\% (we thus removed the 10\% highest gradients).. \label{appendix:fig:removing_high_gradient_FMNIST}]
    {   
        \includegraphics[width=0.43\linewidth]{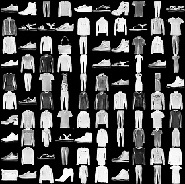}
    }\hfill
    \subfloat[F-MNIST: examples of data points removed by our JBT with a truncation ratio of 90\% (these are the 10\% highest gradients data points). \label{appendix:fig:appendix_reducing_z_variance_FMNIST}]
    {   
        \includegraphics[width=0.43\linewidth]{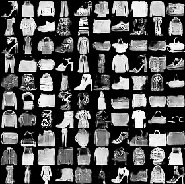}
    }
    \caption{Visualization of our truncation method on CIFAR10. \label{appendix:fig:appendix_trunc_MNIST_FMNIST}} 
    \subfloat[CIFAR-10: examples of data points selected by our JBT with a truncation ratio of 90\% (we thus removed the 10\% highest gradients). \label{appendix:fig:appendix_removing_high_gradient_CIFAR10}]
    {   
        \includegraphics[width=0.43\linewidth]{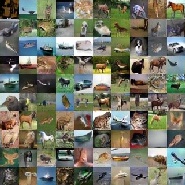}
    }\hfill
    \subfloat[MNIST: examples of data points removed by our JBT with a truncation ratio of 90\% (these are the 10\% highest gradients data points). \label{appendix:fig:appendix_reducing_z_variance_CIFAR10}]
    {   
        \includegraphics[width=0.43\linewidth]{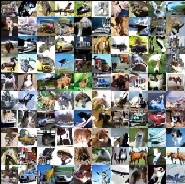}
    }
    \caption{Visualization of our truncation method (JBT) on three real-world datasets: MNIST, F-MNIST and CIFAR-10. \label{appendix:fig_appendix_trunc_supplementary_results}} 
\end{figure}

\begin{figure*}[ht]
    {   
        \includegraphics[width=0.30\linewidth]{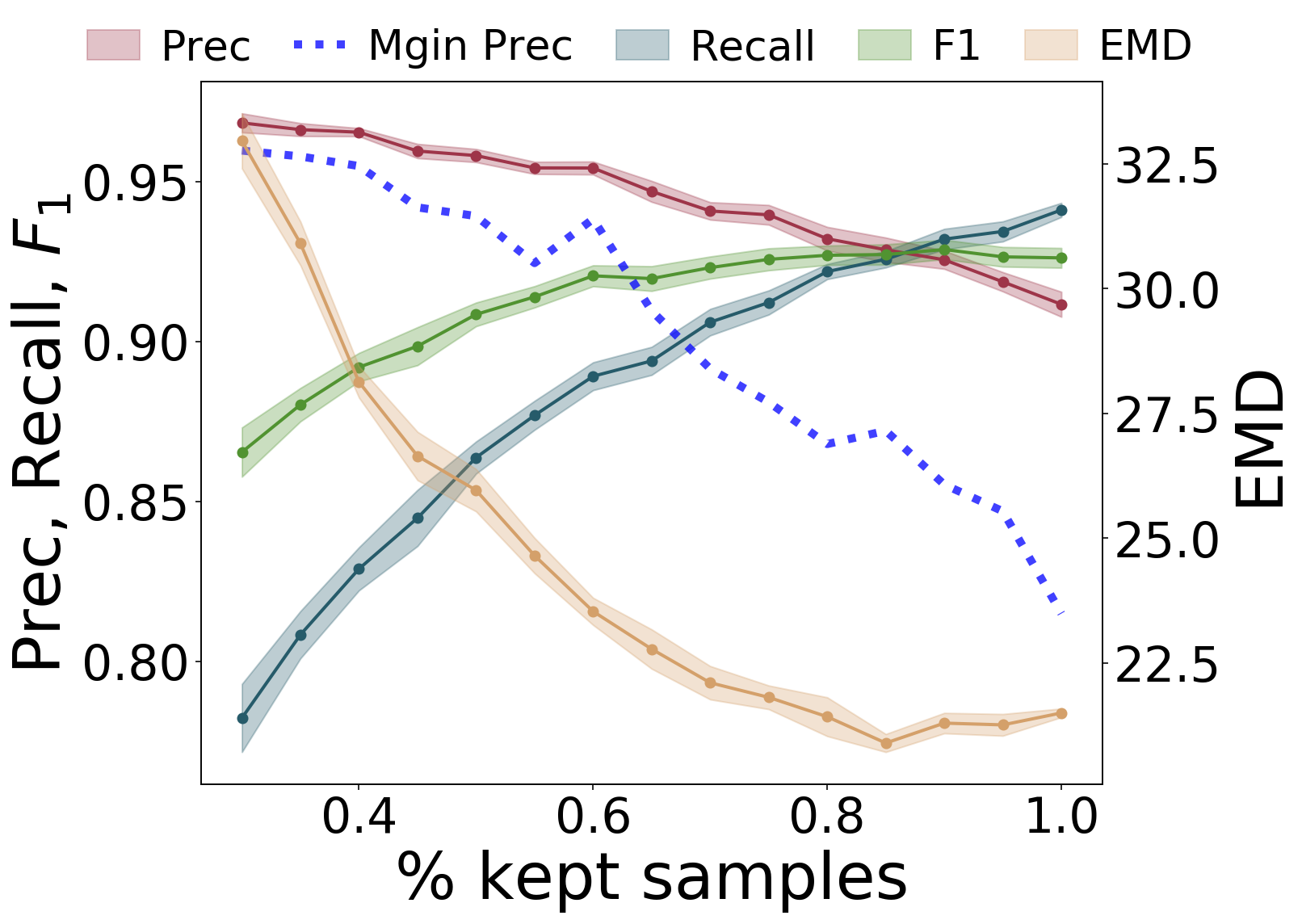}
    } \hfill
    {   
        \includegraphics[width=0.30\linewidth]{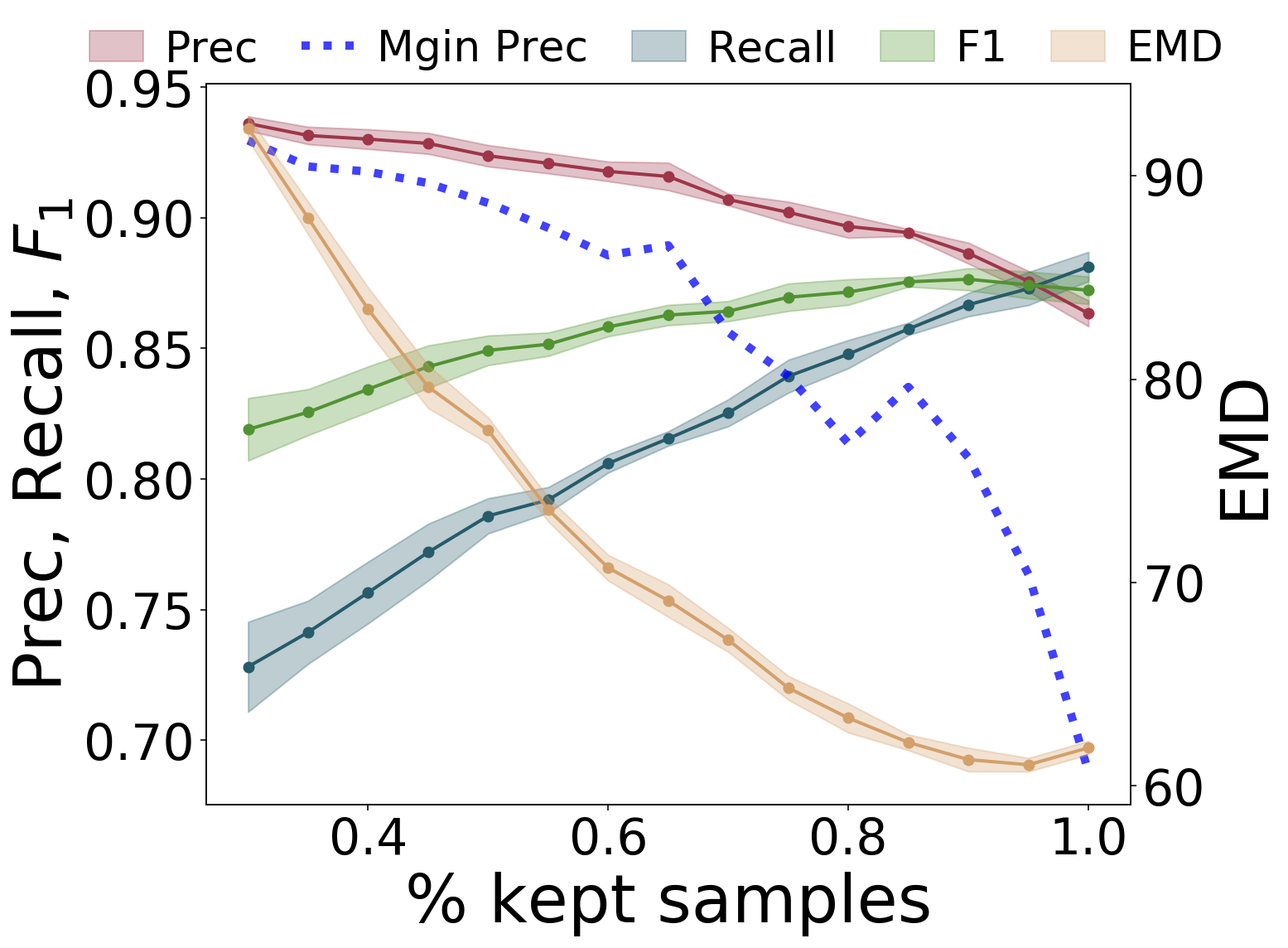}
    } \hfill
    {   
        \includegraphics[width=0.30\linewidth]{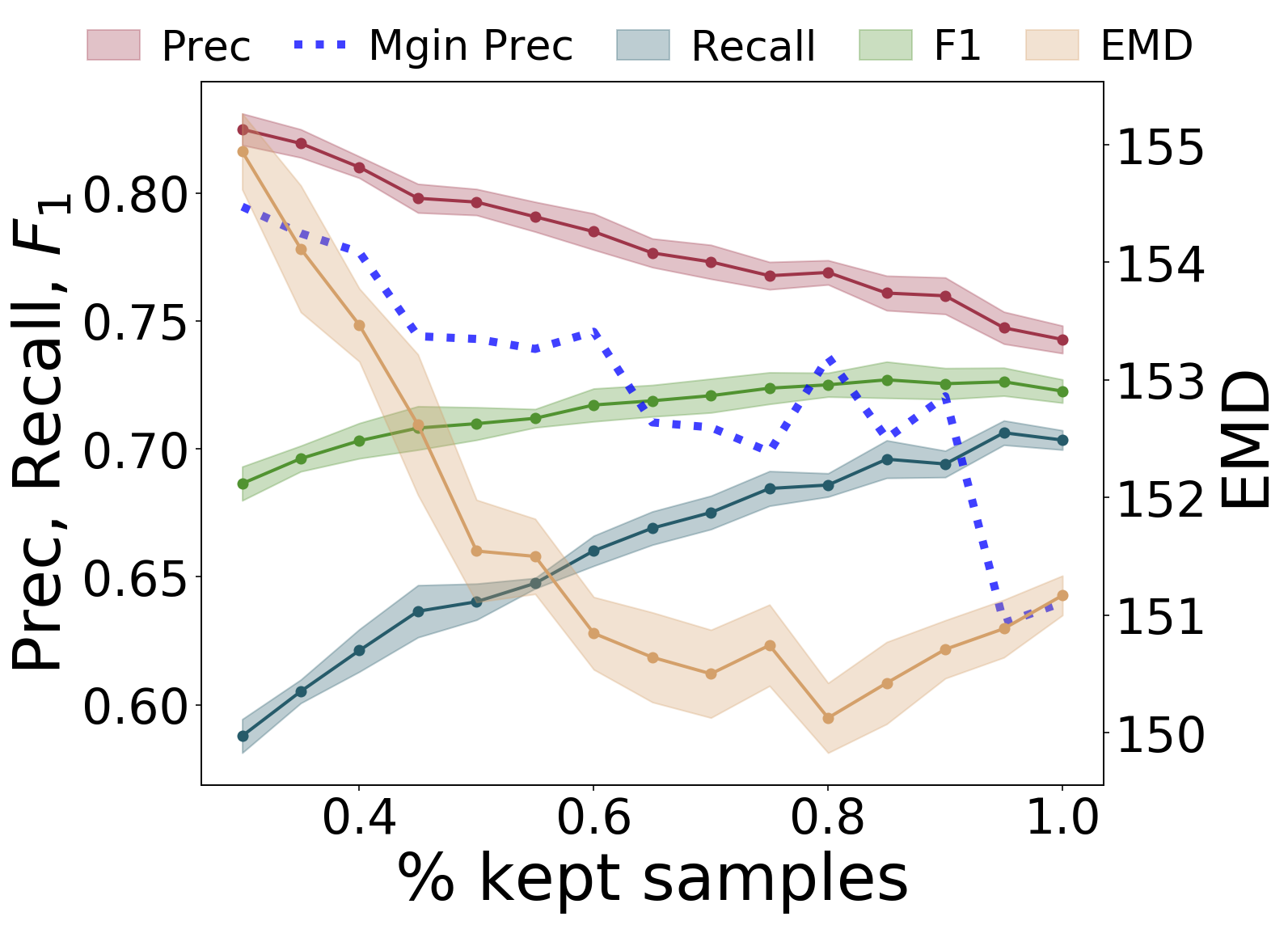}
    }\\
    \vspace{-0.4cm}
    \caption{For high levels of kept samples, the marginal precision plummets of newly added samples, underlining the efficiency of our truncation method (JBT). 
    Reported confidence intervals are $97\%$ confidence intervals. On the second row, generated samples ordered by their JFN (left to right, top to bottom). In the last row, the data points generated are blurrier and outside the true manifold.  \label{fig:trunc_MNIST_FMNIST_with_emd}} 
\end{figure*}

\begin{table*}[ht]
\begin{center}
\begin{tabular}{|l|c|c|c|c|c|c|c|}
\cline{2-7}
\multicolumn{1}{l|}{\textbf{MNIST}}  & Prec. & Rec. & F1 & Haus. & FID & EMD \\
\hline
WGAN& $91.2 \pm 0.3$ & $\mathbf{93.7 \pm 0.5}$ & $\mathbf{92.4 \pm 0.4}$ & $49.7 \pm 0.2$ & $24.3 \pm 0.3$ & $21.5 \pm 0.1$ \\
WGAN 90\% lowest JFN & $92.5 \pm 0.5$ & $92.9 \pm 0.3$ & $\mathbf{92.7 \pm 0.4}$ & $\mathbf{48.1 \pm 0.2}$ & $26.9 \pm 0.5$ & $21.3 \pm 0.2$\\
WGAN 80\% lowest JFN & $\mathbf{93.3 \pm 0.3}$ & $91.8 \pm 0.4$ & $\mathbf{92.6 \pm 0.4}$ & $50.6 \pm 0.4$ & $33.1 \pm 0.3$ & $21.4 \pm 0.4$ \\
W-Deligan & $89.0 \pm 0.6$ & $\mathbf{93.6 \pm 0.3}$ & $91.2 \pm 0.5$ & $50.7 \pm 0.3$ & $31.7 \pm 0.5$ & $22.4 \pm 0.1$ \\
DMLGAN & $\mathbf{93.4 \pm 0.2}$ & $92.3 \pm 0.2$ & $\mathbf{92.8 \pm 0.2}$ & $\mathbf{48.2 \pm 0.3}$ & $\mathbf{16.8 \pm 0.4}$ & $\mathbf{20.7 \pm 0.1}$ \\ \hline
\multicolumn{6}{l} {\textbf{Fashion-MNIST}} \\
\hline
WGAN & $86.3 \pm 0.4$ &  $\mathbf{88.2 \pm 0.2}$ & $\mathbf{87.2 \pm 0.3}$ & $140.6 \pm 0.7$ & $259.7 \pm 3.5$ & $61.9 \pm 0.3$ \\
WGAN 90\% lowest JFN & $88.6 \pm 0.6$ & $86.6 \pm 0.5$ & $\mathbf{87.6 \pm 0.5}$ & $\mathbf{138.7 \pm 0.9}$ & $\mathbf{257.4 \pm 3.0}$ & $\mathbf{61.3 \pm 0.6}$ \\
WGAN 80\% lowest JFN & $\mathbf{89.8 \pm 0.4}$ & $84.9 \pm 0.5$ & $\mathbf{87.3 \pm 0.4}$ & $146.3 \pm 1.1$ & $396.2 \pm 6.4$ & $63.3 \pm 0.7$\\
W-Deligan & $88.5 \pm 0.3$ & $85.3 \pm 0.6$ & $86.9 \pm 0.4$ & $141.7 \pm 1.1$ & $310.9 \pm 3.1$ & $\mathbf{60.9 \pm 0.4}$ \\
DMLGAN & $87.4 \pm 0.3$ & $\mathbf{88.1 \pm 0.4}$ & $\mathbf{87.7 \pm 0.4}$ & $141.9 \pm 1.2$ & $\mathbf{253.0 \pm 2.8}$ & $\mathbf{60.9 \pm 0.4}$\\ \hline
\multicolumn{6}{l} {\textbf{CIFAR10}} \\
\hline
WGAN & $74.3 \pm 0.5$ & $\mathbf{70.3 \pm 0.4}$ & $\mathbf{72.3 \pm 0.5}$ & $334.7 \pm 3.5$ & $\mathbf{634.8 \pm 4.6}$ & $151.2 \pm 0.2$ \\
WGAN 90\% lowest JFN & $\mathbf{76.0 \pm 0.7}$ & $69.4 \pm 0.5$ & $\mathbf{72.5 \pm 0.6}$ &  $\mathbf{318.1 \pm 3.7}$ & $\mathbf{631.3 \pm 4.5}$ & $150.7 \pm 0.2$\\
WGAN 80\% lowest JFN & $\mathbf{76.9 \pm 0.5}$ & $68.6 \pm 0.5$ & $\mathbf{72.5 \pm 0.5}$ & $\mathbf{323.5 \pm 4.0}$ & $725.0 \pm 3.5$ & $\mathbf{150.1 \pm 0.3}$\\
W-Deligan & $71.5 \pm 0.7$ & $\mathbf{69.8 \pm 0.7}$ & $70.6 \pm 0.7$ & $328.7 \pm 2.1$ & $727.8 \pm 3.9$ & $154.0 \pm 0.3$ \\ 
DMLGAN & $74.1 \pm 0.5$ & $65.7 \pm 0.6$ & $69.7 \pm 0.6$ & $328.6 \pm 2.7$ & $967.2 \pm 4.1$ & $152.0 \pm 0.4$\\ \hline
\end{tabular}
\end{center}
\caption{Scores on MNIST and Fashion-MNIST. JFN stands for Jacobian Frobenius norm.\label{appendix:table_prec_rec_mnist_fmnist_cifar10} $\pm$ is $97\%$ confidence interval.}
\end{table*}

\clearpage
\section{More results on BigGAN and ImageNet} \label{appendix:biggan}
In Figure \ref{appendix:fig:bubble_class}, we show images from the Bubble class of ImageNet. It supports our claim of manifold disconectedness, even within a class, and outlines the importance of studying the learning of disconnected manifolds in generative models. Then, in Figure \ref{appendix:fig:biggan_histograms}, we give more exemples from BigGAN 128x128 class-conditionned generator. We plot in the same format than in \ref{subsection:biggan}. Specifically, for different classes, we plot 128 images ranked by JFN. Here again, we see a concentration of off-manifold samples on the last row, proving the efficiency of our method. Example of classes responding particularly well to our ranking are House Finch \subref{appendix:fig:house_finch_ranked}, Monnarch Butterfly \subref{appendix:fig:monarch_butterfly_ranked} or Wood rabbit \subref{appendix:fig:wood_rabbit}. For each class, we also show an histogram of JFN based on 1024 samples. It shows that the JFN is a good indicator of the complexity of the class. For example, classes such as Cornet \subref{fig:cornet_ranked_appendix} or Football helmet \subref{appendix:fig:football_helmet} are very diverse and disconnected, resulting in high JFNs.  

\begin{figure}[H]
    \centering
    \includegraphics[width=0.75\linewidth]{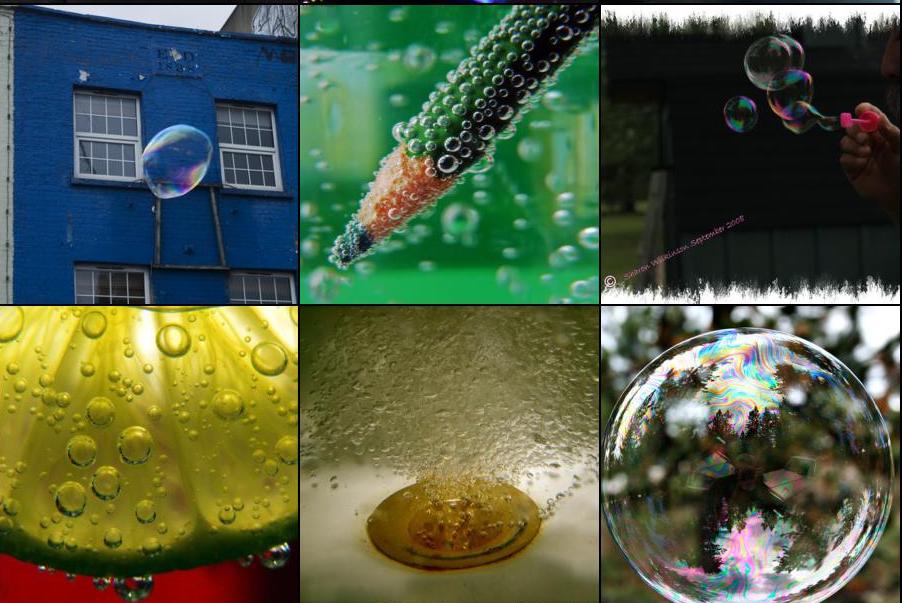}
    \caption{Images from the Bubble class of ImageNet showing that the class is complex and slightly multimodal. \label{appendix:fig:bubble_class}}
\end{figure}

\begin{figure}[H]
    \subfloat['Black swan' class. \label{appendix:swan}]
    {   
        \includegraphics[width=0.45\linewidth]{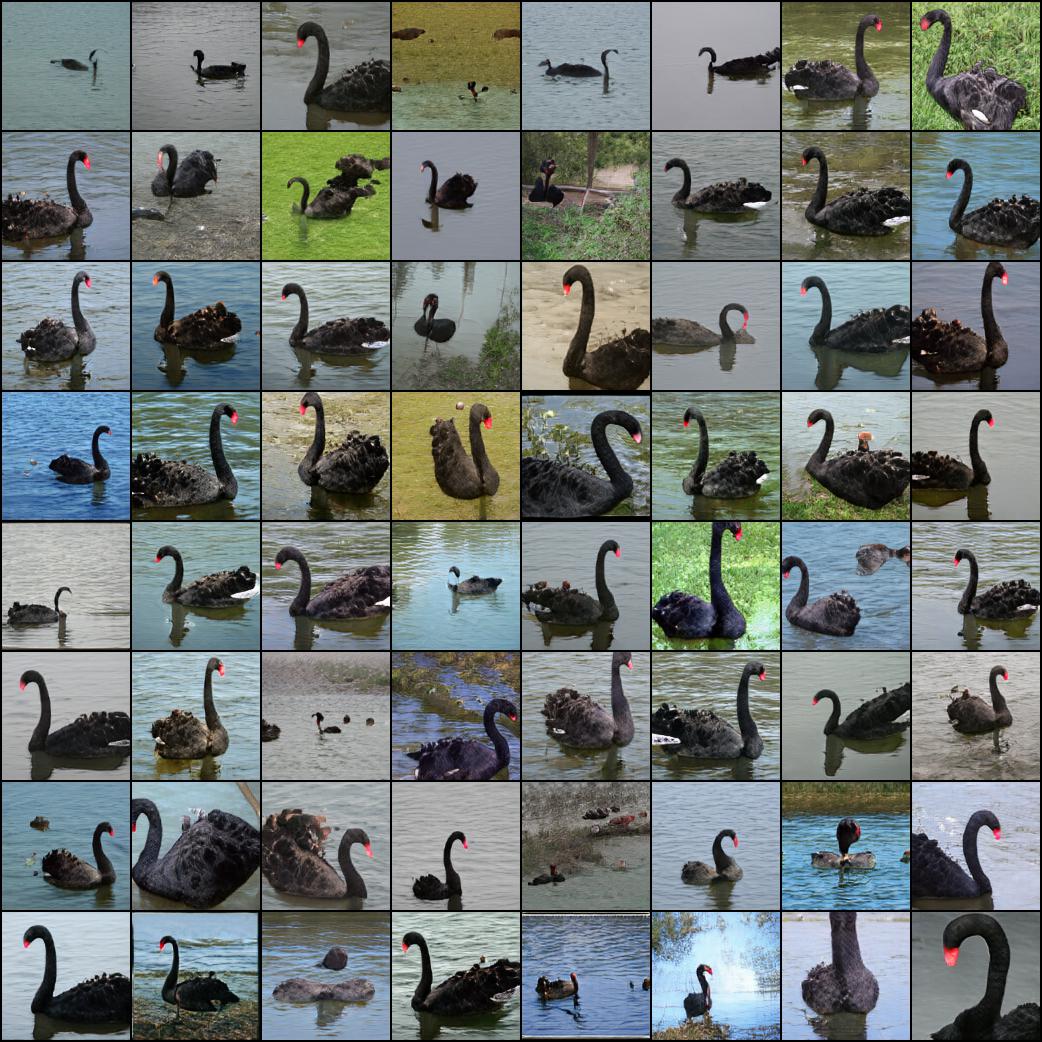}
    }\hfill
        \subfloat['Black swan' class histogram. \label{appendix:swanh}]
    {   
        \includegraphics[width=0.45\linewidth]{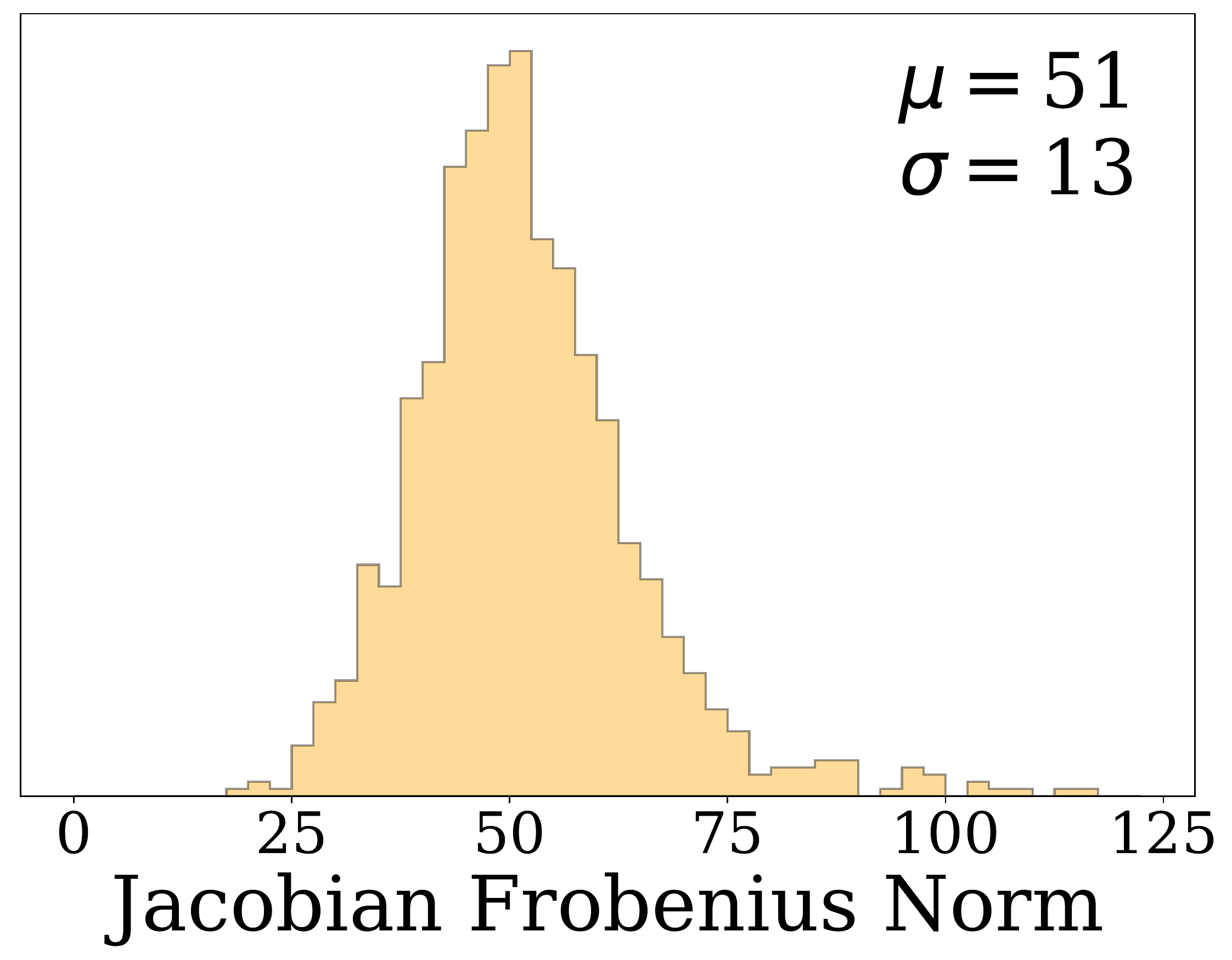}
    }\hfill
    \subfloat['House finch' class. \label{appendix:fig:house_finch_ranked}]
    {   
        \includegraphics[width=0.45\linewidth]{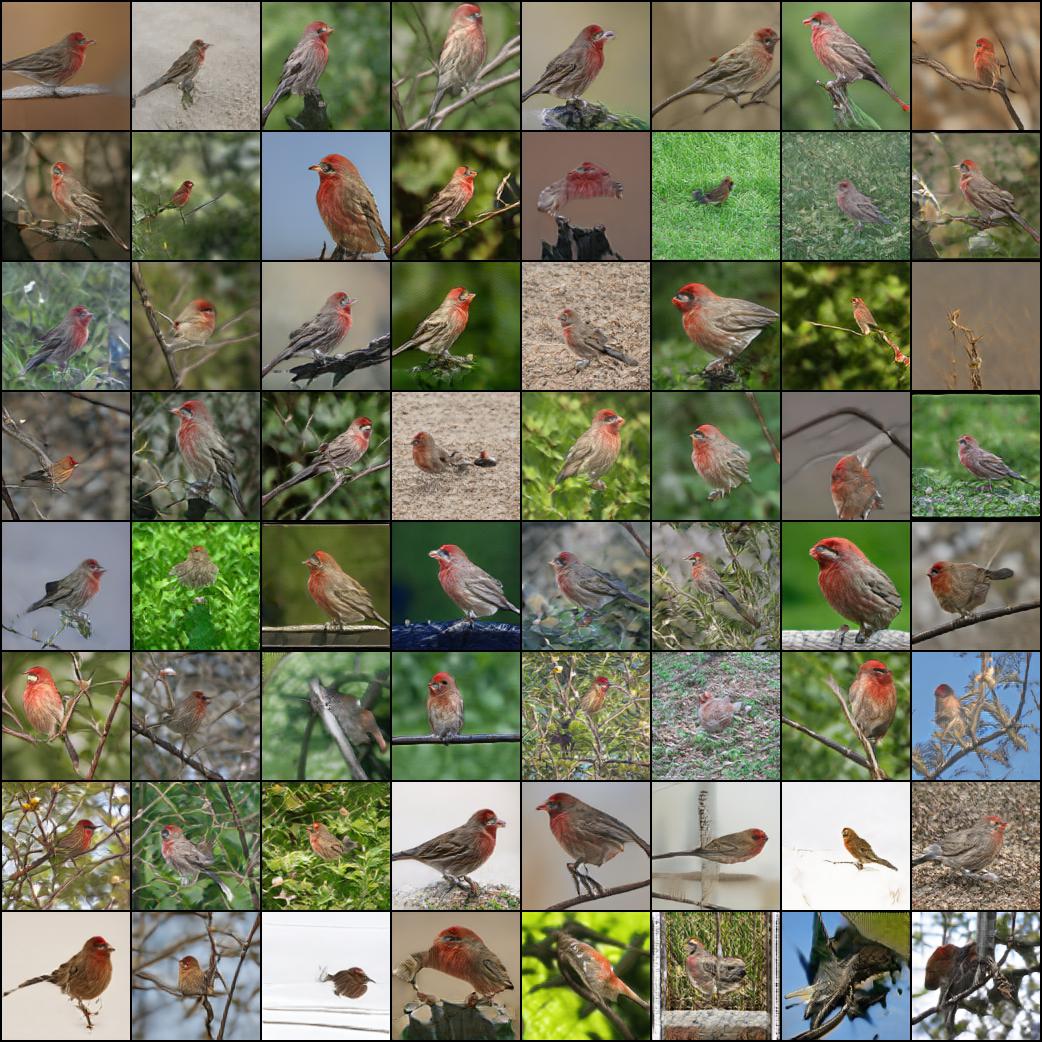}
    }  \hfill
    \subfloat['House finch' class histogram.]
    {   
        \includegraphics[width=0.50\linewidth]{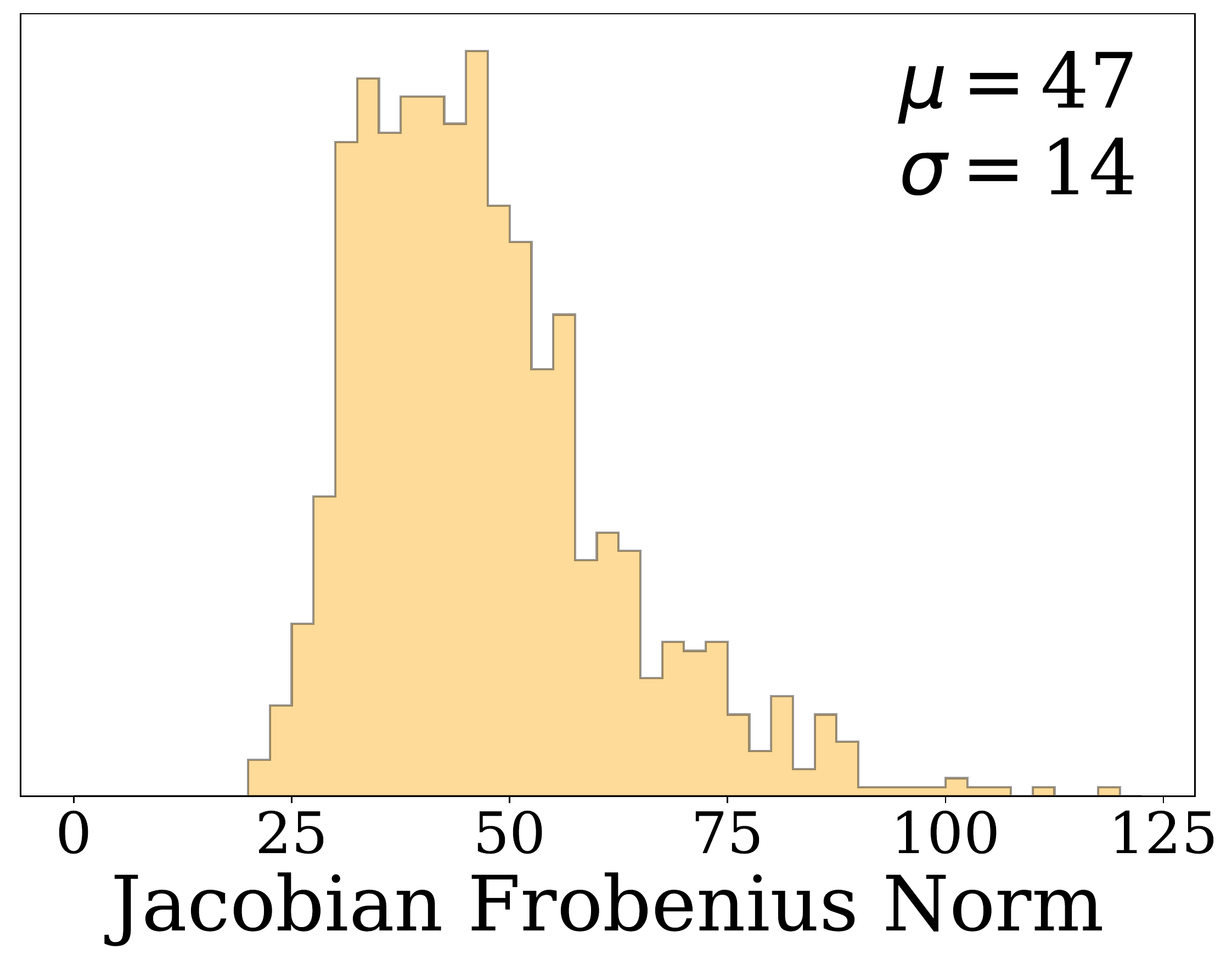}
    }  \hfill
    \subfloat['Indigo bunting' class.]
    {   
        \includegraphics[width=0.45\linewidth]{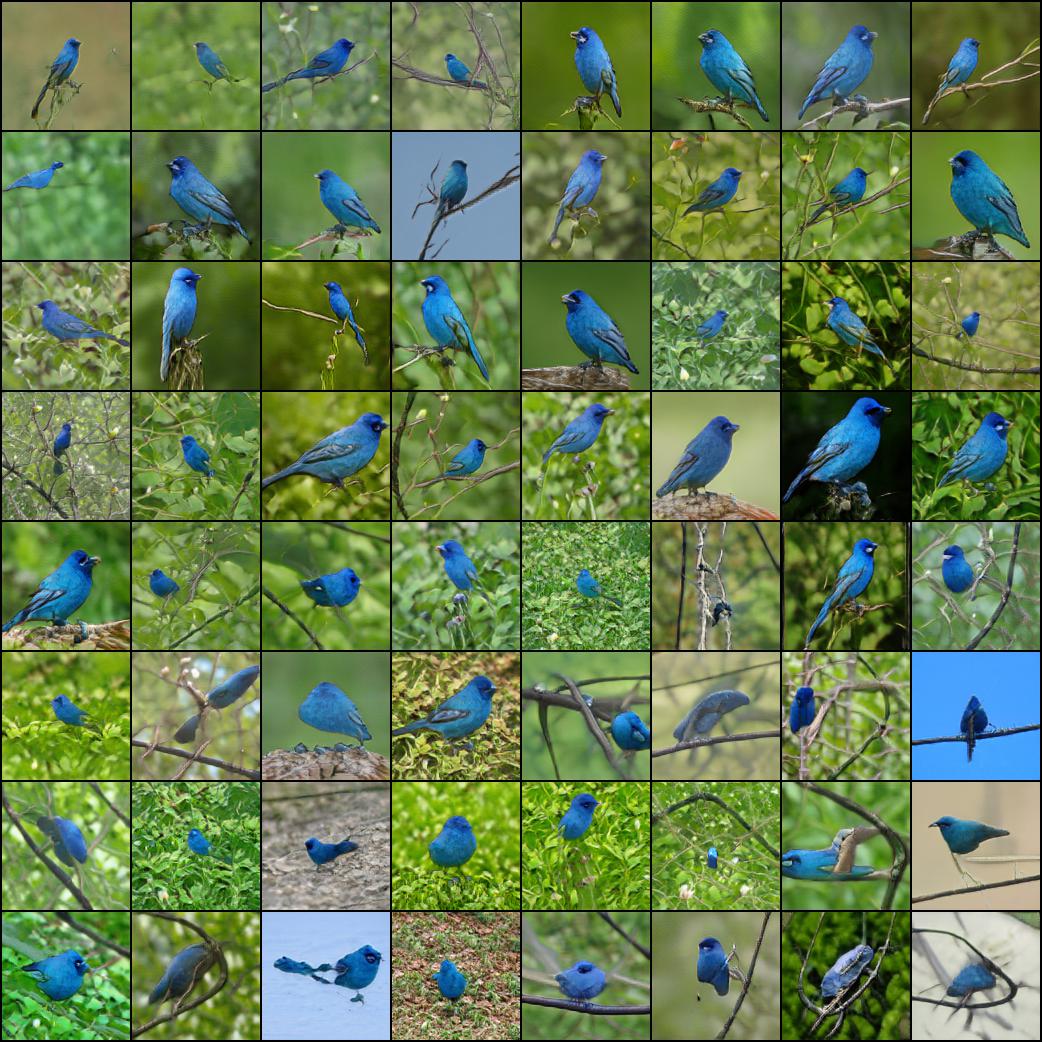}
    }  \hfill
    \subfloat['Indigo bunting' class histogram. ]
    {   
        \includegraphics[width=0.50\linewidth]{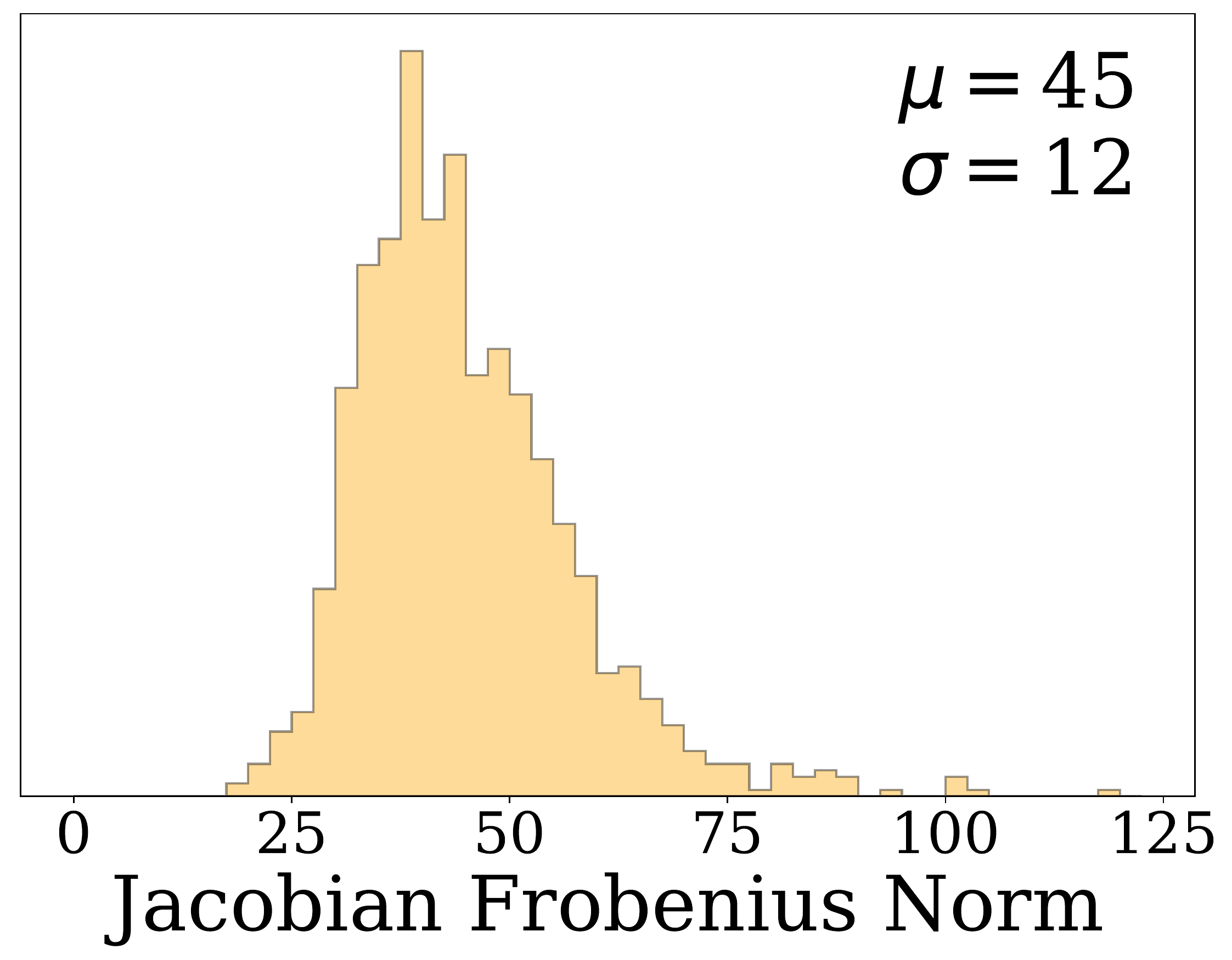}
    }  \hfill
    \subfloat['Cheetah' class.]
    {   
        \includegraphics[width=0.45\linewidth]{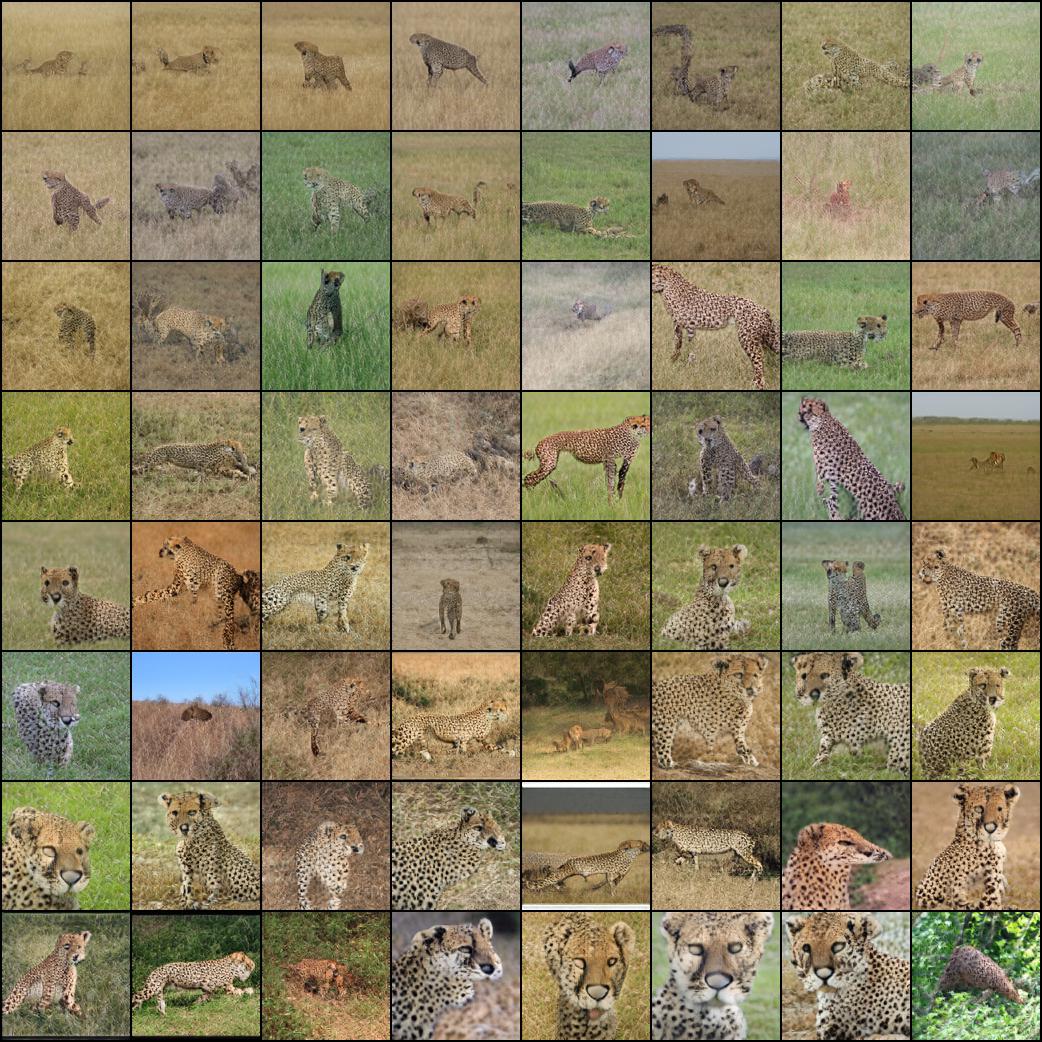}
    }  \hfill
    \subfloat['Cheetah' class histogram. ]
    {   
        \includegraphics[width=0.50\linewidth]{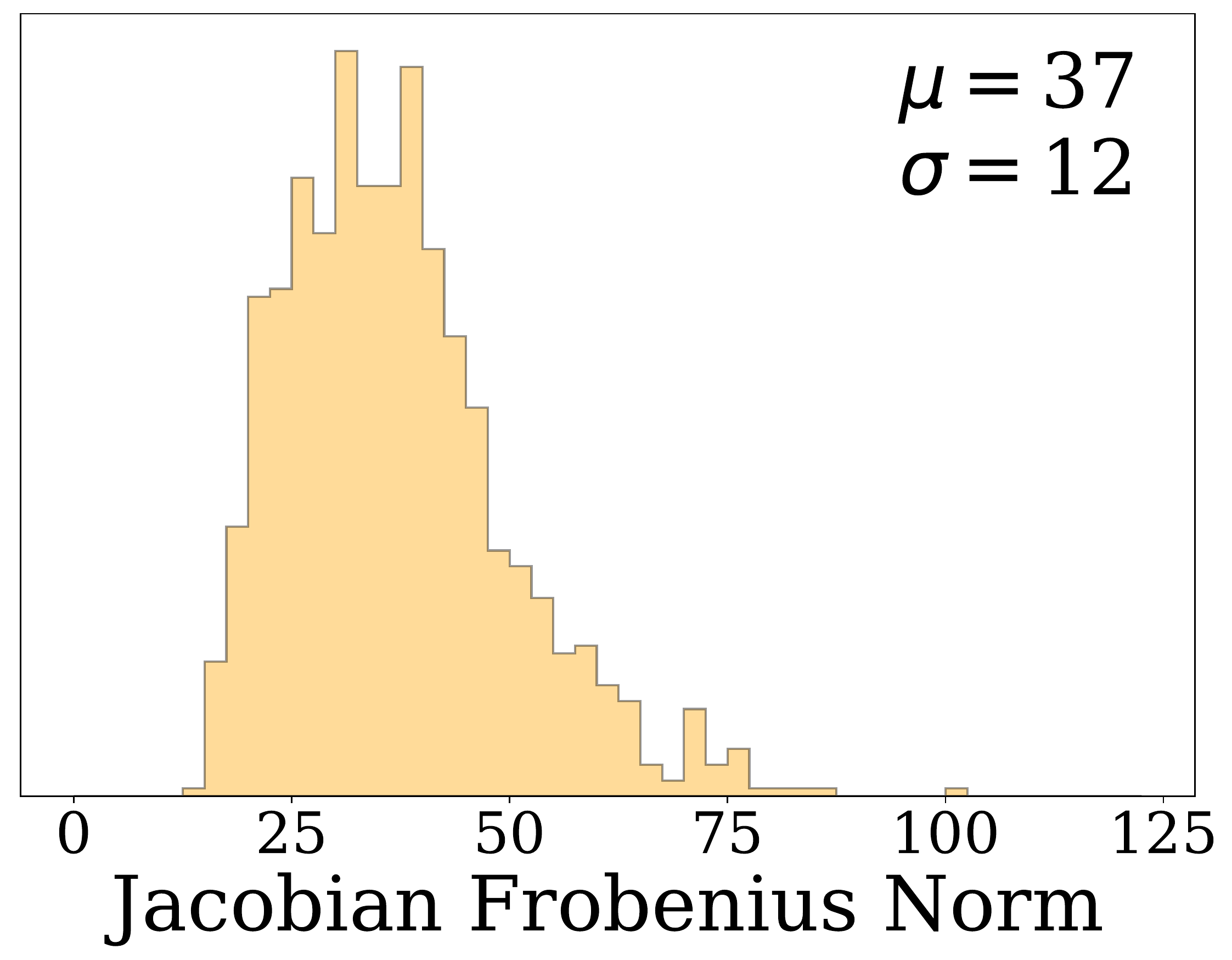}
    }  \hfill
    \phantomcaption
\end{figure}

\begin{figure}
  \ContinuedFloat 
   \subfloat['Monarch butterfly' class. \label{appendix:fig:monarch_butterfly_ranked}]
    {   
        \includegraphics[width=0.45\linewidth]{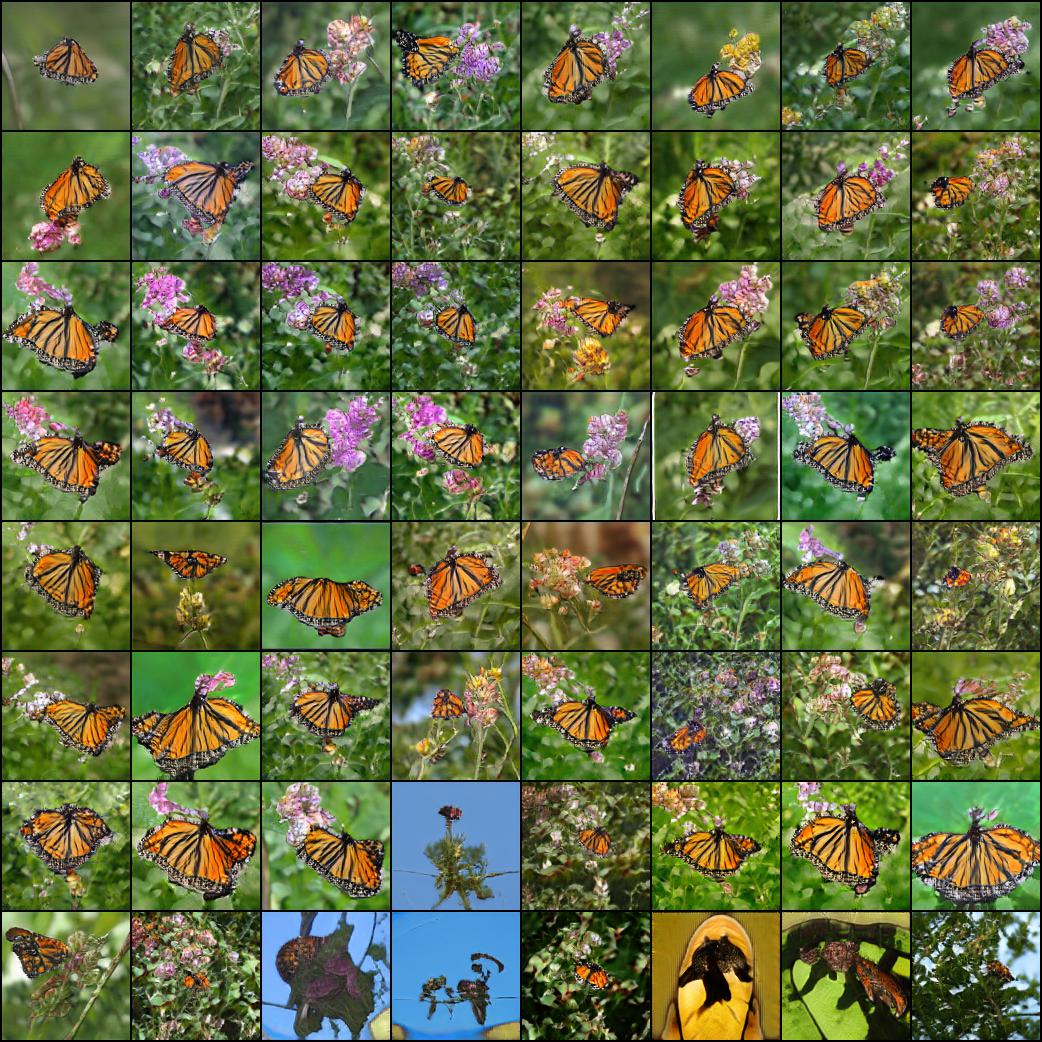}
    }  \hfill
    \subfloat['Monarch butterfly' class histogram.]
    {   
        \includegraphics[width=0.50\linewidth]{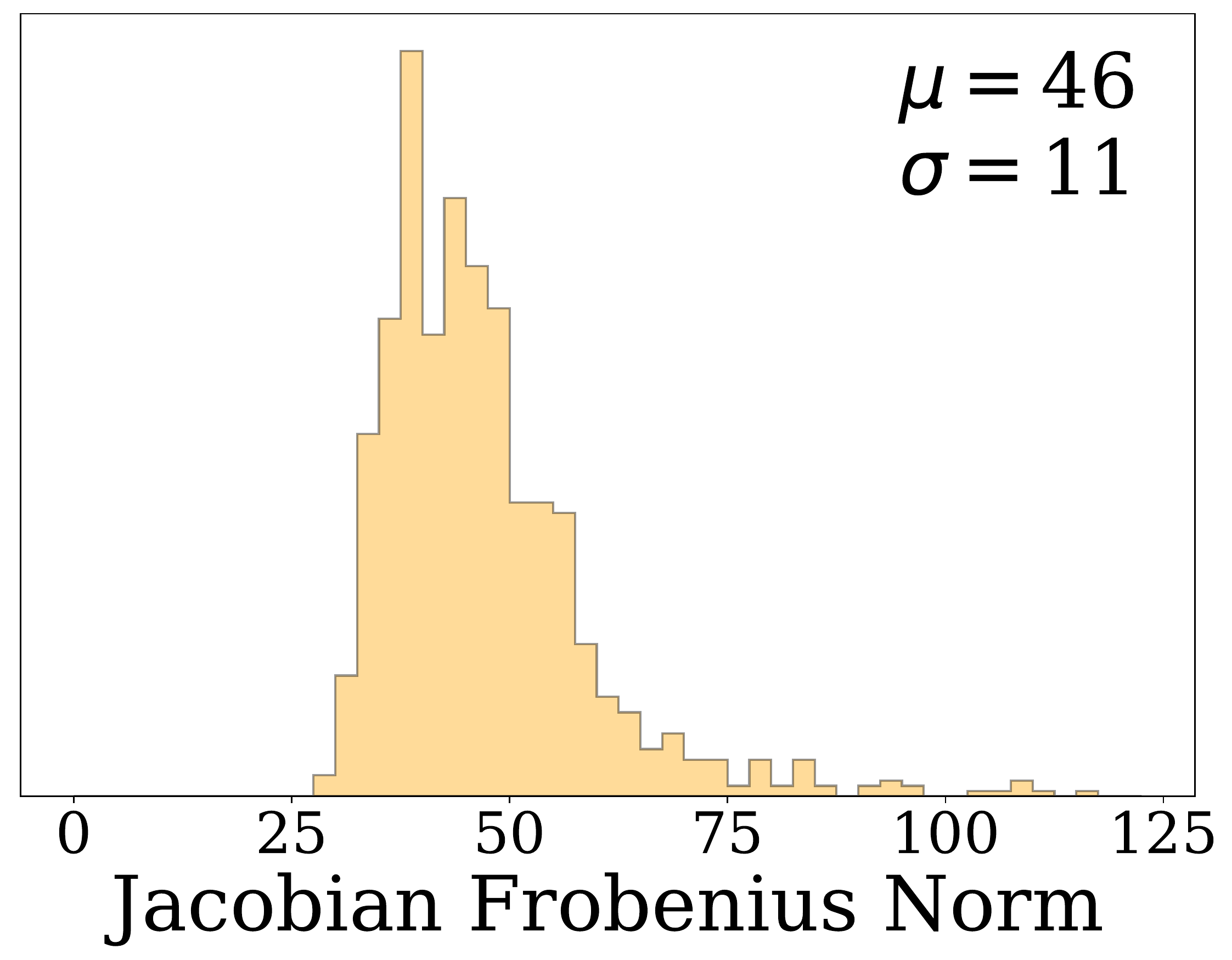}
    }  \hfill
    \subfloat['Loggerhead turtle' class. ]
    {   
        \includegraphics[width=0.45\linewidth]{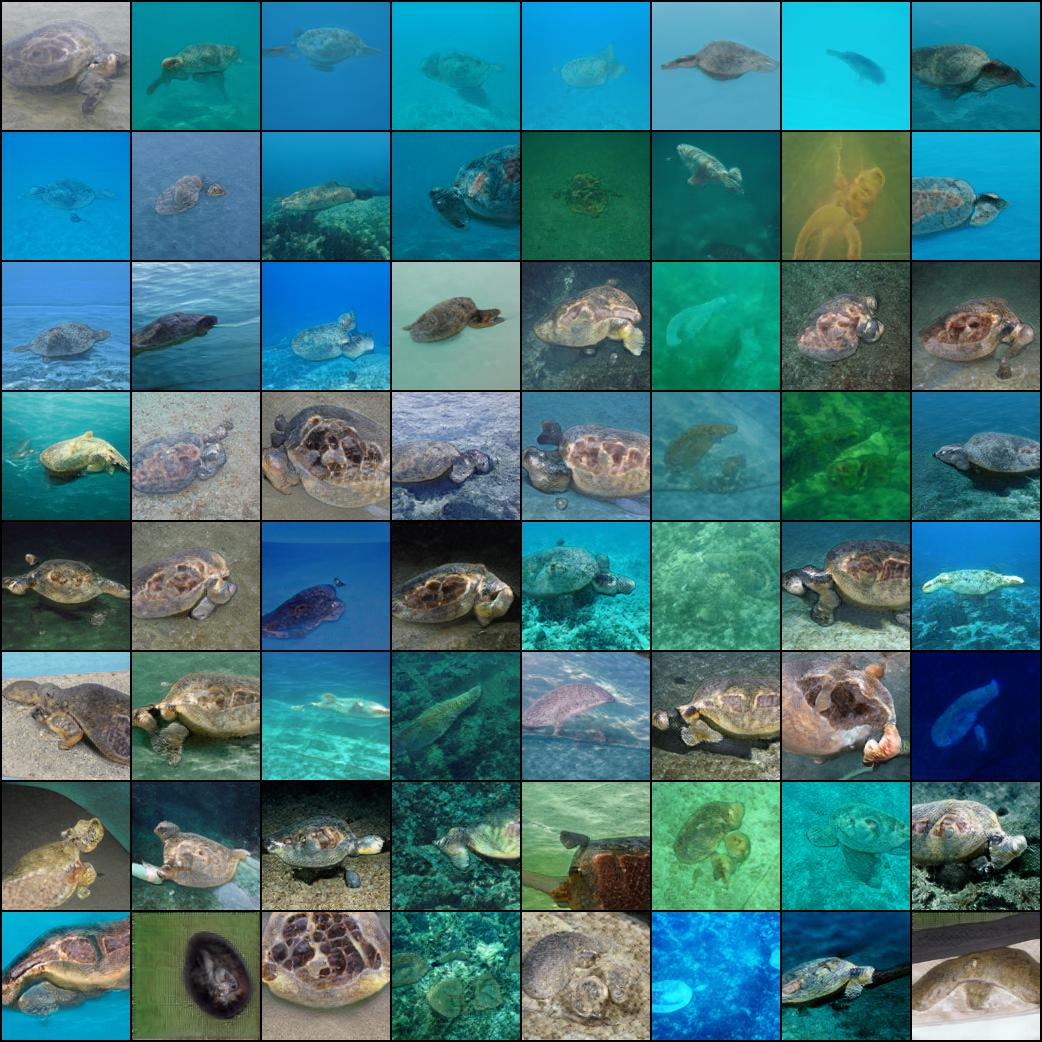}
    }  \hfill
    \subfloat['Loggerhead turtle' class histogram. ]
    {   
        \includegraphics[width=0.50\linewidth]{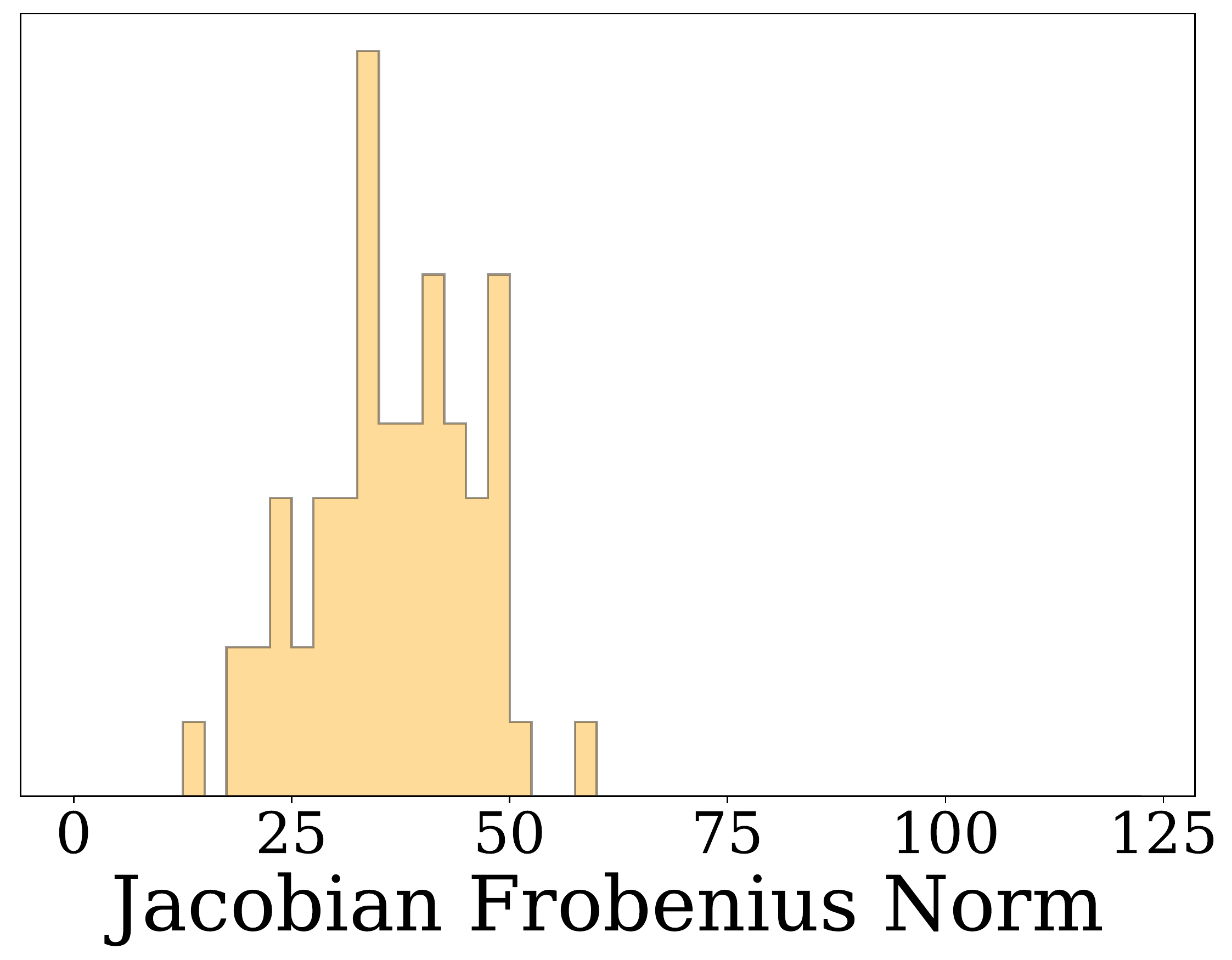}
    }  \hfill
    \subfloat['Wood rabbit' class. \label{appendix:fig:wood_rabbit}]
    {   
        \includegraphics[width=0.45\linewidth]{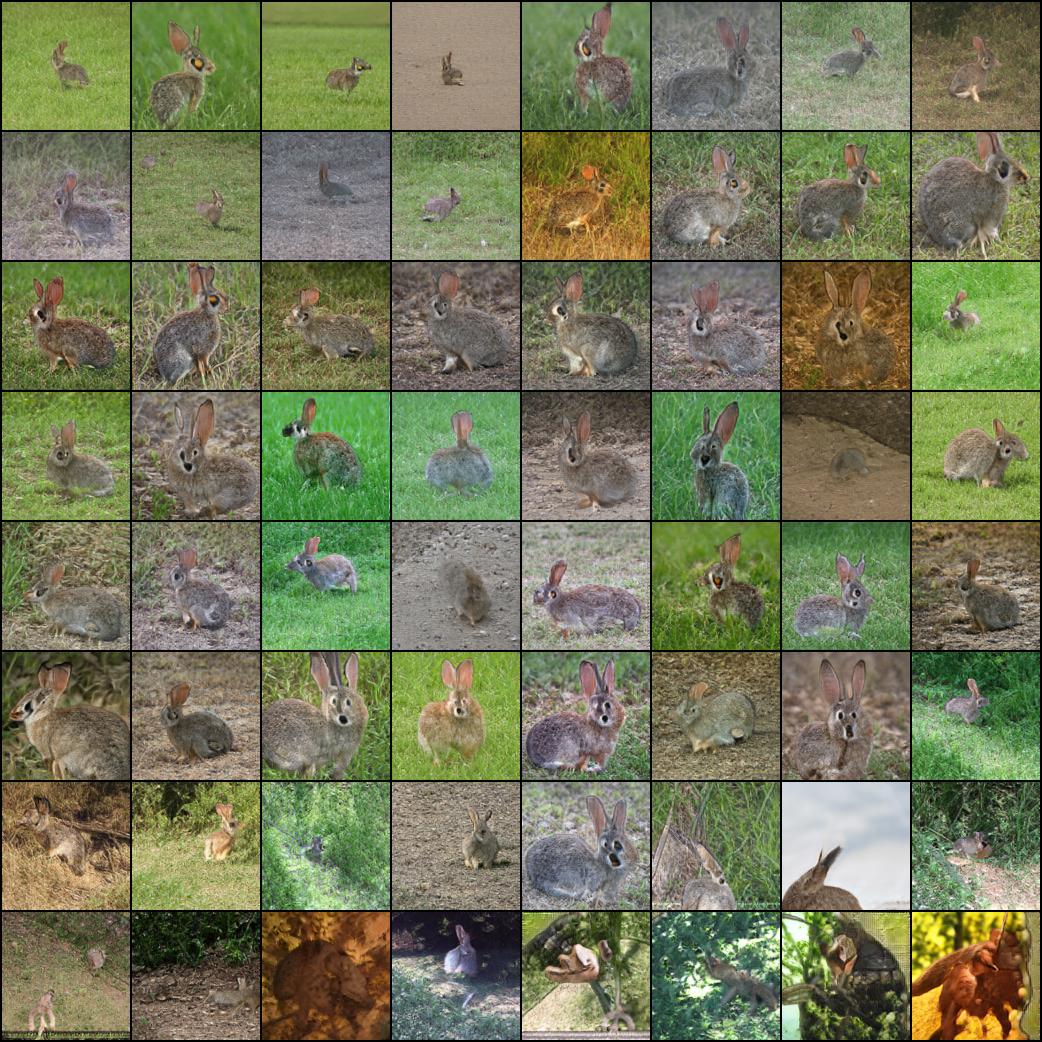}
    }  \hfill
    \subfloat['wood rabbit' class histogram. ]
    {   
        \includegraphics[width=0.50\linewidth]{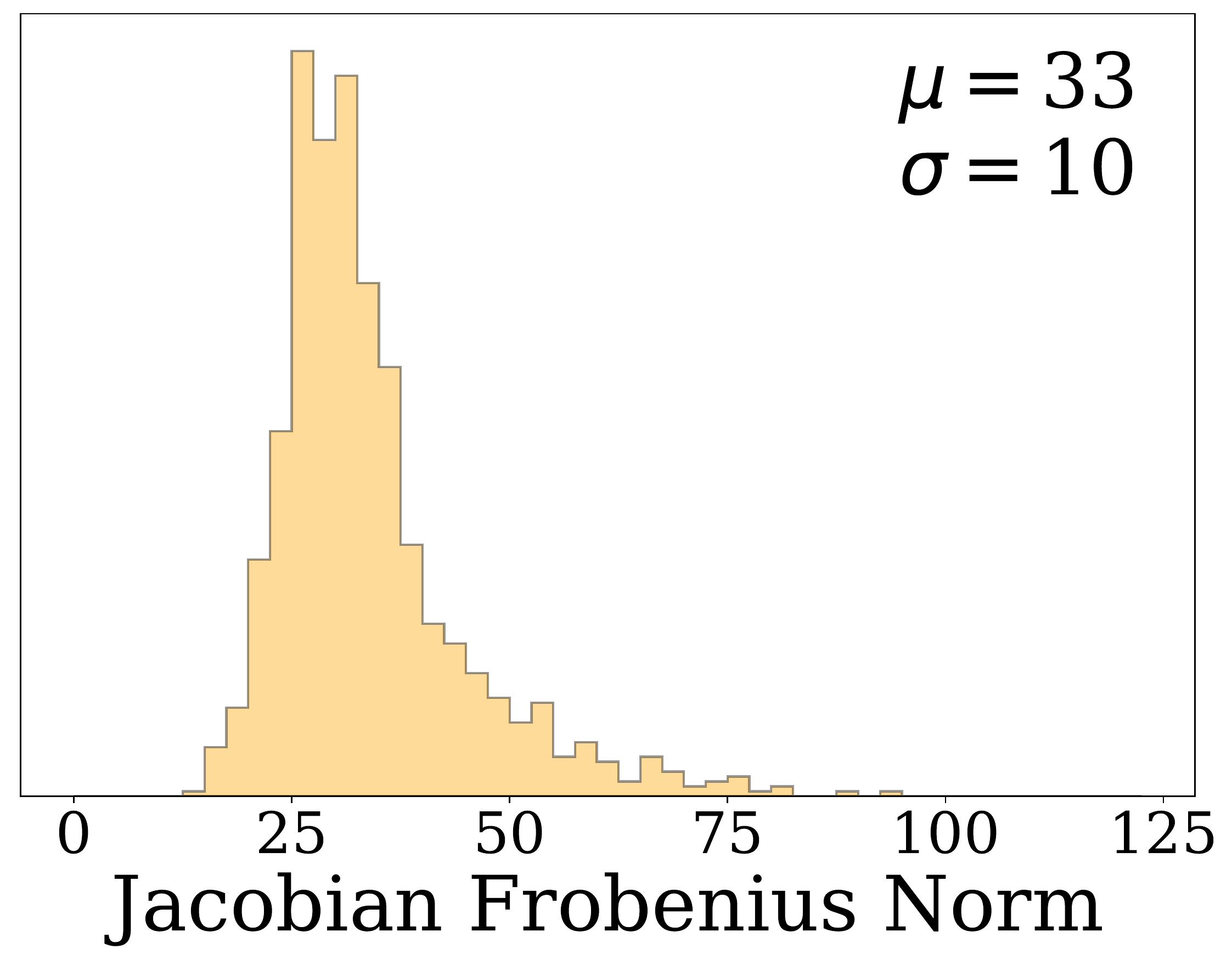}
    }  \hfill
    \subfloat['Trash can' class.]
    {   
        \includegraphics[width=0.45\linewidth]{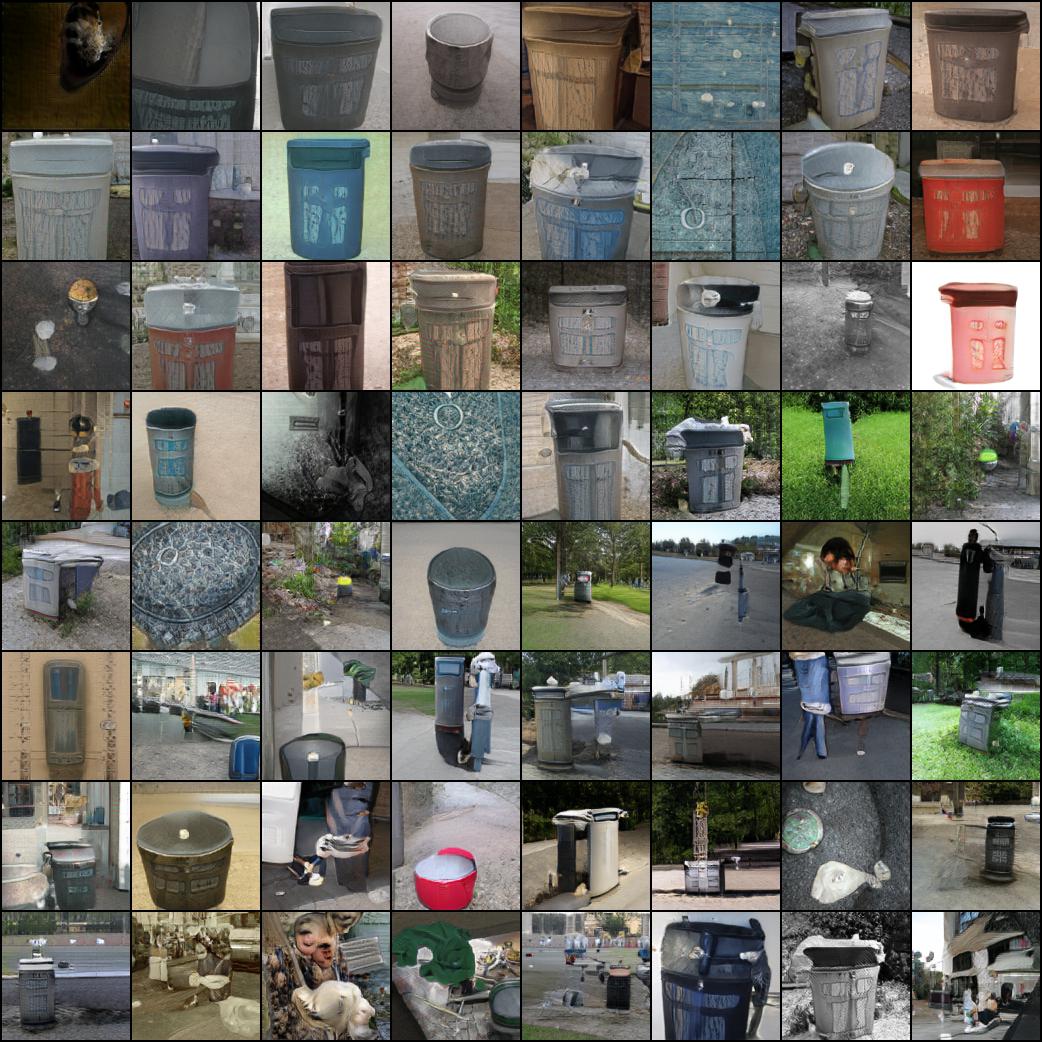}
    }  \hfill
    \subfloat['Trash can' class histogram. ]
    {   
        \includegraphics[width=0.50\linewidth]{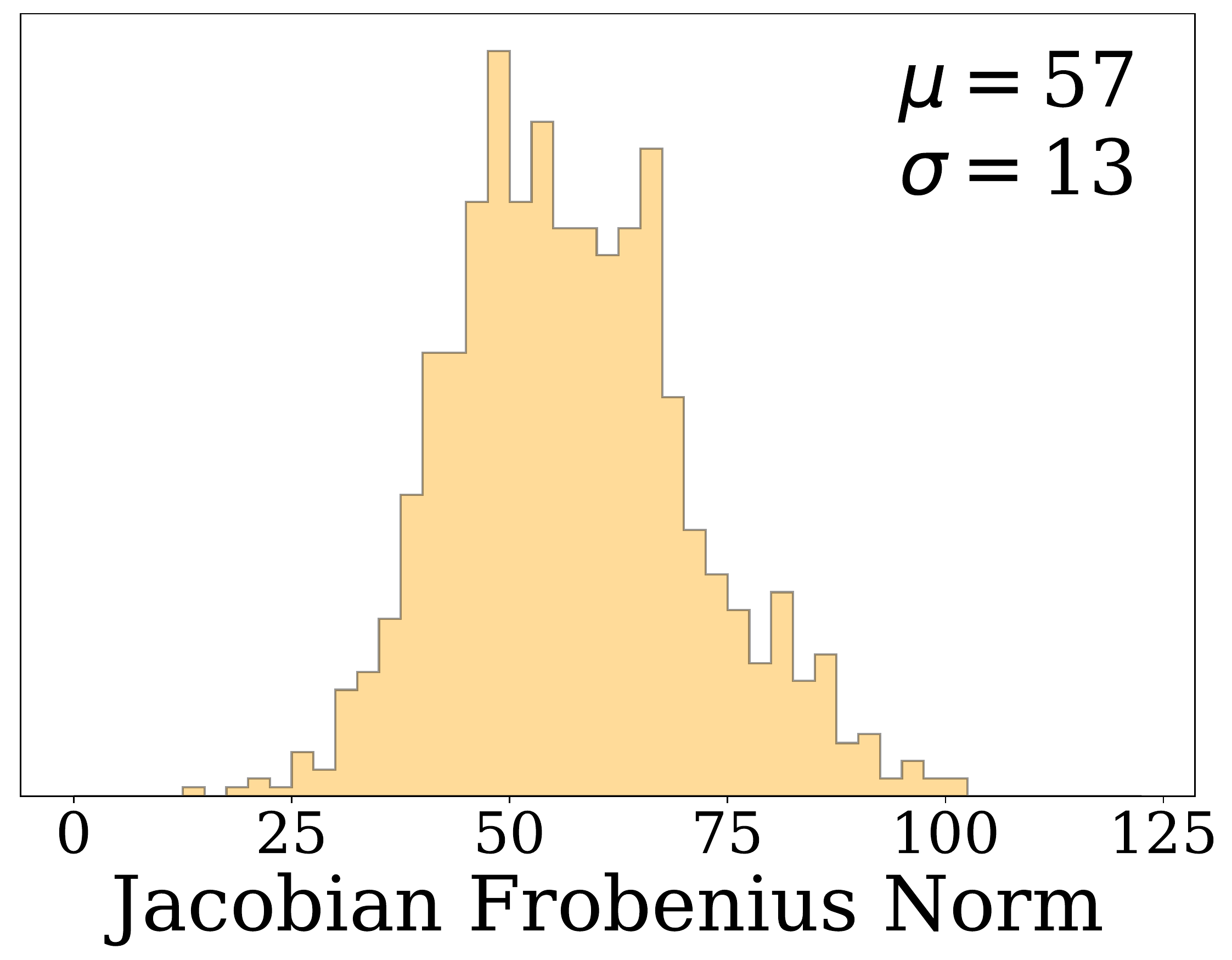}
    }  \hfill
    \phantomcaption
\end{figure}

\begin{figure}
  \ContinuedFloat 
   \subfloat['Cornet/Horn' class. \label{fig:cornet_ranked_appendix} ]
    {   
        \includegraphics[width=0.45\linewidth]{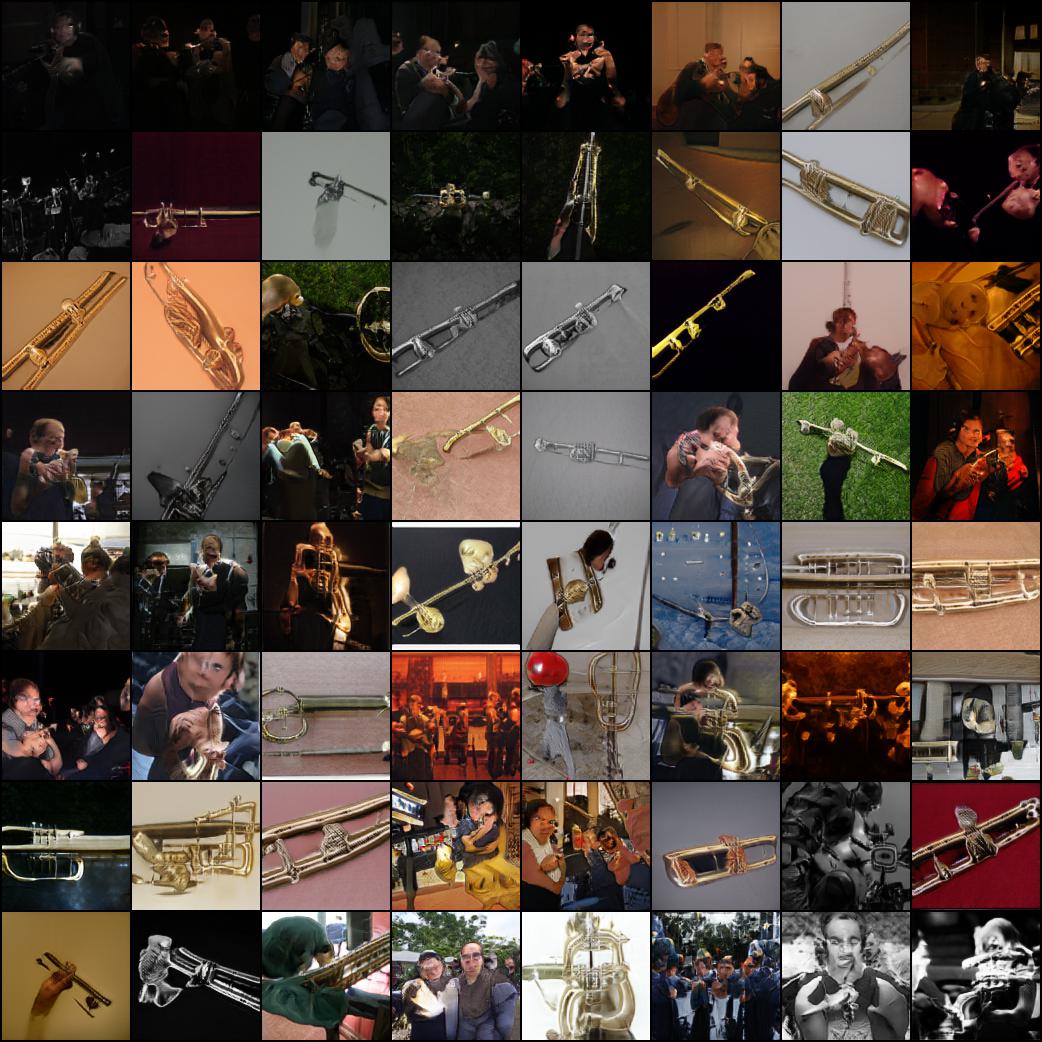}
    }  \hfill
    \subfloat['Cornet/Horn' class histogram. ]
    {   
        \includegraphics[width=0.50\linewidth]{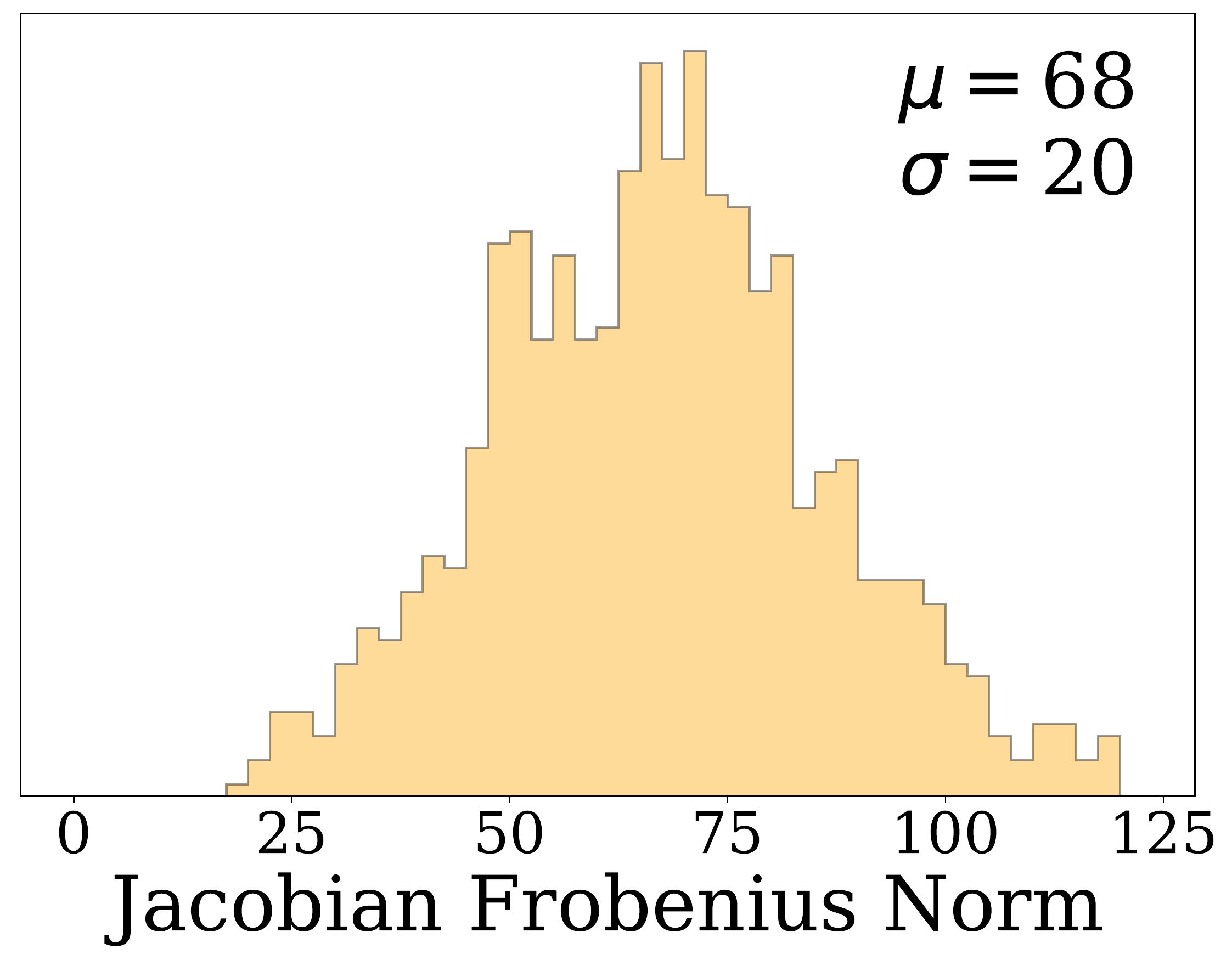}
    }  \hfill
    \subfloat['Football helmet' class. \label{appendix:fig:football_helmet} ]
    {   
        \includegraphics[width=0.45\linewidth]{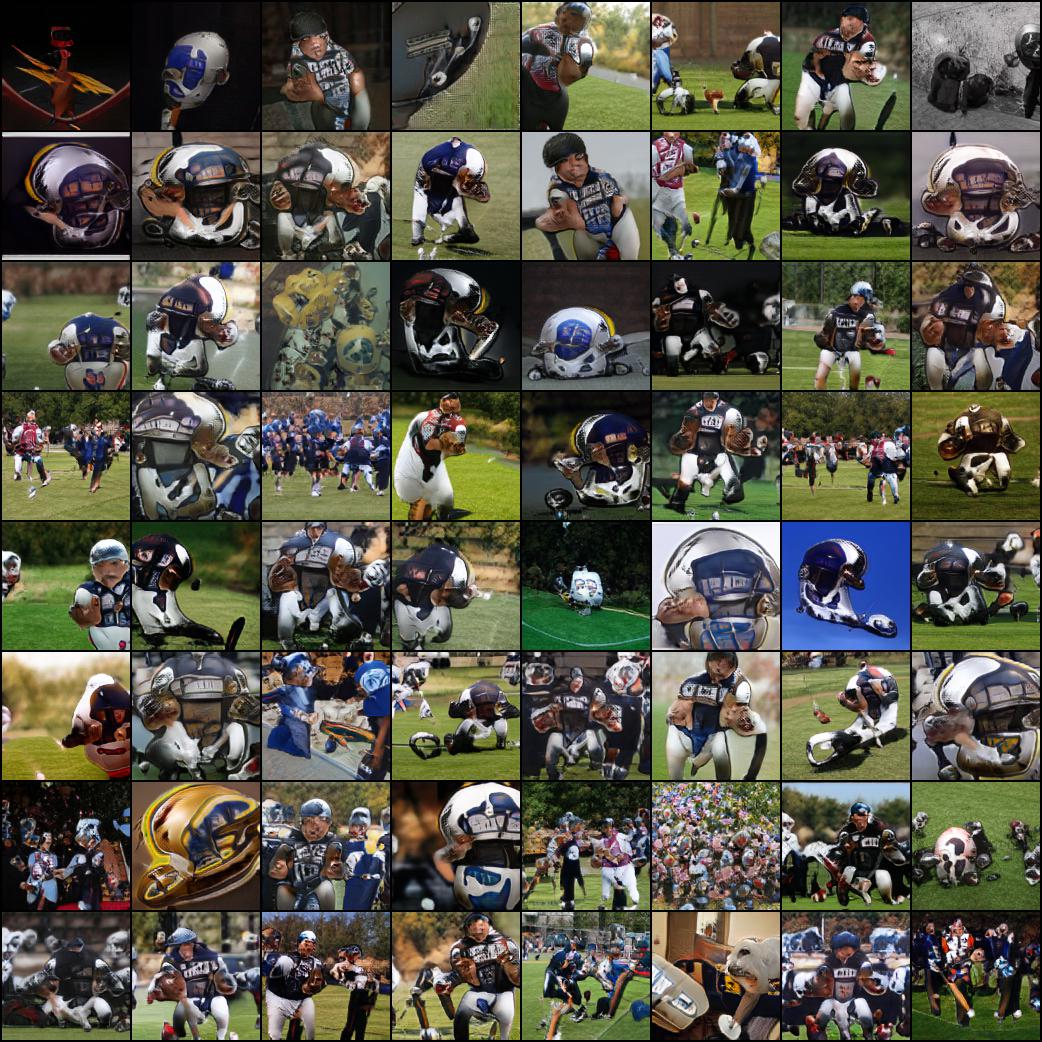}
    }  \hfill
    \subfloat['Football helmet' class histogram. ]
    {   
        \includegraphics[width=0.50\linewidth]{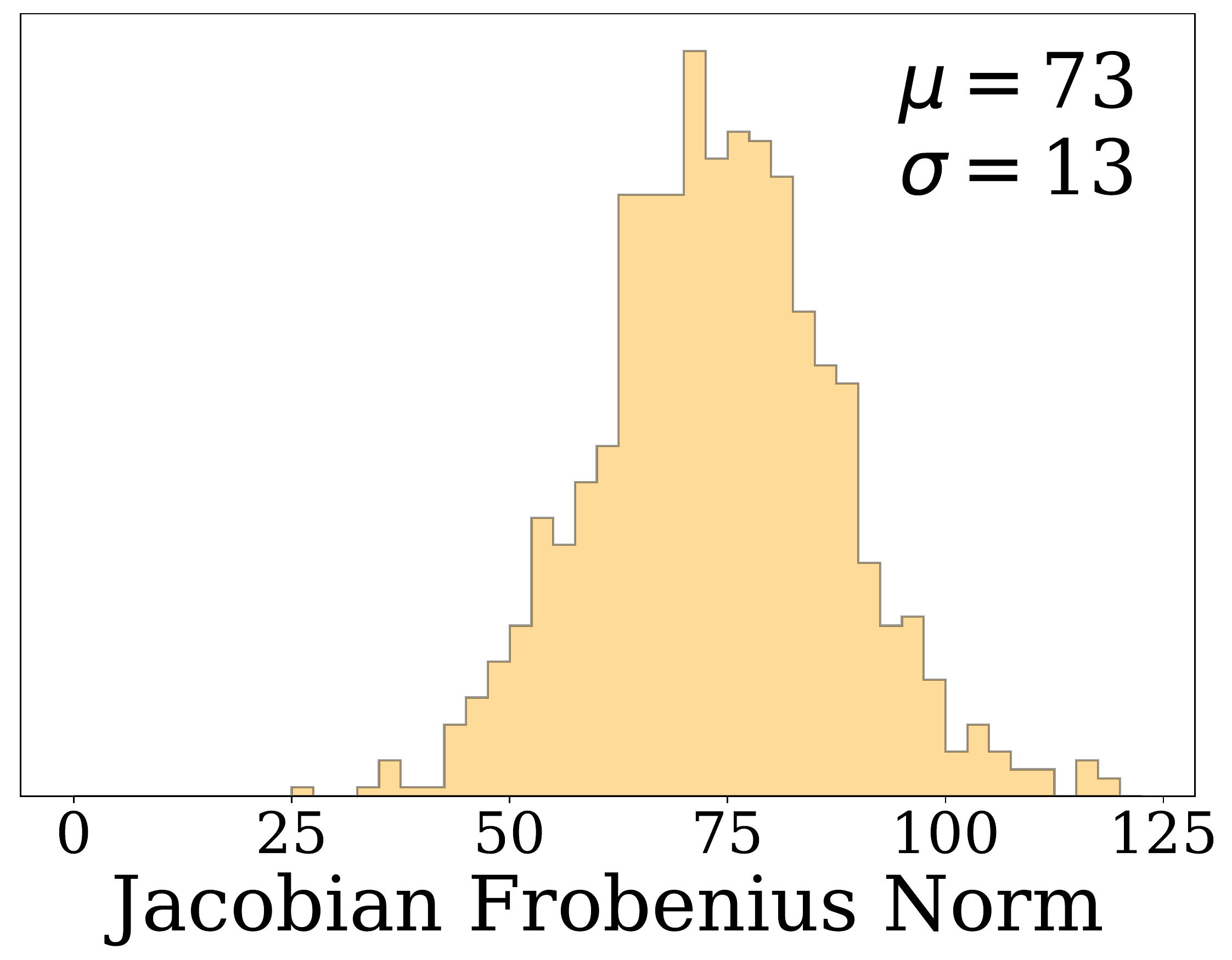}
    }  \hfill
    \subfloat['Harmonica' class. ]
    {   
        \includegraphics[width=0.45\linewidth]{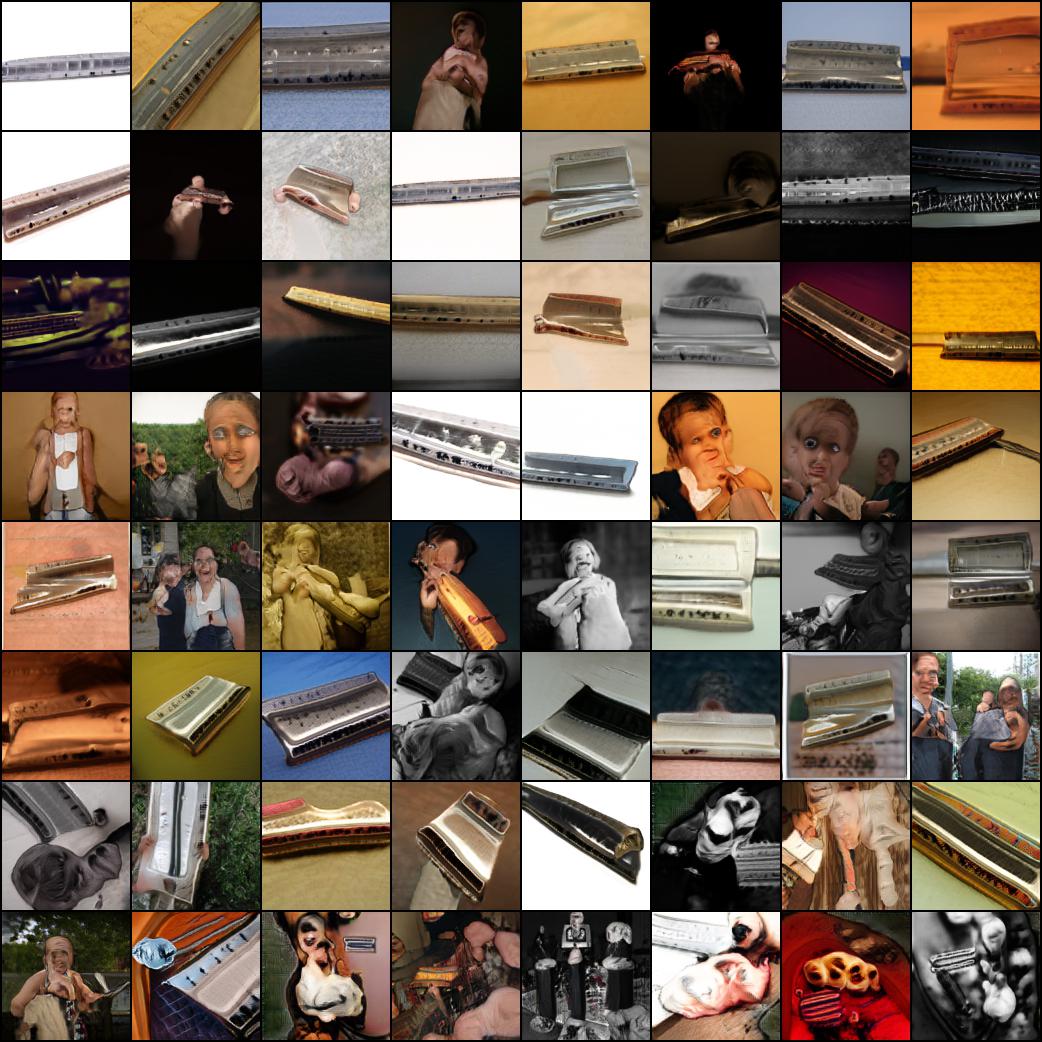}
    }  \hfill
    \subfloat['Harmonica' class histogram. ]
    {   
        \includegraphics[width=0.50\linewidth]{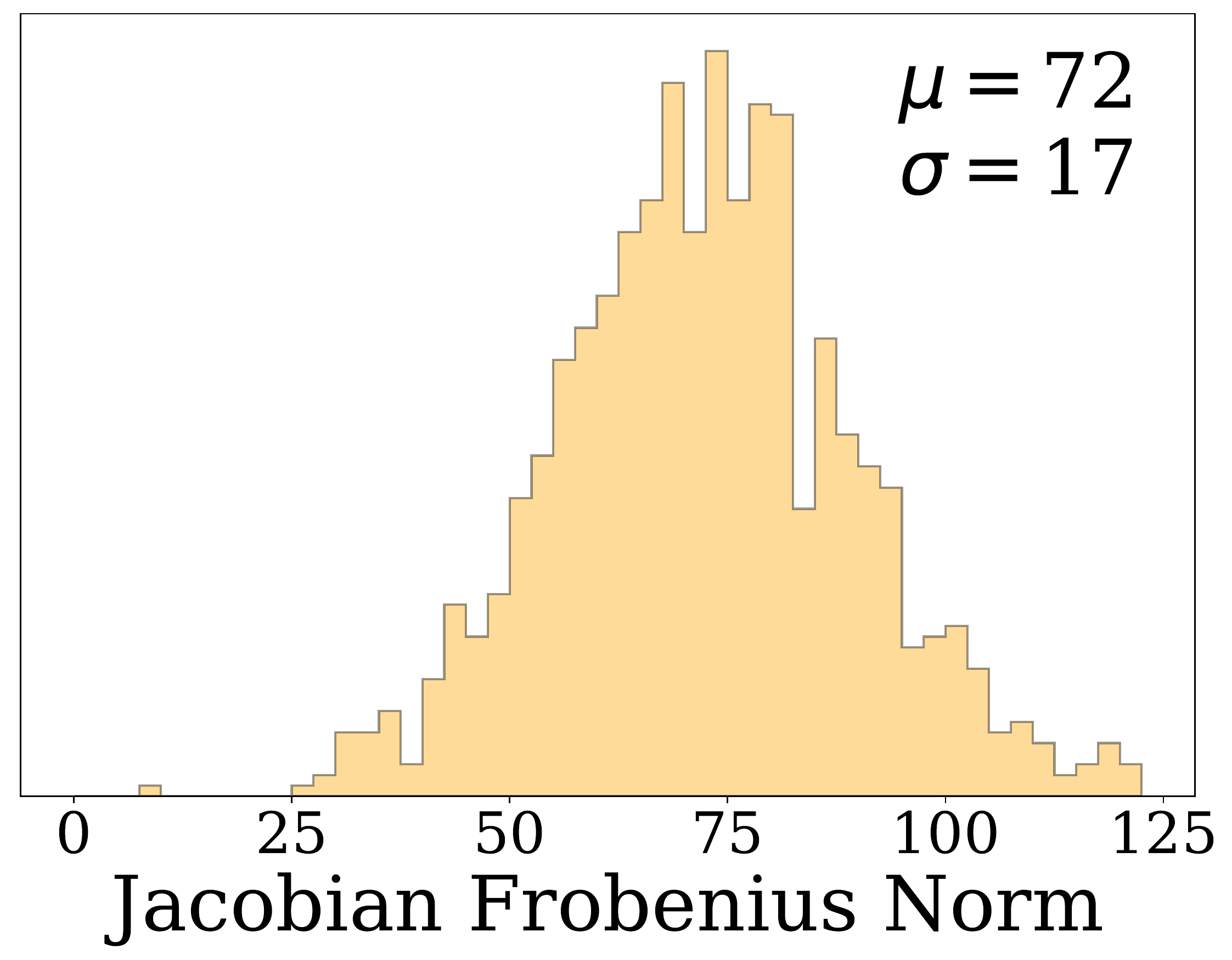}
    }  \hfill
    \subfloat['Parachute' class. ]
    {   
        \includegraphics[width=0.45\linewidth]{images/BigGAN/random_samples_class_701.jpg}
    }  \hfill
    \subfloat['Parachute' class histogram. ]
    {   
        \includegraphics[width=0.50\linewidth]{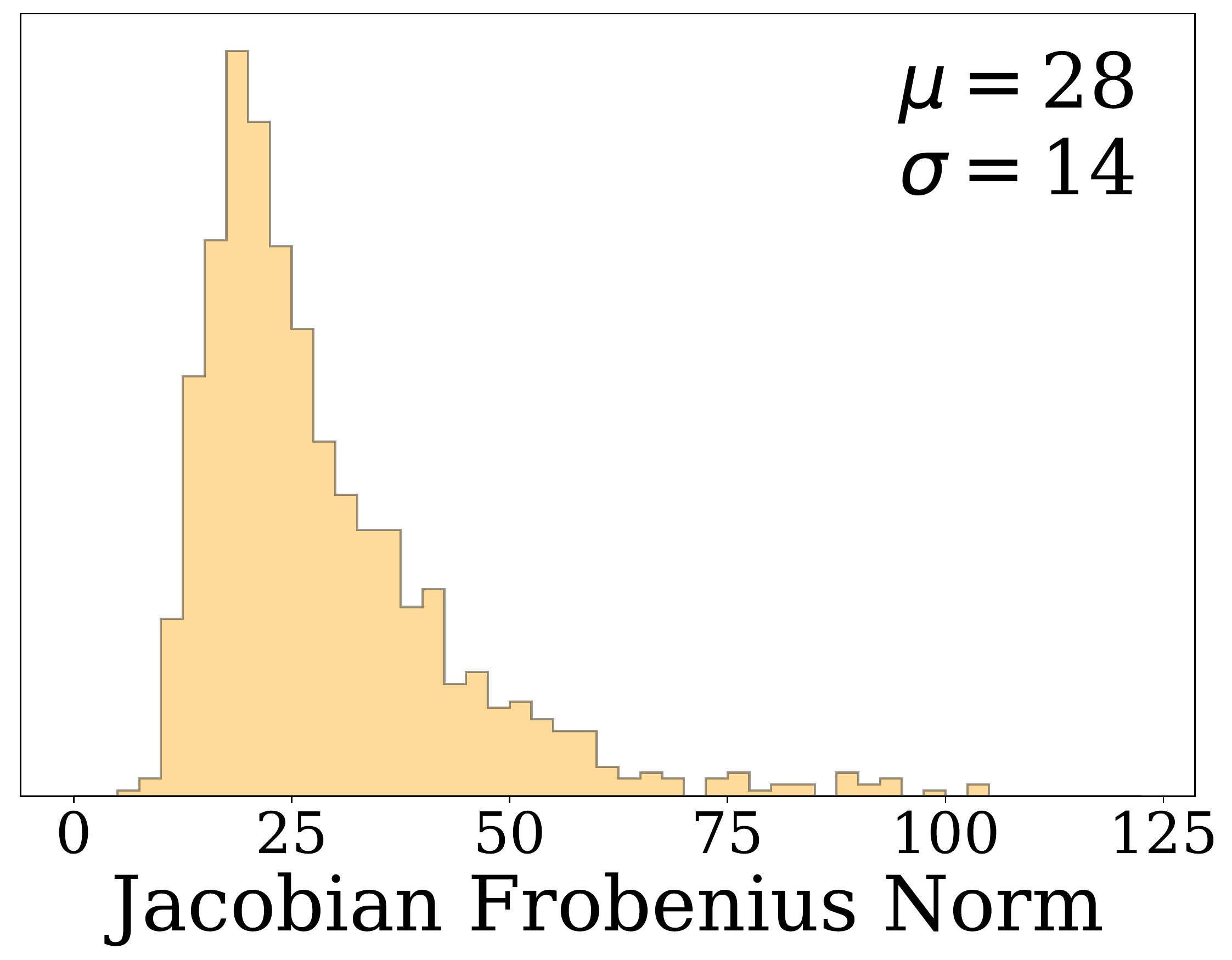}
    }  \hfill
    \subfloat['Peacock' class. ]
    {   
        \includegraphics[width=0.45\linewidth]{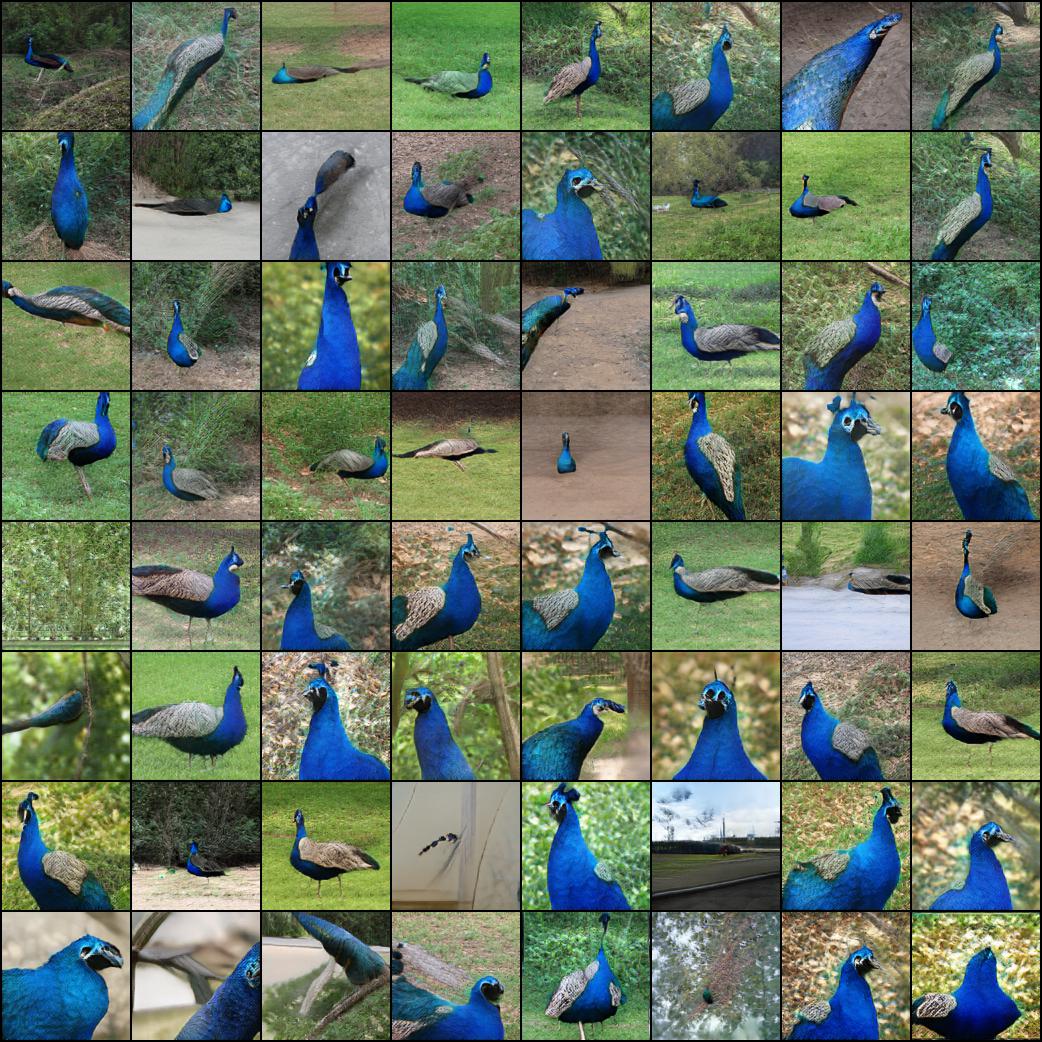}
    }  \hfill
    \subfloat['Peacock' class histogram. ]
    {   
        \includegraphics[width=0.50\linewidth]{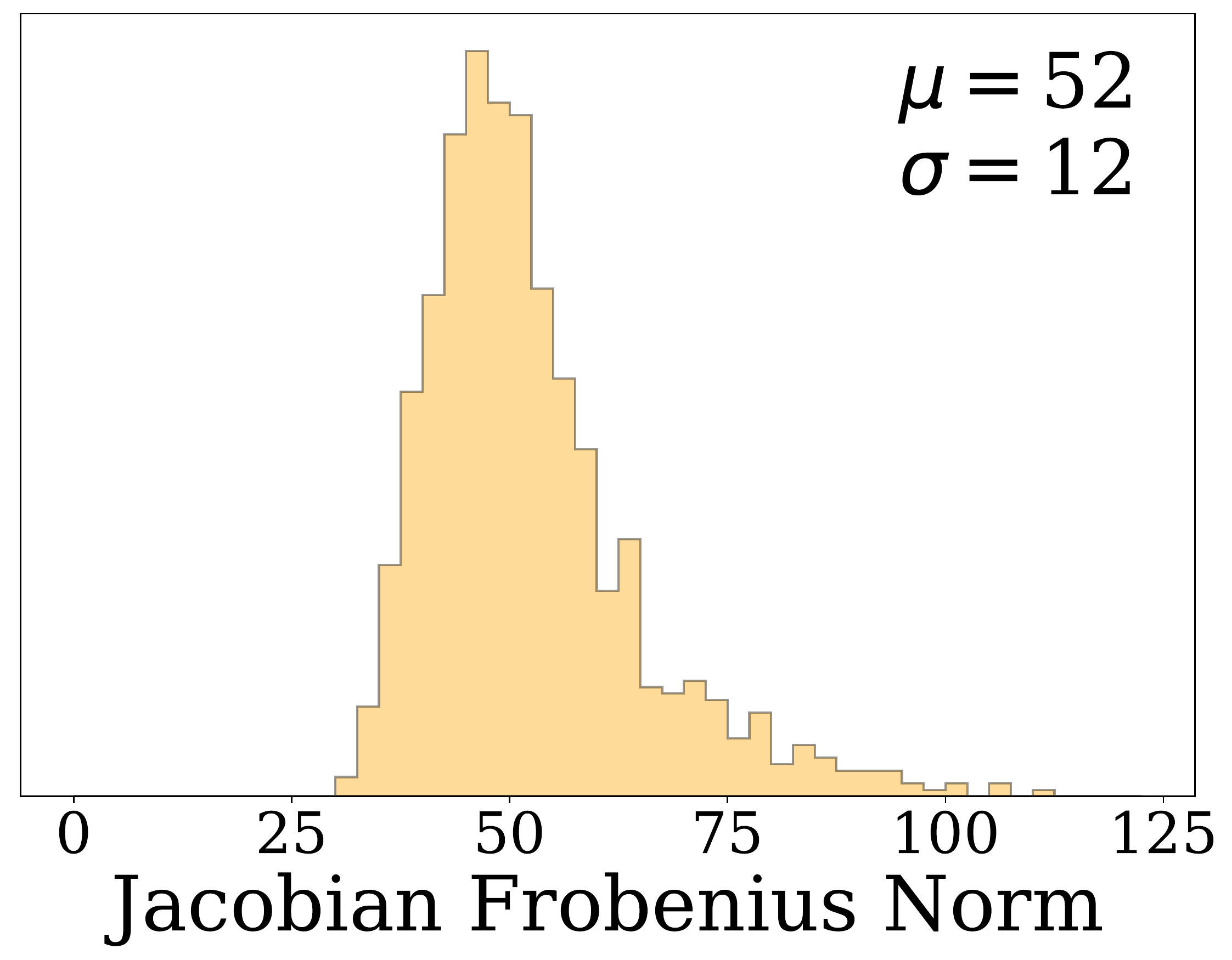}
    }  \hfill
    \caption{For several classes with BigGAN model. \label{appendix:fig:biggan_histograms}}
\end{figure}

\clearpage
\onecolumn

\section{Network Architecture and Hyperparameters}
\label{appendix:experimental_details}

\begin{table*}[h]
\centering
\caption{Models for Synthetic datasets}
\resizebox{0.5\columnwidth}{!}{
\begin{tabular}{llllll}
    \toprule
    Operation & Feature Maps & Activation \\
    \midrule
    G(z): $z \sim \mathcal{N}(0,1)$ & 2 & \\
    Fully Connected - layer1 & 20 & ReLU \\
    Fully Connected - layer2 & 20 & ReLU \\
    \midrule
    D(x) \\
    Fully Connected - layer1 & 20 & ReLU \\
    Fully Connected - layer2 & 20 & ReLU \\
    \midrule
    Batch size & 32\\
    Leaky ReLU slope & 0.2\\
    Gradient Penalty weight & 10\\
    Learning Rate & 0.0002\\
    Optimizer & Adam: $\beta_1=0.5$ & $\beta_2=0.5$\\
    \bottomrule
\end{tabular}
}
\label{appendix:tab:jsd_g_synthetic}
\end{table*}

\begin{table*}[h]
\centering
\caption{WGAN for MNIST/Fashion MNIST}
\resizebox{0.65 \linewidth}{!}{
\begin{tabular}{llllll}
    \toprule
    Operation & Kernel & Strides & Feature Maps & Activation \\
    \midrule
    G(z): $z \sim \mathrm{N}(0,Id)$ & & & 100 & \\
    Fully Connected & & & $7\times7\times128$ &  \\
    Convolution & $3\times3$ & $1\times1$ & $7\times7\times64$ & LReLU \\
    Convolution & $3\times3$ & $1\times1$ & $7\times7\times64$ & LReLU \\
    Nearest Up Sample & & & $14\times14\times64$ & \\
    Convolution & $3\times3$ & $1\times1$ & $14\times14\times32$ & LReLU \\
    Convolution & $3\times3$ & $1\times1$ & $14\times14\times32$ & LReLU\\
    Nearest Up Sample & & & $14\times14\times64$ & \\
    Convolution & $3\times3$ & $1\times1$ & $28\times28\times16$ & LReLU \\
    Convolution & $5\times5$ & $1\times1$ & $28\times28\times1$ & Tanh \\
    \midrule
    D(x) & & & $28\times28\times1$ & & \\
    Convolution & $4\times4$ & $2\times2$ & $14\times14\times32$ & LReLU \\
    Convolution & $3\times3$ & $1\times1$ & $14\times14\times32$ & LReLU \\
    Convolution & $4\times4$ & $2\times2$ & $7\times7\times64$ & LReLU \\
    Convolution & $3\times3$ & $1\times1$ & $7\times7\times64$ & LReLU \\
    Fully Connected & & & $1$ & - \\
    \midrule
    Batch size & 256\\
    Leaky ReLU slope & 0.2\\
    Gradient Penalty weight & 10\\
    Learning Rate & 0.0002\\
    Optimizer & Adam &  $\beta_1:0.5$ &  $\beta_2:0.5$\\
    \bottomrule
\end{tabular}
}
\label{appendix:tab_jsd_g_wgan_mnist}
\end{table*}

For DeliGan, we use the same architecture and simply add 50 Gaussians for the reparametrization trick. For DMLGAN, we re-use the architecture of the authors.

\begin{table*}
\centering
\caption{DMLGAN for MNIST/Fashion MNIST}
\resizebox{0.7 \columnwidth}{!}{
\begin{tabular}{llllll}
    \toprule
    Operation & Kernel & Strides & Feature Maps & BN & Activation \\
    \midrule
    G(z): $z \sim \mathrm{N}(0,Id)$ & & & 100 & & \\
    Fully Connected & & & $7\times7\times128$ & - &  \\
    Convolution & $3\times3$ & $1\times1$ & $7\times7\times64$ & - & Leaky ReLU \\
    Convolution & $3\times3$ & $1\times1$ & $7\times7\times64$ & - & Leaky ReLU \\
    Nearest Up Sample & & & $14\times14\times64$ & - & \\
    Convolution & $3\times3$ & $1\times1$ & $14\times14\times32$ & - & Leaky ReLU \\
    Convolution & $3\times3$ & $1\times1$ & $14\times14\times32$ & - & Leaky ReLU\\
    Nearest Up Sample & & & $14\times14\times64$ & - & \\
    Convolution & $3\times3$ & $1\times1$ & $28\times28\times16$ & - & Leaky ReLU \\
    Convolution & $5\times5$ & $1\times1$ & $28\times28\times1$ & - & Tanh \\
    \midrule
    Encoder Q(x), Discriminator D(x) & & & $28\times28\times1$ & & \\
    Convolution & $4\times4$ & $2\times2$ & $14\times14\times32$ & - & Leaky ReLU \\
    Convolution & $3\times3$ & $1\times1$ & $14\times14\times32$ & - & Leaky ReLU \\
    Convolution & $4\times4$ & $2\times2$ & $7\times7\times64$ & - & Leaky ReLU \\
    Convolution & $3\times3$ & $1\times1$ & $7\times7\times64$ & - & Leaky ReLU \\
    D Fully Connected & & & $1$ & - & - \\
    Q Convolution &  $3\times3$ & & $7\times7\times64$ & Y & Leaky ReLU \\
    Q Convolution &  $3\times3$ & &$7\times7\times64$& Y & Leaky ReLU \\
    Q Fully Connected & & & $n_g = 10$ & - & Softmax \\
    \midrule
    Batch size & 256\\
    Leaky ReLU slope & 0.2\\
    Gradient Penalty weight & 10\\
    Learning Rate & 0.0002\\
    Optimizer & Adam & $\beta_1=0.5$ & $\beta_2=0.5$\\
    \bottomrule
\end{tabular}
}
\label{appendix:tab_jsd_g_dmlgan_mnist}
\end{table*}

\begin{table*}
\centering
\caption{WGAN for CIFAR10, from \cite{gulrajani2017improved}}
\resizebox{0.7 \columnwidth}{!}{
\begin{tabular}{llllll}
    \toprule
    Operation & Kernel & Strides & Feature Maps & BN & Activation \\
    \midrule
    G(z): $z \sim \mathrm{N}(0,Id)$ & & & 128 & & \\
    Fully Connected & & & $4\times4\times128$ & - &  \\
    ResBlock & $[3\times3]\times 2$ & $1\times1$ & $4\times4\times128$ & Y & ReLU \\
    Nearest Up Sample & & & $8\times8\times128$ & - & \\
    ResBlock & $[3\times3]\times 2$ & $1\times1$ & $8\times8\times128$ & Y & ReLU \\
    Nearest Up Sample & & & $16\times16\times128$ & - & \\
    ResBlock & $[3\times3]\times 2$ & $1\times1$ & $16\times16\times128$ & Y & ReLU \\
    Nearest Up Sample & & & $32\times32\times128$ & - & \\
    Convolution & $3\times3$ & $1\times1$ & $32\times32\times3$ & - & Tanh \\
    \midrule
    Discriminator D(x) & & & $32\times32\times3$ & & \\
    ResBlock & $[3\times3]\times2$ & $1\times1$ & $32\times32\times128$ & - & ReLU \\
    AvgPool & $2\times2$ & $1\times1$ & $16\times16\times128$ & - & \\
    ResBlock & $[3\times3]\times2$ & $1\times1$ & $16\times16\times128$ & - & ReLU \\
    AvgPool & $2\times2$ & $1\times1$ & $8\times8\times128$ & - & \\
    ResBlock & $[3\times3]\times2$ & $1\times1$ & $8\times8\times128$ & - & ReLU \\
    ResBlock & $[3\times3]\times2$ & $1\times1$ & $8\times8\times128$ & - & ReLU \\
    Mean pooling (spatial-wise) & - & - & $128$ & - & \\
    Fully Connected & & & $1$ & - & - \\
    \midrule
    Batch size & 64\\
    Gradient Penalty weight & 10\\
    Learning Rate & 0.0002\\
    Optimizer & Adam & $\beta_1=0.$ & $\beta_2=0.9$\\
    Discriminator steps & 5 \\
    \bottomrule
\end{tabular}
}
\label{appendix:tab_jsd_g_wgan}
\end{table*}

\begin{table*}
\centering
\caption{DMLGAN for CIFAR10, from \cite{gulrajani2017improved}}
\resizebox{0.7 \columnwidth}{!}{
\begin{tabular}{llllll}
    \toprule
    Operation & Kernel & Strides & Feature Maps & BN & Activation \\
    \midrule
    G(z): $z \sim \mathrm{N}(0,Id)$ & & & 128 & & \\
    Fully Connected & & & $4\times4\times128$ & - &  \\
    ResBlock & $[3\times3]\times 2$ & $1\times1$ & $4\times4\times128$ & Y & ReLU \\
    Nearest Up Sample & & & $8\times8\times128$ & - & \\
    ResBlock & $[3\times3]\times 2$ & $1\times1$ & $8\times8\times128$ & Y & ReLU \\
    Nearest Up Sample & & & $16\times16\times128$ & - & \\
    ResBlock & $[3\times3]\times 2$ & $1\times1$ & $16\times16\times128$ & Y & ReLU \\
    Nearest Up Sample & & & $32\times32\times128$ & - & \\
    Convolution & $3\times3$ & $1\times1$ & $32\times32\times3$ & - & Tanh \\
    \midrule
    Encoder Q(x), Discriminator D(x) & & & $32\times32\times3$ & & \\
    ResBlock & $[3\times3]\times2$ & $1\times1$ & $32\times32\times128$ & - & ReLU \\
    AvgPool & $2\times2$ & $1\times1$ & $16\times16\times128$ & - & \\
    ResBlock & $[3\times3]\times2$ & $1\times1$ & $16\times16\times128$ & - & ReLU \\
    AvgPool & $2\times2$ & $1\times1$ & $8\times8\times128$ & - & \\
    ResBlock & $[3\times3]\times2$ & $1\times1$ & $8\times8\times128$ & - & ReLU \\
    D ResBlock & $[3\times3]\times2$ & $1\times1$ & $8\times8\times128$ & - & ReLU \\
    D Mean pooling (spatial-wise) & $2\times2$ & $1\times1$ & $128$ & - & \\
    D Fully Connected & & & $1$ & - & - \\
    Q ResBlock & $[3\times3]\times2$ & $1\times1$ & $8\times8\times128$ & - & ReLU \\
    Q Mean pooling (spatial-wise) & $2\times2$ & $1\times1$ & $128$ & - & \\
    Q Fully Connected & & & $n_g = 10$ & - & Softmax \\
    \midrule
    Batch size & 64\\
    Gradient Penalty weight & 10\\
    Learning Rate & 0.0002\\
    Optimizer & Adam & $\beta_1=0.$ & $\beta_2=0.9$\\
    Discriminator steps & 5 \\
    \bottomrule
\end{tabular}
}
\label{appendix:tab_jsd_g_dml}
\end{table*}

\end{document}